\documentclass[preprint,3p,times]{elsarticle}

\usepackage{paper, math}
\graphicspath{{figs/}}

\journal{Computer Methods in Applied Mechanics and Engineering}

\begin{document}

\begin{frontmatter}

	%% Title, authors and addresses

	%% use the tnoteref command within \title for footnotes;
	%% use the tnotetext command for theassociated footnote;
	%% use the fnref command within \author or \affiliation for footnotes;
	%% use the fntext command for theassociated footnote;
	%% use the corref command within \author for corresponding author footnotes;
	%% use the cortext command for theassociated footnote;
	%% use the ead command for the email address,
	%% and the form \ead[url] for the home page:
	%% \title{Title\tnoteref{label1}}
	%% \tnotetext[label1]{}
	%% \author{Name\corref{cor1}\fnref{label2}}
	%% \ead{email address}
	%% \ead[url]{home page}
	%% \fntext[label2]{}
	%% \cortext[cor1]{}
	%% \affiliation{organization={},
	%%             addressline={},
	%%             city={},
	%%             postcode={},
	%%             state={},
	%%             country={}}
	%% \fntext[label3]{}

	\title{DeepRTE:\ Pre-trained Attention-based Neural Network for
		Radiative Transfer} %% Article title

	%% use optional labels to link authors explicitly to addresses:
	%% \author[label1,label2]{}
	%% \affiliation[label1]{organization={},
	%%             addressline={},
	%%             city={},
	%%             postcode={},
	%%             state={},
	%%             country={}}
	%%
	%% \affiliation[label2]{organization={},
	%%             addressline={},
	%%             city={},
	%%             postcode={},
	%%             state={},
	%%             country={}}

	\author[math]{Yekun Zhu} %% Author name
	\ead{zhuyekun123@sjtu.edu.cn}
	\author[math,ins]{Min Tang\corref{cor}}
	\ead{tangmin@sjtu.edu.cn}
	\author[math,ins,cma]{Zheng Ma\corref{cor}}
	\ead{zhengma@sjtu.edu.cn}
	%% Author affiliation
	\affiliation[math]{organization={School of Mathematical Sciences, Shanghai Jiao Tong University},%Department and Organization
		% addressline={No. 800 Dongchuan Road},
		city={Shanghai},
		postcode={200240},
		country={China}
	}
	\affiliation[ins]{organization={Institute of Natural Sciences, MOE-LSC, Shanghai Jiao Tong University},%Department and Organization
		% addressline={No. 800 Dongchuan Road},
		city={Shanghai},
		postcode={200240},
		country={China}
	}
	\affiliation[cma]{organization={CMA-Shanghai, Shanghai Jiao Tong University},%Department and Organization
		% addressline={No. 800 Dongchuan Road},
		city={Shanghai},
		postcode={200240},
		country={China}
	}
	\cortext[cor]{Corresponding authors}

	%% Abstract
	\begin{abstract}
  % In this paper, we propose a novel neural network approach, termed DeepRTE, to address the steady-state Radiative Transfer Equation (RTE). The RTE is a differential-integral equation that governs the propagation of radiation through a participating medium, with applications spanning diverse domains such as neutron transport, atmospheric radiative transfer, heat transfer, and optical imaging. Our proposed DeepRTE framework leverages pre-trained attention-based neural networks to solve the RTE with high accuracy and computational efficiency. The efficacy of the proposed approach is substantiated through comprehensive numerical experiments.

In this paper, we propose a novel neural network approach, termed DeepRTE, to address the steady-state Radiative Transfer Equation (RTE).
The RTE is a differential-integral equation that governs the propagation of radiation through a participating medium, with applications spanning diverse domains such as neutron transport, atmospheric radiative transfer, heat transfer, and optical imaging.
Our DeepRTE framework demonstrates superior computational efficiency for solving the steady-state RTE, surpassing traditional methods and existing neural network approaches.
This efficiency is achieved by embedding physical information through derivation of the RTE and mathematically-informed network architecture.
Concurrently, DeepRTE achieves high accuracy with significantly fewer parameters, largely due to its incorporation of mechanisms such as multi-head attention.
Furthermore, DeepRTE is a mesh-free neural operator framework with inherent zero-shot capability. This is achieved by incorporating Green’s function theory and pre-training with delta-function inflow boundary conditions into both its architecture design and training data construction.
The efficacy of the proposed approach is substantiated through comprehensive numerical experiments.
\end{abstract}
% % REQUIRED
% \begin{keywords}
%     Deep Neural Networks, Neural Operator, Radiative Transfer
% Equation, Transformer.
% \end{keywords}

% % REQUIRED
% \begin{MSCcodes}
%     68Q25, 68R10, 68U05
% \end{MSCcodes}

	% \begin{abstract}
	% %% Text of abstract
	% Abstract text.
	% \end{abstract}

	%%Graphical abstract
	%\begin{graphicalabstract}
	%\includegraphics{grabs}
	%\end{graphicalabstract}

	% %%Research highlights
	% \begin{highlights}
	% 	% \item DeepRTE achieve
	% 	\item DeepRTE exceeds traditional and NN solvers via physics-driven architecture.
	% 	\item Zero-Shot learning in DeepRTE enables generalization to new boundary conditions without retraining.
	% 	\item Achieve high-precision solutions with minimal computational resources.
	% 	\item A mesh-free and accept multiple input functions neural operator.
	% \end{highlights}

	%% Keywords
	\begin{keyword}
		%% keywords here, in the form: keyword \sep keyword
		Neural operator \sep Radiative transfer
		equation \sep Attention \sep Pre-training

		%% PACS codes here, in the form: \PACS code \sep code

		%% MSC codes here, in the form: \MSC code \sep code
		%% or \MSC[2008] code \sep code (2000 is the default)

	\end{keyword}

\end{frontmatter}

%% Add \usepackage{lineno} before \begin{document} and uncomment
%% following line to enable line numbers
%% \linenumbers

%% main text
%%
\section{Introduction}\label{sec:intro}

The radiative transport equation (RTE) is a
fundamental equation in a wide variety of applications including
neutron transport~\cite{case1967linear,lewis1993computational},
atmospheric radiative transfer~\cite{marshak20063d}, heat
transfer~\cite{koch2004evaluation}, and optical
imaging~\cite{klose2002optical,tarvainen2005hybrid,joshi2008radiative}.
It describes the physical phenomena
that particles transport and interact with the background media
through absorption, emission, and scattering processes.

The steady-state RTE without source term writes
\begin{equation}\label{eq:rte}
	\bOmega \cdot \nabla I(\br, \bOmega) + \mut(\br) I(\br, \bOmega) =
	\frac{\mus(\br)}{{S_{d-1}}}\int_{{\sS^{d-1}}}   p(\bOmega, \bOmega^*)
	I(\br, \bOmega^*) \diff{\bOmega^*},
\end{equation}
where $I$ is the radiation intensity that depends on both the position
variable $\br\in D\subset\R^{d}$, angular variable
$\bOmega\in\sS^{d-1}$ with $\sS^{d-1}$ being a $d$-dimensional unit sphere;  $S_{d-1} = \dfrac{2\pi^{d/2}}{\Gamma(\frac{d}{2})}$ is the surface area of the unit sphere;
$\mut$ and $\mus$ are the total and scattering cross section
coefficients respectively; $p(\bOmega,\bOmega^*)$ is the phase function (or scattering function) that describes the probability of the particles moving with velocity $\bOmega^*$ scattered to velocity $\bOmega$.

The RTE is a high-dimensional differential-integral equation whose analytical solution is not available. As a result, many researchers have focused on developing numerical solutions to the RTE. In most physical applications, the spatial dimension is $3$ and the angular dimension is $2$ corresponding to $d=3$. Solving the RTE under these conditions is computationally expensive and is considered as one of the core tasks in high-performance computing.

Numerical methods for solving the RTE can be broadly classified into two categories:
1) Deterministic methods: These are based on PDE solvers and they typically involve discretizing Eq. \eqref{eq:rte} in both spatial and angular variables. Popular spatial discretization methods include: the diamond difference method (DD) ~\cite{lathrop1969spatial}, the upwind
scheme~\cite{klose2002optical}, and other finite difference schemes
(FD, TPFM)~\cite{han2014two}, finite element methods
(FEM)~\cite{martin1981phase,tarvainen2005coupled}, finite volume
methods (FVM)\cite{ren2004algorithm}, the Discontinuous Galerkin
methods
(DG)~\cite{cockburn2003discontinuous,wareing2001discontinuous,morel2005sn}.
The DD and FD methods are well-suited for structured grids, while FEM, FVM, and DG can be applied to unstructured grids.
For angular discretization, several methods are commonly used, including: the $P_n$
method~\cite{case1967linear}, the
FEM~\cite{martin1981phase,tarvainen2005coupled}, the discrete
ordinate method
(DOM)~\cite{lewis1993computational,koch2004evaluation,klose2002optical}, etc..
2) Monte Carlo methods: Monte Carlo methods are widely used due to their ability to handle complex geometries and their suitability for parallel computing \cite{lux2018monte}. However, these methods produce results with statistical noise that decreases slowly with the number of sampled particles, following the $1/\sqrt{N}$ rule \cite{spanier2008monte}. A very large number of samples is needed to obtain accurate results, highlighting a key trade-off between the method's flexibility and its computational cost.

To effectively solve the steady-state RTE \eqref{eq:rte}, both deterministic and Monte Carlo methods face their own unique challenges.
For deterministic methods, the scattering term (integral term) couples the density fluxes of different angles together. This coupling means that after discretization, one must solve a large and not overly sparse linear system (due to the integral term). For this large linear system, an iterative method is usually the only feasible approach. Therefore, developing a faster iterative strategy is of crucial importance.
For the Monte Carlo method, one of the main challenges is the statistical noise inherent in the method. It can be computationally expensive due to the need for a large number of samples to achieve accurate results.
Moreover, multiscale phenomena are important in some applications. When the computational domain is large compared to the mean free path ($1/\mut$), the RTE can exhibit quite different behaviors in different regions. For example, it may be more transport-like near the source and more diffusion-like after a significant number of scatterings. The multiscale parameters can affect the performance of iterative methods and accuracy of the Monte Carlo methods.
%More important, these regions of different behaviors can not be clearly defined as a priori knowledge. This poses a major difficulty for designing a fast iterative solver in the whole domain.
%For example, diffusion synthetic acceleration method, which is based on diffusion approximation of RTE, may not be effective in transport region, especially when the scattering is very anisotropic such as in forward-peaking case in optical imaging.

With the rapid development of machine learning, particularly deep learning, many researchers have been exploring the use of neural networks to solve PDEs. Two of the most well-known approaches are: 1) Physics-Informed Neural Networks (PINNs)~\cite{raissi2017physics, wang2021understanding}: These networks focus on incorporating physical constraints into the training process, specifically through the design of the loss function. PINNs are often used to approximate the solutions of PDEs by ensuring that the neural network output satisfies the underlying physical laws; 2) Neural Operators: These methods, such as DeepONet~\cite{lu2021learning,jin2022mionet,wang2021learning,lu2022multifidelity,zhu2023reliable} and Fourier Neural Operator (FNO)~\cite{li2020fourier,wen2022u,guibas2021adaptive}, aim to approximate the solution operator of PDEs. More precisely, they map initial conditions, boundary conditions, or coefficient functions directly to the solution of the PDE. Both approaches have their own strengths and limitations and they can be combined to outperform conventional numerical methods. These methods have been successfully applied to solve a variety of PDEs, including the Poisson equation, reaction-diffusion equations, Burgers' equation, and the Navier-Stokes equations. However, the application of neural networks to solve the RTE remains relatively underexplored.
In parallel, neural network methods have shown promising results for solving inverse problems in radiative transfer, particularly for two-dimensional optical tomography applications~\cite{fan2019solving}.

In this work, we focus on the neural operator approach to solve the RTE. The classical operator learning approaches DeepONet and FNO have shown success in various operator learning tasks, but they face fundamental limitations when directly applied to the RTE. DeepONet relies on predefined grid structures and its branch-trunk architecture struggles to capture the complex, high-frequency features characteristic of radiative transfer solutions.
Fourier Neural Operators face a couple of key challenges: First, calculating Fourier transforms can be computationally expensive, particularly when dealing with complex, high-dimensional data. Second, FNO struggles when the input parameters and the output solutions reside in different types of spaces. This is because it's difficult to define a meaningful Fourier transform between such distinct domains. These architectural constraints limit both methods' ability to learn the underlying solution operator for radiative transfer problems.

To overcome above challenges, we introduce DeepRTE, an innovative framework that integrates a pre-trained, attention-based neural network while inherently respecting the physical principles governing the transport process.
DeepRTE is an operator learning approach that maps boundary conditions, cross section ($\mu_s, \mu_t$) and the scattering kernel $p(\bOmega,\bOmega^*)$ directly to the solution. This end-to-end design allows for efficient inference and robust generalization across different problem settings.
DeepRTE offers several key advantages:
\begin{enumerate}
	\item \emph{Fewer Parameters, Greater Precision}: DeepRTE demonstrates that smaller, physics-informed models can outperform larger, purely data-driven frameworks such as multiple-input operators (MIO)~\cite{jin2022mionet}, achieving robust generalization and accuracy with significantly fewer parameters. By embedding the physical principles of RTE, DeepRTE reduces computational complexity while enhancing robustness and interpretability. This approach proves that efficient models, when guided by domain knowledge, can deliver superior results in scientific computing.
	\item \emph{Zero-Shot Learning Capability}: DeepRTE exhibits
	      exceptional zero-shot learning ability, allowing it to
	      generalize effectively to new boundary conditions and configurations
	      without the need for additional retraining.
	      This capability is particularly useful for real-world applications
	      where previously unseen conditions can be commonly encountered.
	\item \emph{Physical Interpretability}: DeepRTE integrates mathematical semi-analytical formulations and physical insights through attenuation and scattering modules, utilizing the Green's function integral to ensure that the learned operator remains linear and physically interpretable. This alignment with physical principles enhances the model's reliability and applicability across various scenarios.
\end{enumerate}

The paper is organized as follows.
In Section~\ref{sec:preliminaries}, we provide a brief overview of the semi-analytical form of RTE solution operator.
In Section~\ref{sec:architecture} and Section~\ref{sec:training}, we present the DeepRTE framework, detailing the architecture, training process, and key components.
In Section~\ref{sec:experiments}, we evaluate the performance of DeepRTE on a series of RTE problems, comparing it with traditional numerical methods and other operator learning frameworks.
In Section~\ref{sec:ablation-study}, we investigate our architecture design choices through ablation studies and compare with baseline model. Finally, we conclude with a summary of our findings and discuss potential future research directions.

\paragraph{Code and data availability}
All code for this work is openly available on GitHub under the repositories \href{https://github.com/mazhengcn/deeprte}{deeprte}\footnote{\url{https://github.com/mazhengcn/deeprte}} and \href{https://github.com/mazhengcn/rte-dataset}{rte-dataset}\footnote{ \url{https://github.com/mazhengcn/rte-dataset}} under an open-source license.
The codebase includes the implementation of DeepRTE models, configuration files, and utilities for dataset generation, model training, evaluation and experimentations.
Additionally, we utilized widely adopted open-source libraries, including JAX~\cite{jax2018github}, Flax~\cite{flax2020github}, Optax~\cite{deepmind2020jax}, TensorFlow Datasets and so on.
Pre-trained models and datasets are also accessible on Hugging Face via \href{https://huggingface.co/mazhengcn/deeprte}{mazhengcn/deeprte}\footnote{\url{https://huggingface.co/mazhengcn/deeprte}} and \href{https://huggingface.co/datasets/mazhengcn/rte-dataset}{mazhengcn/rte-dataset}\footnote{\url{https://huggingface.co/datasets/mazhengcn/rte-dataset}}.
All algorithms, models, and results reported in this paper can be fully reproduced using the provided repositories and resources.

\section{Analytical structure of the solution operator}\label{sec:preliminaries}
%We aim to explore fundamental properties of the solution operator for stationary RTE for scattering atmosphere ($q\equiv 0$) with inflow boundary condition.
%These properties will serve as the foundation for constructing our neural network model in the subsequent sections.
Let
\begin{equation}\label{eq:boundary}
	\Gamma_{\pm} := \{(\br,\bOmega) \mid \br\in\partial
	D,\;\bOmega\in\sS^{d-1},\;\mp\bn(\br)\cdot\bOmega<0 \},
\end{equation}
where $\bn(\br)$ is the outer normal direction of $D$ at
$\br\in\partial D$. $I(\br,\bOmega)$ satisfies the
following RTE with inflow boundary conditions:
\begin{equation}\label{eq:rte-with-bc}
	\begin{aligned}\bOmega \cdot \nabla I(\br,\bOmega) + \mut(\br) I(\br,\bOmega) & = \frac{\mus(\br)}{S_{d-1}} \int_{\sS^{d-1}} p(\bOmega,\bOmega^*) I(\br,\bOmega^*)\diff{\bOmega^*}, &  & \text{in } D\times\sS^{d-1}, \\
               I|_{\Gamma_{-}}(\br,\bOmega)                                   & = I_{-}(\br,\bOmega),                                                                               &  & \text{on }\Gamma_{-},
	\end{aligned}
\end{equation}
where $I_{-}$ is a given function on $\Gamma_{-}$ and $S_{d-1} = \dfrac{2\pi^{d/2}}{\Gamma(d/2)}$ is the surface area of the unit sphere $\sS^{d-1}$.

The coefficient functions in this boundary value problem are the total cross section $\mu_t$, the scattering cross section $\mu_s$, and the scattering kernel $p$. We also define the absorption cross section $\mu_a = \mu_t - \mu_s$. The kernel $p(\bOmega,\bOmega^*)$ represents the probability density of scattering from direction $\bOmega^*$ into direction $\bOmega$. We assume reciprocity (indistinguishable particles), i.e., $p(\bOmega,\bOmega^*) = p(\bOmega^*,\bOmega)$, and the normalization
\begin{equation}
	\frac{1}{S_{d-1}}\int_{\sS^{d-1}} p(\bOmega,\bOmega^*) \,\mathrm{d}\bOmega^* = 1
	\quad\text{for a.e. }\bOmega\in\sS^{d-1}.
\end{equation}
In many applications, one additionally assumes rotational invariance of the scattering kernel:
$p(\bOmega,\bOmega^*) = \tilde p(\bOmega\cdot\bOmega^*)$, so $p$ depends only on the scattering angle
$\theta = \arccos(\bOmega\cdot\bOmega^*)$. We adopt this assumption to simplify our description; however, our algorithm applies equally to general kernels.

In a more general settings where $\mut$ and $\mus$ not only depend on $\br$ but also depend on the velocity speed and angular: $\bm{v} = |\bm{v}|\bOmega$, solvability in $L^1$ and $L^\infty$ has been established in~\cite{case1963existence} under the sub-criticality conditions
\begin{equation}
	\mu = \frac{\mus}{\mut} < 1 - \nu, \quad \nu > 0,
\end{equation}
These imply that the scattering operator is a small perturbation of the differential operator on the left-hand side of~\eqref{eq:rte-with-bc} and contraction arguments apply.
Corresponding results in $L^p$ for $1 \leq p \leq \infty$ can be found in~\cite{agoshkov2012boundary,choulli1999inverse}.
Note that \textit{a priori} estimates for the solution derived under these conditions typically degenerate when $\nu\to 0$.
In~\cite{vladimirov1963mathematical}, solvability in $L^1$ was established provided that
\begin{equation}
	\mua = \mut - \mus \geq 0, \quad \text{and } \mut > 0.
\end{equation}
Existence results in $L^2$ were developed under these conditions in~\cite{choulli1999inverse,manteuffel1999boundary,egger2012mixed} by variational arguments.
Note that the assumption $\mut > 0$ excludes the presence of void regions and that the \textit{a priori} estimates again degenerate when $\mut \to 0$.
Based on monotonicity arguments, existence of solutions in $L^1$ was established in~\cite{pettersson2001stationary,falk2003existence}, without the strict positivity assumption on $\mut$.
For velocities with uniform speed $|\bm{v}|$ (as in our case $|\bm{v}|=1$), solvability in $L^2$ was established without lower bounds on $\mut$ in~\cite{egger2014stationary}.
While the previous results are based on some sort of contraction principle, it is possible to obtain existence of solutions also via compactness arguments and Riesz-Schauder or analytic Fredholm theory~\cite{stefanov2008inverse}.
These results, however, do not lead to computable a-priori bounds.

In this paper, we use the existence and uniqueness results in~\cite{egger2014lp} and~\cite{egger2014stationary}. The former work established the existence and uniqueness of solutions in $L^p$ for $1\leq p\leq\infty$ under a more general setting where $\mut$ and $\mus$ depend on both $\br$ and $\bm{v}$ which states that:
\begin{thm}\label{thm:existence-uniqueness-lp}
	See~\cite{egger2014lp} Assume the following conditions hold:
	\begin{enumerate}
		\item Let $\bm{v}=|\bm{v}|\bOmega\in V\subset\mathbb{R}^3$ be open and $D\subset\mathbb{R}^d$ is a bounded Lipschitz domain;
		\item $\mut:D\times V \to \mathbb{R}$ is non-negative and $\tau\mut\in L^\infty(D\times V)$. Here $\tau(\br,\bm{v})$ denotes the length of the line segment through $\br$ in direction $\bm{v}$ completely contained in $D$;
		\item $p: V\times V\to\mathbb{R}$ is non-negative and measurable and
		      \begin{equation}
			      \mua=\mut - \mus \geq 0.
		      \end{equation}
	\end{enumerate}
	Then, for all $1\leq p \leq\infty$ and all admissible data $I_{-}$ the radiative transfer problem~\eqref{eq:rte-with-bc} admits a unique solution $I$ that satisfies
	\begin{equation}
		\|\tau^{-\frac{1}{p}}I\|_{L^p(D\times V)} \leq e^{C_p}\|I_{-}\|_{L^p(\Gamma_{-};|\bn\cdot\bOmega|)}, \quad \text{with }
		C_p =\|\tau\mus\|_{L^\infty}.
	\end{equation}
	Morever, if we further assume $\mut>0$ and for some $\nu>0$,
	\begin{equation}
		\frac{\mus}{\mut} \leq 1 - \nu,
	\end{equation}
	then
	\begin{equation}
		\|\mut^{\frac{1}{p}}I\|_{L^p(D\times V)} \leq \nu^{-\frac{1}{p}}\|I_{-}\|_{L^p(\Gamma_{-};|\bn\cdot\bOmega|)},
	\end{equation}
	this allows to consider also the case $\mut\to\infty$ which may be important for asymptotic considerations.
\end{thm}
For our specific settings and numerical purpose, most often we will use the later theorem in which we assume $\mut$ and $\mus$ depend only on $\br$. The existence and uniqueness of solutions to~\eqref{eq:rte-with-bc} in $L^2(D\times\sS^{d-1})$ established in~\cite{egger2014stationary} states:
\begin{thm}\label{thm:existence-uniqueness-l2}
	Assume domain $D$ is bounded, let $\mua=\mut-\mus$, $\mus$, $p$ be non-negative and bounded measurable functions, i.e.,
	\begin{equation}
		\mua=\mut -\mus \geq 0, \quad \mus \geq 0, \quad \text{and} \quad \|\mua\|_{L^\infty(D)}<\infty, \|\mus\|_{L^\infty(D)}<\infty.
	\end{equation}
	and assume that
	\begin{equation}
		\frac{1}{S_{d-1}}\int_{\sS^{d-1}}p(\bOmega,\bOmega^*)\diff{\bOmega^*}=1, \quad\text{for a.e. } \bOmega\in \sS^{d-1}.
	\end{equation}
	Then for any $I_{-}\in L^2(\partial D\times\sS^{d-1})$ the RTE with boundary condition~\eqref{eq:rte-with-bc} has a unique solution $I\in L^2(D\times\sS^{d-1})$. Moreover, the \textit{a priori} estimate
	\begin{equation}
		\|\bOmega\cdot\nabla I\|_{L^2(D\times\sS^{d-1})}+\|I\|_{L^2(D\times\sS^{d-1})} \leq C \|I_{-}\|_{L^2(\partial D\times\sS^{d-1})},
	\end{equation}
	holds with a constant $C$ that only depends on $\text{diam}(D)$, $\|\mua\|_{L^\infty(D)}$ and $\|\mus\|_{L^\infty(D)}$.
\end{thm}

According to above existence and uniqueness results, if $\mut$, $\mus$ and $p$ satisfy the assumptions in~\ref{thm:existence-uniqueness-l2}, we can define the solution operator of \eqref{eq:rte-with-bc} as
%If the problem~\eqref{eq:rte-with-bc} is solvable, one can define the following solution operator
\begin{equation}
	\A: (I_{-};\mu_t,\mu_s,p) \mapsto I,
\end{equation}
that maps the inflow intensity $I_{-}$ on the boundary, the total
cross section coefficient $\mut$, the scattering coefficient $\mus$
and the scattering kernel $p$ to the solution $I$ inside the whole computational domain.
It is important to note that $\A$ is linear in $I_{-}$ but nonlinear in other function inputs.

Our \emph{Goal} is to develop an efficient and accurate method for numerically approximating the solution operator $\mathcal{A}$. A comprehensive understanding of the structure of the solution operator—even from a formal perspective—is crucial for constructing efficient and well-generalized deep neural networks to approximate $\mathcal{A}$. We will decompose $\mathcal{A}$ into solution operators that are computationally tractable, so that valuable insights for the efficient utilization of neural networks can be gained. This, in turn, enables us to design an architecture specifically tailored for solving RTE. We first review the analytical form of the solution operator, which can be expressed as a sequence expansion.

\subsection{Structure of solution operator}
%According to the standard approaches in ~
In order to express the solution operator in a more tractable form, we follow the standard approach in~\cite{case1963existence,choulli1999inverse} to convert RTE to a fixed-point equation.
Let us start by reformulating the radiative transfer problem as an equivalent integral equation
in the usual way. All the calculations below are formal and can be found in~\cite{egger2014lp,fan2019solving}.

We define $\tau_{\br,\bOmega}(s_1,s_2)$ be the \emph{optical depth} or \emph{optical thickness} from $\br-s_1\bOmega$ to $\br-s_2\bOmega$ along the characteristic line, i.e.,
\begin{equation} \label{eq:optical-depth}
	\tau_{\br,\bOmega}(s_1,s_2):=\int_{s_1}^{s_2} \mut(\br-s \bOmega) \diff{s},
\end{equation}
and $s_{-}(\br,\bOmega)$ be the distance of a particle moving in direction $\bOmega$ traveling from $\br$ to the domain boundary, i.e.,
\begin{equation}
	s_{-}(\br,\bOmega):= \inf \{s \geq 0 \mid \br-s\bOmega \in\partial D\}.
\end{equation}
One can then get the solution of the boundary value problem
\begin{equation}\label{eq:rte-sweep}
	\begin{aligned}\bOmega \cdot \nabla J_b(\br,\bOmega) + \mut(\br) J_b(\br,\bOmega) & = 0,                  &  & \text{in } D\times\sS^{d-1}, \\
               J_b|_{\Gamma_{-}}(\br, \bOmega)                                    & = I_{-}(\br,\bOmega), &  & \text{on }\Gamma_{-},
	\end{aligned}
\end{equation}
by
\begin{equation}\label{eq:attenuation-op}
	J_b(\br, \bOmega)= \J I_{-}(\br, \bOmega) = e^{-\tau_{\br,\bOmega}(0,s_{-}(\br,\bOmega))} I_{-}\left(\br-s_{-}(\br, \bOmega)
	\bOmega, \bOmega\right).
\end{equation}

Similarly, when $I$ is known, the solution of
\begin{equation}\label{eq:lifting}
	(\bOmega\cdot\nabla + \mut)J = \mus I, \quad  J|_{\Gamma_{-}} = 0,
\end{equation}
can be given by
\begin{equation}\label{eq:lifting-op}
	J=\cL I(\br, \bOmega) = \int_0^{s_{-}(\br, \bOmega)} e^{-\tau_{\br,\bOmega}(0,s)}\mus(\br-s\bOmega)I(\br-s\bOmega, \bOmega)\diff{s},
\end{equation}
where one can refer $\cL$ being the lifting operator, and we also define the scattering operator (see Remark~\ref{remk:mus}):
\begin{equation}\label{eq:scattering-op}
	\cS I(\br, \bOmega) = \frac{1}{S_{d-1}}\int_{\sS^{d-1}} p(\bOmega, \bOmega^*)I(\br, \bOmega^*) \diff{\bOmega^*}.
\end{equation}
By inversing the operator $\bOmega \cdot \nabla + \mut(\br)$, the boundary value problem~\eqref{eq:rte-with-bc} can then be seen to be equivalent to the following operator equation in integral form~\cite{case1963existence}:
%Denote by The RTE problem can then be seen to be equivalent to the following operator equation in integral form and
\begin{equation}\label{eq:rte-op-form}
	I = \cL\cS I + \J I_{-},
\end{equation}
which is a Fredholm integral equation of the second kind \cite{choulli1999inverse,egger2014lp,fan2019solving}.
One can see that $\cL\cS I$ is the contribution from the scattering process and $\J I_{-}$ takes into account the boundary conditions.

To establish the existence of a unique fixed point, one must demonstrate that $\cL\cS$ is a contraction mapping in an appropriate function space. This result is proven in~\cite{egger2014lp} under the assumptions stated in Theorem~\ref{thm:existence-uniqueness-lp}.

The solution operator can also be expressed through the corresponding fixed-point iteration
\begin{equation}\label{eq:fixed-point-iteration}
	I^{m+1} = \cL\cS I^m + \J I_{-}.
\end{equation}
When the initial guess is set to $I^0=\J I_{-}$, this iteration takes the same form as the well-known classical source iteration method~\cite{adams2002fast}. Furthermore, as established in~\cite{egger2014lp}:
\begin{thm}
	Under the assumptions in Theorem~\ref{thm:existence-uniqueness-lp}, for all $1\leq p \leq\infty$, we have
	\begin{equation}
		\rho_p(\cL\cS):=\lim_{n\to\infty}\sqrt[n]{\|(\cL\cS)^n\|_{L^p(D\times V;\tau^{-1})}} \leq 1 - e^{-C_p} < 1.
	\end{equation}
\end{thm}
This shows that the spectral radius of the fixed-point operator $\cL\cS$ is uniformly bounded away from one, ensuring convergence of the iteration to the solution of the RTE.
Consequently, the solution operator admits the series representation
\begin{equation}\label{eq:soln-op-series}
	\A[\mu_t,\mu_s,p] I_{-} = \sum_{n=0}^\infty(\cL\cS)^n\J I_{-}.
\end{equation}

This iterative structure provides valuable insight for constructing neural network architectures, as the layered structure of neural networks naturally parallels the iterative nature of the solution operator.
\begin{remark}\label{remk:mus}
	In the literature, the scattering cross section $\mus$ is typically included in the scattering operator $\mathcal{S}$, not the lifting operator $\mathcal{L}$. However, we position $\mus$ within $\mathcal{L}$ to accommodate the specific architecture of our model. This choice arises because the lifting operator $\mathcal{L}$ is designed to inherently capture characteristic properties, whereas $\mathcal{S}$ is not. This distinction allows our framework to effectively integrate $\mus$ while maintaining consistency with the model's structural requirements, see Alg~\ref{alg:attenuation-module} and Alg~\ref{alg:scattering-module}.
\end{remark}
\subsection{Green's function}
In this subsection, we study the distributional kernel of the solution operator $\A$. The following calculations are formal; for a rigorous treatment, we refer the reader to~\cite{choulli1999inverse}, Section~3.

Let us denote $G(\br,\bOmega,\br',\bOmega')$ be the solution (in the distribution sense) to
\begin{equation}\label{eq:rte-greens-function}
	\begin{aligned}
		(\bOmega \cdot \nabla + \mu_t)G & =  \mus\cS G, \quad                                     &  & \text{on } D\times\sS^{d-1}, \\
		G|_{\Gamma_{-}}                 & = \delta_{\{\br'\}}(\br)\delta(\bOmega-\bOmega'), \quad &  & \text{on }\Gamma_{-},
	\end{aligned}
\end{equation}
where $(\br',\bOmega')\in\Gamma_{-}$ are considered as parameters; $\delta_{\{\br'\}}$ denotes a distribution on $\partial D$ defined by
\begin{equation}\label{eq:delta-func}
	(\delta_{\{\br'\}},\phi) = \int_{\partial D}
	\delta_{\{\br'\}}(\br)\phi(\br)\diff{\br} = \phi(\br'),\qquad
	\forall \phi(\br)\in C_c^\infty(\partial D);
\end{equation}
$\delta(\bOmega)$ represents the standard Dirac delta function in $\sS^{d-1}$.
$G$ is referred as the Green's function of the RTE, representing the system's response to a point source at the boundary. The solution $I$
to~\eqref{eq:rte-with-bc} can be represented in the following integral form:
\begin{equation} \label{eq:rte-op}
	I(\br,\bOmega) = \int_{\Gamma_{-}}
	G(\br,\bOmega,\br',\bOmega')I_{-}(\br',\bOmega')\diff{\br'}\diff{\bOmega'}
	= \A[\mu_t,\mu_s,p] I_{-}(\br,\bOmega).
\end{equation}

The above formulation elucidates that if one can determine
$G(\br,\bOmega,\br',\bOmega')$ for all $(\br',\bOmega')$ at the
boundary, the solution can be readily obtained through the
integral operator in~\eqref{eq:rte-op}. This insight informs
the structure of our proposed deep neural network operator.
%To this end, we consider the solution to~\eqref{eq:rte-with-bc} under a specific boundary condition:
%\begin{equation}
%  I_{-}(\br, \bOmega) = \delta_{\{\br'\}}(\br)\delta(\bOmega-\bOmega'),
%\end{equation}

According to~\eqref{eq:fixed-point-iteration} and~\eqref{eq:soln-op-series}, $G$ can be obtained by the following fixed-point iterations:
\begin{equation}\label{eq:greens-function-iterations}
	\begin{aligned}
		 & G^{m+1} =\cL\cS G^m+G^0= \cL\cS G^m+\J\left(\delta_{\{\br'\}}(\br)\delta(\bOmega-\bOmega')\right).
	\end{aligned}
\end{equation}
or expressed as a sequence expansion:
\begin{equation}\label{eq:greens-function-series}
	G(\br,\bOmega,\br',\bOmega') = \sum_{n=0}^\infty(\cL\cS)^n\J
	\left(\delta_{\{\br'\}}(\br)\delta(\bOmega-\bOmega')\right),
\end{equation}
Then the structure of the solution operator can be expressed as
\begin{equation} \label{eq:rte-op1}
	\begin{aligned}
		\A[\mu_t,\mu_s,p] I_{-}(\br,\bOmega) & = \int_{\Gamma_{-}}
		G(\br,\bOmega,\br',\bOmega')I_{-}(\br',\bOmega')\diff{\br'}\diff{\bOmega'}               \\
		                                     & = \sum_{n=0}^\infty\int_{\Gamma_{-}}(\cL\cS)^n \J
		\left(\delta_{\{\br'\}}(\br)\delta(\bOmega-\bOmega')\right)I_{-}(\br',\bOmega')\diff{\br'}\diff{\bOmega'}.
	\end{aligned}
\end{equation}

\begin{remark}If $ \mu_t $, $ \mu_s $, and $ p $ are given and fixed, a naive idea is using a MLP to construct the Green's function $ G(\mathbf{r}, \Omega, \mathbf{r}', \Omega') $. However, $ \mu_t $ and $ \mu_s $ are spatially dependent functions that can sometimes be discontinuous, while $ p(\Omega, \Omega^\ast) $ is defined on $ \mathbb{S}^{d-1} \times \mathbb{S}^{d-1} $. The dependence of the Green's functions on these parameter functions can be very complex.
	Therefore, to train a solution operator that can be applied to all $ \mu_t $, $ \mu_s $, and $ p $ in the admissible set, one must consider the particular way in which the solution depends on the parameter functions within the structure of the Neural Operator.
\end{remark}
%Under the assumption of smooth boundaries and considering all quantities in the distributional sense, the series in~\eqref{eq:greens-function-series} converges and the iteration~\eqref{eq:greens-function-iterations} gives $G = \lim_{n\to\infty}G^n$.

\section{Network architecture}\label{sec:architecture}
Our goal is to learn the following solution operator for RTE:
\begin{equation}\label{eq:solution-op}
	\ANN \approx \A: (I_{-}; \mut, \mus, p) \rightarrow I,
\end{equation}
on a fixed bounded domain $D$ and for simplicity we will drop $D$ dependence in the following discussion see Remark~\ref{remk:geometry-adaptation}.
More specifically, the solution operator $\ANN$ is represented by the integral of its distribution kernel, i.e., Green's function defined in~\eqref{eq:rte-op}
\begin{equation}\label{eq:greens-integral-nn}
	I(\br, \bOmega)\approx I^\text{NN}(\br,\bOmega) = \ANN[\mut,\mus,p] I_{-} = \int_{\Gamma_-} G^{\text{NN}}(\br, \bOmega, \br', \bOmega') I_{-}(\br', \bOmega') \diff{\br'} \diff{\bOmega'},
\end{equation}
where $G^{\text{NN}}$ is the Green's function parametrized by the deep neural networks.
\begin{figure}[H]
	\centering
	\includegraphics[width=1\textwidth]{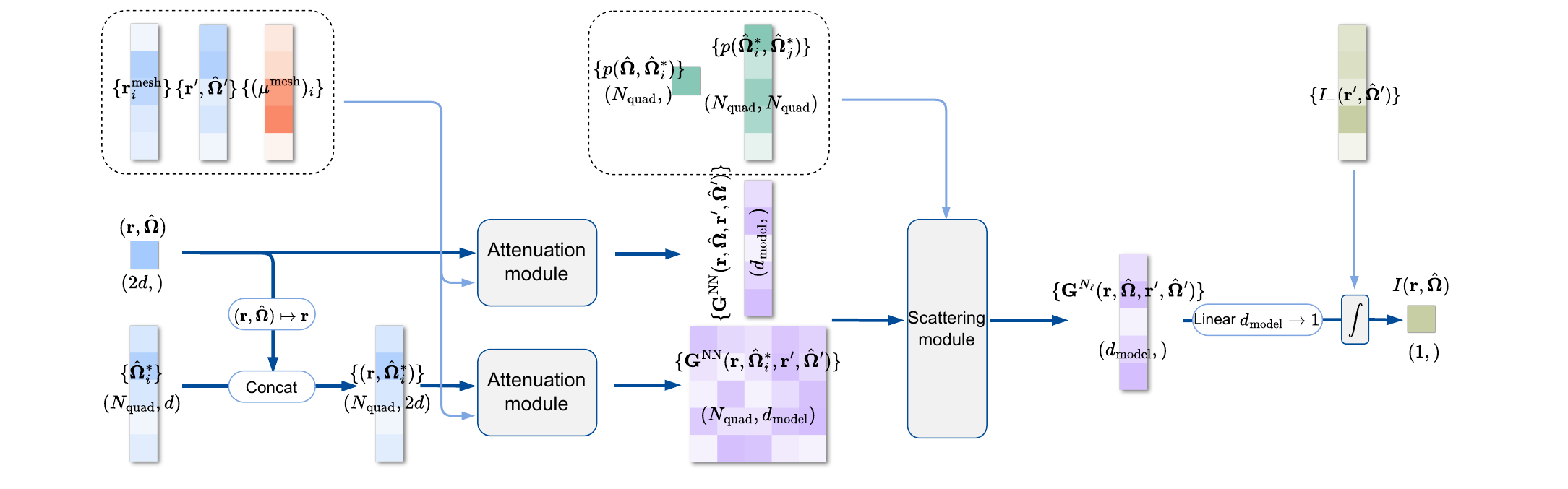}
	\caption{DeepRTE network architecture. The diagram illustrates the alterations in various inputs throughout the process. Attenuation Module processes phase coordinates and coefficients to approximate operators $\J$, $\cL$, followed by Scattering Module handling angular quadrature evaluations for operator iterations. Ultimately the Green's function $G^{\text{NN}}$ is multiplied with the boundary conditions and integrated over $\Gamma_-$  to compute radiation field.}\label{fig:arch-overview}
\end{figure}
The overall structure of DeepRTE is shown in Fig.~\ref{fig:arch-overview}, which is formed by modules inspired by the iterative process in~\eqref{eq:greens-function-iterations} to some finite steps.
The input features include:
\begin{itemize}\label{list:attenuation-inputs}
	\item $(\br,\bOmega)$: interior phase coordinate at which the solution needs to be evaluated;
	\item $(\br',\bOmega')$: the phase coordinates belong to $\Gamma_-$ that are used to compute the integration in~\eqref{eq:greens-integral-nn} by some given quadrature rule;
	\item $\{{(\mu_t^{\text{mesh}})}_i=\mu_t(\br^{\text{mesh}}_i)\}$, $\{{(\mu_s^{\text{mesh}})}_i=\mu_s(\br^{\text{mesh}})\}$, $\{\br^{\text{mesh}}_i\}$: spatial dependent total cross section and scattering cross section, and the corresponding coordinates of the spatial mesh points. Here the mesh is not necessary to be uniform, it can  be structured or unstructured.
\end{itemize}
These features are fed into the \textbf{Attenuation module}, which takes into account the dependence of the function $\mut$ and $\mus$.
The output of the Attenuation module is then combined with the following features:
\begin{itemize}
	\item  $\{p(\bOmega, \bOmega^*_i)\}$, $\{\omega_i\}$: phase/scattering function $p$ evaluated on any angle $\bOmega$, quadrature points $\bOmega^*_i$ and the corresponding quadrature weights $\omega_i$;
	\item  $\{p(\bOmega^*_i, \bOmega^*_j)\}$, $\{\omega_j\}$: phase/scattering function $p$ evaluated on quadrature points $(\bOmega^*_i, \bOmega^*_j)$, needed for the computation of residual connection in the network.
	      % \textcolor{blue}{ \item  $\{{(\mu_s^{\text{mesh}})}_i=\mu_s(\br^{\text{mesh}})\}$:  spatial dependent scattering coefficients.}
\end{itemize}
are fed into the \textbf{Scattering module} which simulates the iterative process \eqref{eq:fixed-point-iteration}.
Finally, the outputs of the scattering module is then projected to the scalar Green's function $G^{\text{NN}}$ by a linear layer, One can then multiply $G^{\text{NN}}$ with the boundary condition $I_{-}(\br',\bOmega')$, and integrate over the boundary phase coordinates $(\br',\bOmega')$ to obtain the solution $I$ on any phase point $(\br,\bOmega)$.

As a summary we put the whole process into the Alg.~\ref{alg:green-function}. The detailed description of the key ideas and components is provided in following subsections.
\begin{algorithm}[H]
	\caption{Green Function}\label{alg:green-function}
	\vspace{0.5em}
	\Def$\text{ GreenFunction}(\br,\bOmega,\br^{\prime},\bOmega^{\prime},\{\br_i^{\text{mesh}}\}, \{{(\mut^{\text{mesh}})}_i\},\{{(\mus^{\text{mesh}})}_i\},\{p(\bOmega,\bOmega^*_i)\},\{p(\bOmega^*_i,\bOmega^*_j)\}, \{\omega_j\})$:
	\begin{algorithmic}[1]
		\vspace{0.5em}
		\Comment{Output of attenuation module in latent space}
		\vspace{0.5em}
		\State$\bm{g}=\hyperref[alg:attenuation-module]{\text{AttenuationModule}}(\br,\bOmega,\br^{\prime},\bOmega^{\prime},\{\br_i^{\text{mesh}}\},\{{(\mut^{\text{mesh}})}_i\},\{{(\mus^{\text{mesh}})}_i\})$ \hfill $\bm{g}\in \mathbb{R}^{d_{\text{model}}}$
		\vspace{0.5em}
		\Comment{Ouput of the attenuation module at velocity quadrature points are also needed}
		\vspace{0.5em}
		\State$\{\bm{g}^*_j\}=\hyperref[alg:attenuation-module]{\text{AttenuationModule}}(\br,\{\bOmega_j^*\},\br^{\prime},\bOmega^{\prime},\{\br_i^{\mathrm{mesh}}\},\{{(\mut^{\mathrm{mesh}})}_i\}, \{{(\mus^{\mathrm{mesh}})}_i\})$ \hfill $\bm{g}^*_j\in \mathbb{R}^{d_{\text{model}}}$
		\vspace{0.5em}
		\Comment{Scattering module}
		\vspace{0.5em}
		\State$\bm{g}\gets\hyperref[alg:scattering-module]{\text{ScatterringModule}}(\bm{g}, \{\bm{g}^*_j\}, \{p(\bOmega,\bOmega^*_i)\},\{p(\bOmega^*_i,\bOmega^*_j)\}, \{\omega_j\})$ \hfill $\bm{g}\in \mathbb{R}^{d_{\text{model}}}$
		\vspace{0.5em}
		\Comment{Final output are projeted to the scalar Green's function}
		\vspace{0.5em}
		\State$g = \text{LinearNoBias}(\bm{g})$ \hfill $g\in\mathbb{R}$
		\vspace{0.5em}
		\Ret$g$
	\end{algorithmic}
\end{algorithm}

\begin{remark}\label{remk:geometry-adaptation}
	We remark that our current implementation of DeepRTE learns the solution operator for a specific geometry. Like other mainstream operator learning frameworks, when the geometry changes, the model theoretically requires retraining from scratch. However, several strategies can mitigate this limitation:

	One can train the model on a regular domain (e.g., the square domain used in our experiments) that encompasses all possible target geometries. The boundary conditions can then be appropriately extended to this regular domain. When applying to new geometries within this encompassing domain, the trained model can be used directly without retraining.

	Another approach involves using neural networks as encoder and decoder to map various geometries to a fixed latent domain and vice versa. The DeepRTE model is trained on this latent domain. For new geometries, transfer learning can fine-tune the encoder, decoder, and model components using limited data.
\end{remark}

\subsection{Attenuation module}\label{sec:attenuation-module}

The Attenuation module utilizes a neural network to provide a discrete representation of the Green's function of the operator $\J$ and $\cL$.
The Green's function for the operator $\J$ and $\cL$ can be unified as follows
\begin{equation}\label{eq:L-op}
	\int_0^{s_{-}(\br, \bOmega)} e^{-\tau_{\br,\bOmega}(0,s)}u(\br-s\bOmega, \bOmega)\diff{s},
\end{equation}
with $u(\br, \bOmega)=\mus(\br)I(\br, \bOmega)$ for $\cL$ and $u(\br, \bOmega)=I_-(\br-s_-\bOmega, \bOmega)\delta_{\{s\}}(s_-)$ for $\J$. This implies that it is important to have a representation of the Green's function along the characteristic line $\br-s \bOmega$. %as specified in~\eqref{eq:attenuation-op}.
In our architectural framework, % the attenuation module utilizes a neural network to approximate discretely approximate the initial state vector along characteristic lines. This approximation
along the characteristic line, the discrete representation of dimension $d_{\text{model}}$ is expressed as:
% \begin{equation}
%   \bG^{\text{NN}} = {(G^{0}, \cL G^{0}, \ldots, \cL^{d_{\text{model}}-1}G^{0})}^{\transpose} \in  \mathbb{R}^{d_\text{model}},
% \end{equation}
\begin{equation}
	\bG^{\text{NN}} =
	\begin{pmatrix}
		G(\br-s_1 \bOmega, \bOmega; \br', \bOmega') \\
		\vdots                                      \\
		G(\br-s_{d_{\text{model}}} \bOmega, \bOmega; \br', \bOmega')
	\end{pmatrix}\in \mathbb{R}^{d_{\text{model}}}.
\end{equation}

The parameter $d_{\text{model}}$ represents the dimension of the discrete representation, which is the number of sampling points along the characteristic line. Each component $G(\br-s_i \bOmega, \bOmega; \br', \bOmega')$ with $s_i\in (0,s_-(\br,\bOmega))$, corresponds to the Green's function's evaluation at a point on the characteristic line. It is important to note that the particular values of $s_i$ are not explicitly used in the NN construction. This representation effectively captures the essential attenuation characteristics while maintaining computational tractability through dimension reduction.

% We do not directly learn $\cL$ because $\cL$ inherently involves information along the characteristic line, and aggregating such information is computationally expensive. Repeated application of $\cL$ would significantly increase the computational load. Hence, we employ an MLP to directly integrate the effects of $\cL$, $\cL^2$, $\cL^3$, \ldots, so that subsequent iterative steps do not require explicit computation along the characteristic line.

The neural network output $\bG^{\text{NN}}$ depends on the phase space coordinates $(\br,\bOmega)$ and $(\br',\bOmega')$, along with $\mut$.
For implementation, we construct a neural network architecture with:
\begin{equation}\label{G_equation}
	\bG^{\text{NN}}(\br,\bOmega,\br^{\prime},\bOmega^{\prime}; \mut) =\text{MLP}\left(\br,\bOmega,\br^{\prime},\bOmega^{\prime},\tau_-^{\text{NN}}\right),
\end{equation}
where the optical depth $\tau_-^{\text{NN}}$ is computed by a specialized subnetwork called OpticalDepthNet. OpticalDepthNet learns how the total cross section $\mut$ determines the optical depth in complex ways, which takes into account how the function $\mu_t$ effects the Green's function.
The final Multi-layer perceptrons (MLP) combines all this information to predict the transport behavior.

The structure of the attenuation module is shown in Fig.~\ref{fig:attenuation-module} and Alg.~\ref{alg:attenuation-module}. The detailed description of $\tau_-^{\text{NN}}$, i.e., the OpticalDepthNet is provided in the following subsection.
\begin{figure}
	\centering
	\includegraphics[width=1\textwidth]{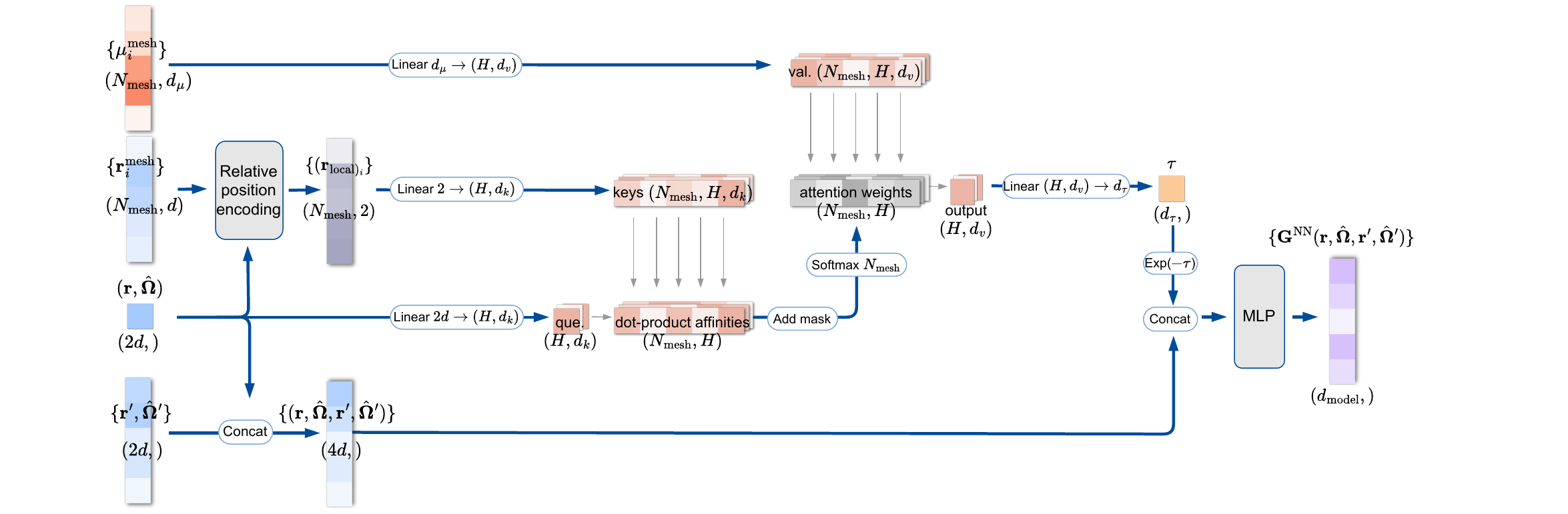}
	\caption{
		Attenuation module architecture. The network takes as input the spatial coordinates ($\br,\br'$), angular variables ($\bOmega,\bOmega'$), and cross sections ($\mut,\mus$). The OpticalDepthNet submodule first computes the optical depth $\tau_-^{\text{NN}}$ along the characteristic line. These features are then processed by an MLP to produce the truncated Green's function representation $\bG^{\text{NN}}\in\mathbb{R}^{d_{\text{model}}}$, which samples the radiation field at discrete points $\br-s_i\bOmega$ (circles along characteristic line). This architecture efficiently captures both local scattering effects and non-local attenuation while maintaining computational tractability through dimension reduction.
	}\label{fig:attenuation-module}
\end{figure}
\begin{algorithm}[H]
	\caption{Attenuation Module}\label{alg:attenuation-module}
	\vspace{0.5em}
	$
		\begin{aligned}
			\Def \text{ AttenuationModule}(
			 & \br,\bOmega,\br^{\prime},\bOmega^{\prime},\{\br_i^{\text{mesh}}\},\{{(\mut^{\text{mesh}})}_i\},\{{(\mus^{\text{mesh}})}_i\}, \\
			 & N_{\text{mlp}}=4,d_{\text{mlp}}=128,d_{\text{model}}=16):
		\end{aligned}$
	\begin{algorithmic}[1]
		\vspace{0.5em}
		\Comment{Optical depth network}
		\vspace{0.5em}
		\State$\tau_{-}^{\text{NN}} = \hyperref[alg:optical-depth-net]{\text{OpticalDepthNet}}(\br,\bOmega,\{\br_i^{\text{mesh}}\},\{{(\mut^{\text{mesh}})}_i\},\{{(\mus^{\text{mesh}})}_i\}, d_k=12, H=2)$
		\vspace{0.5em}
		\State$\bm{z} = \text{concat}(\br, \bOmega, \br^{\prime}, \bOmega^{\prime}, e^{-\tau_{-}^{\text{NN}}})$
		\vspace{0.5em}
		\Comment{MLP}
		\vspace{0.5em}
		\ForAll{$l\in[1,\ldots,N_{\text{mlp}-1}]$}
		\vspace{0.5em}
		\State$\bm{z}\gets\text{tanh}(\text{Linear}(\bm{z}))$ \hfill $\bm{z} \in \mathbb{R}^{d_{\text{mlp}}}$
		\vspace{0.5em}
		\EndFor%
		\vspace{0.5em}
		\Comment{Output projection}
		\vspace{0.5em}
		\State$\bm{g} = \text{Linear}(\bm{z})$ \hfill $\bm{g} \in \mathbb{R}^{d_{\text{model}}}$
		\vspace{0.5em}
		\Ret$\bm{g}$
	\end{algorithmic}
\end{algorithm}

\subsubsection{Optical depth network}
The Optical Depth Network is designed to capture the influence of $\mut$ along the characteristic line. Traditional methods for performing integration along the characteristic line often encounter significant challenges. For example, by noting \eqref{eq:optical-depth}, if we want to determine the optical depth which is defined by
\begin{equation}
	\tau_{-,t}(\br,\bOmega):=\tau(0, s_-(\br, \bOmega))=\int^{s_-(\br, \bOmega)}_{0} \mut(\br-s\bOmega) \diff{s},
\end{equation}
the numerical evaluation of the integration
along the characteristic line for each $(\br,\bOmega)$ is computationally expensive and sometimes even impossible for very complex spatial geometry.

By incorporating the \emph{multi-head attention along the characteristic line}, our proposed OpticalDepthNet, once trained, can approximate $\tau_{-,t}(\br,\bOmega)$ for a given geometry at any phase point $(\br,\bOmega)$. It uses multi-head attention as follows:
\begin{equation}\label{eq:optical-depth-net}
	\tau_{-,t}^{\text{NN}} = \text{OpticalDepthNet}\left(\br,\bOmega; \{\br^{\text{mesh}}_i\}, \{{(\mu_t^{\text{mesh}})}_i\}\right) = \text{MultiHead}(Q, K, V),
\end{equation}
where $\{\br^{\text{mesh}}_i\}$ are the spatial mesh where the total cross section $\mu_t(\br)$ takes values $\{{(\mu_t^{\text{mesh}})}_i\}$ on. It is important to note that $\{\br^{\text{mesh}}_i\}$ only considers the spatial dependence of $\mu_t$, but it does not necessarily coincide with the locations where the numerical solution $I$ is evaluated. This means that $\{\br^{\text{mesh}}_i\}$ can be very coarse, especially if  $\mu_t$ is, for example, piece-wise constant.
We will first describe the structure of the multi-head attention mechanism, define its inputs $Q$, $K$ and $V$ and then explain the motivation behind using multi-head attention in the next subsection.

In the context of deep learning, the inputs of multi-head attention $Q$, $K$ and $V$ are often called \emph{querys, keys, values} respectively. They are defined as follows:
\begin{equation}
	Q = (\br, \bOmega)\in \mathbb{R}^{1\times 2d}, \quad
	K =
	\begin{pmatrix}
		\vdots                               \\
		{(\br^\text{mesh}_{\text{local}})}_i \\
		\vdots
	\end{pmatrix}\in\mathbb{R}^{N_\text{mesh} \times 2},
	\quad
	V =
	\begin{pmatrix}
		\vdots                    \\
		{(\mu_t^{\text{mesh}})}_i \\
		\vdots
	\end{pmatrix}\in\mathbb{R}^{N_\text{mesh}\times 1},
\end{equation}
where $K$'s $i$-th entry ${(\br_{\text{local}})}_i=\left({(r_\text{local})}_i, {(\theta_\text{local})}_i\right)$ is the local coordinates on the characteristic line of the phase point $(\br,\bOmega)$, which is a straight line in direction $\bOmega$ and passing through the spatial point $\br$.
By projecting the $\br^\text{mesh}_i$ onto the characteristic line, one has
\begin{equation}
	\begin{aligned}
		{(\br^\text{mesh}_{\text{local}})}_i
		 & =\text{RelativePositionEncoding}(\br,\bOmega,\br^\text{mesh}_i)                        \\
		 & = \left({(r^\text{mesh}_\text{local})}_i, {(\theta^\text{mesh}_\text{local})}_i\right)
		= \left((\br-\br^{\text{mesh}}_i)\cdot\bOmega, \frac{(\br-\br^{\text{mesh}}_i)}{
			\| \br-\br^{\text{mesh}}_i\|}\cdot \bOmega\right).
	\end{aligned}
\end{equation}
Then the multi-head attention is quite standard in deep learning,
\begin{equation}
	\text{MultiHead}(Q, K, V) = \text{Concat}(\text{head}_1,\ldots,\text{head}_H)W^\tau,
\end{equation}
where for $h=1,\ldots,H$, the $h$-th head is computed as
\begin{equation}
	\begin{aligned}
		\text{head}_h = \text{Attention}(QW_h^Q, KW_h^K, VW_h^V)
		= \text{softmax}\left(\frac{(QW_h^Q){(KW_h^K)}^T}{\sqrt{d_k}} + M\right)VW_h^V,
	\end{aligned}
\end{equation}
with
\begin{equation}
	W_h^Q\in \R^{2d\times d_k}, \quad W_h^K\in \R^{2\times d_k}, \quad W_h^V\in \R^{1\times d_v}, \quad W^\tau\in\R^{Hd_v \times d_{\tau}},
\end{equation}
are the parameters to learn. $M\in\mathbb{R}^{N_\text{mesh}}$ is the mask matrix (not learnable) to make sure the attention only happens on those mesh points that are close to the characteristic line.

\paragraph{Multi-head attention on characteristic line}
The main idea of using multi-head attention is to approximate the integral of optical depth $\tau_-(\br,\bOmega)$ numerically. For any phase point $(\br,\bOmega)$, a simple numerical integration along the characteristic line yields,
\begin{equation}
	\tau_{-,t}(\br,\bOmega) \approx  \sum_{j}^{N\left(s_{-}(\br,\bOmega)\right)} w(\br, \bOmega;s_j)\mut(\br-s_j \bOmega),
\end{equation}
where $N(s_{-}(\br,\bOmega))$ is the number of quadrature points, and $w(\br, \bOmega; s_j)$ is the weight of the $j$-th quadrature point.
Notice here the integral is always in $s$ which means it is a one-dimensional integral along the characteristic line.

Ideally, one can do the ray tracing for every phase point $(\br,\Omega)$ and discretize the integration using the quadrature points along the characteristic line.
However, for each $(\br,\bOmega)$, different quadrature points $s_i$ are used, thus different weights $w(\br,\Omega; s_j)$ and $N\left(s_{-}(\br,\bOmega)\right)$ are employed. It is computationally expensive to determine the position of all quadrature points. On the other hand, in practice, usually we only have the values of $\mu_t$ on the spatial mesh $\{\br_i^{\text{mesh}}\}$, i.e., $\{{(\mu_t^\text{mesh})}_i\}$, so in order to get the value of $\mu_t(\br - s_i\bOmega)$, one needs to do the interpolation/projection. After interpolation/projection, the value of $\mu_t(\br - s_i\bOmega)$ can be expressed by a linear combination of $\{{(\mu_t^\text{mesh})}_i\}$ such that
\begin{equation}\label{eq:crcha}
	\mut(\br-s_j \bOmega)\approx  \sum_{i}^{N_\text{mesh}} \bm{1}_{\mathcal{C}_{\br,\bOmega}}(\rmesh_i) c(\br-s_j\bOmega, \br^{\text{mesh}}_i){(\mu_t^\text{mesh})}_i,
\end{equation}
where $c(\br-s_j\bOmega, \br^{\text{mesh}}_i)$ is the contribution of the $i$-th mesh point $\br_i^{\text{mesh}}$. The value of $c(\br-s_j\bOmega, \br^{\text{mesh}}_i)$ is determined by the relative position between $\br^{\text{mesh}}_i$ and the $j$-th quadrature point on the characteristic line.
As in Fig.~\ref{fig:mask}, the set $\mathcal{C}_{\br,\bOmega}$ is composed by the mesh points $\rmesh$ that are not far away from the characteristic line, which is determined by the some threshold $\delta$ such that
\begin{figure}[htbp]
	\centering
	\includegraphics[width=.8\textwidth]{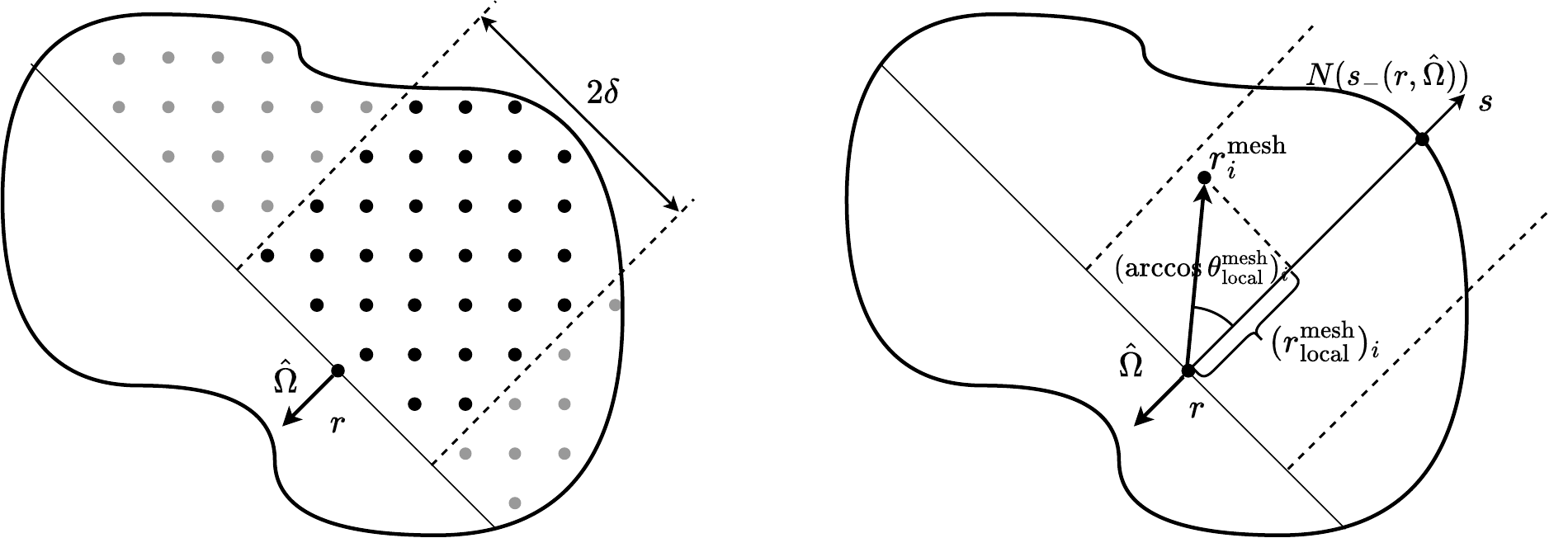}
	\caption{Mask and relative position embedding. Solid dots represent active grid nodes within the $\delta$-neighborhood of the characteristic line. These nodes provide spatial support for optical depth interpolation, thereby avoiding full-domain computation.}\label{fig:mask}
\end{figure}
\begin{equation}
	\mathcal{C}_{\br,\bOmega} = \left \{\rmesh_i \mid \text{distance}\left(\rmesh_i, \{\br - s\bOmega, s\in(0,s_-(\br,\bOmega))\}\right) < \delta\right \}.
\end{equation}
This indicates that only those mesh points near the characteristic line contribute.

More precisely, the relative position of $\br^{\text{mesh}}_i$ and the $j$th quadrature point $\br-s_j \bOmega$ can be determined by the projection of $\br^{\text{mesh}}_i$ on the characteristic line and the angle between the vector ${(\mu_t^\text{mesh})}_i-\br$ and the characteristic line (see Fig.~\ref{fig:mask}).
Noticing that $\|\bOmega \| = 1$ and the characteristic line is at the opposite direction of $\bOmega$, we have
\begin{equation}
	\begin{cases}
		{(r^\text{mesh}_\text{local})}_i      & = -(\br^\text{mesh}_i - \br)\cdot \bOmega,                          \\
		{(\theta^\text{mesh}_\text{local})}_i & = {(r^\text{mesh}_\text{local})}_i / \| \br^\text{mesh}_i - \br \|,
	\end{cases}
\end{equation}
see Alg.~\ref{alg:relative-position-encoding}, then
\begin{equation}
	c(\br-s_j\bOmega, \br^{\text{mesh}}_i) = c\left(s_j; {(r^\text{mesh}_\text{local})}_i, {(\theta^\text{mesh}_\text{local})}_i\right).
\end{equation}
\begin{algorithm}[H]
	\caption{Relative Position Encoding}\label{alg:relative-position-encoding}
	\vspace{0.5em}
	$\Def \text{ RelativePositionEncoding}(\br, \bOmega, \tilde{\br})$:
	\begin{algorithmic}[1]
		\vspace{0.5em}
		\Comment{Relative position}
		\vspace{0.5em}
		\State$\br_{\text{rel}} = \br - \tilde{\br}$
		\vspace{0.5em}
		\Comment{Local coordinates}
		\vspace{0.5em}
		\State$r_{\text{local}} =  \br_{\text{rel}} \cdot \bOmega$
		\vspace{0.5em}
		\State$\theta_{\text{local}} = r_{\text{local}}/\|\br_{\text{local}}\|$
		\vspace{0.5em}
		\Comment{Mask along characteristic}
		\vspace{0.5em}
		\State$m = -10^{10} \text{ if } s_\text{local} < 0 \text{ else } 0$
		\vspace{0.5em}
		\Ret$(r_\text{local},\theta_\text{local}), m$
	\end{algorithmic}
\end{algorithm}
Using above observation, one can approximate the optical depth as
\begin{equation}
	\begin{aligned}
		\tau_{-,t}(\br,\bOmega) & \approx \sum_j^{N\left(s_{-}(\br,\bOmega)\right)} w(\br, \bOmega;s_j)\sum_{i}^{N_\text{mesh}} \bm{1}_\mathcal{C}(\rmesh_i)c\left(s_j; {(r^\text{mesh}_\text{local})}_i, {(\theta^\text{mesh}_\text{local})}_i\right){(\mu_t^\text{mesh})}_i         \\
		                        & =\sum_{i}^{N_\text{mesh}} \bm{1}_\mathcal{C}(\rmesh_i)\left(\sum_{j}^{N\left(s_{-}(\br,\bOmega)\right)} w(\br, \bOmega;s_j)c\left(s_j; {(r^\text{mesh}_\text{local})}_i, {(\theta^\text{mesh}_\text{local})}_i\right)\right){(\mu_t^\text{mesh})}_i \\
		                        & =\sum_{i}^{N_\text{mesh}} \underbrace{\bm{1}_\mathcal{C}(\rmesh_i)W(\br,\bOmega; {(r^\text{mesh}_\text{local})}_i, {(\theta^\text{mesh}_\text{local})}_i)}_{\text{attention weights}}\underbrace{{(\mu_t^\text{mesh})}_i}_{\text{values}},
		% & \approx \sum_{i}^{N_\text{mesh}} \left(\sum_{m}^{d_k} w_m(\br,\bOmega)w_m\left((r_\text{local})_i, (\theta_\text{local})_i\right)\right)\mut(\br_i^{\text{mesh}}).
	\end{aligned}
\end{equation}
Here, we first evaluate the inner summation on $j$ to determine the coefficient $W\left(\br,\bOmega; {(r_\text{local})}_i, {(\theta_\text{local})}_i\right)$, which indicates that the attention weights count for the contribution of each grid points to the integration in the optical depth.

At the continuous level, the weight do not depend on the choices of $N\left(s_{-}(\br,\bOmega)\right)$ and $s_j$. Similar property holds at the discrete level. This is important since the quadrature points are not given explicitly in the NN representation. The coefficient $W$ works as the correlation between the phase point $(\br,\bOmega)$ and the mesh point $\br_i^{\text{mesh}}$ with its local coordinate representation.
In another point of view, the coefficient $W$ can be learned by the multi-head attention mechanism, where the query is the phase point $(\br,\bOmega)$, the key is the local coordinate representation of the mesh point $\br_i^{\text{mesh}}$, and the value is the total absorption coefficient $\mu_t(\br_i^{\text{mesh}})$, given the following approximation:
\begin{equation}
	W\left(\br,\bOmega; {\left(r^\text{mesh}_\text{local}\right)}_i, {\left(\theta^\text{mesh}_\text{local}\right)}_i\right) \approx \sum_{m}^{d_k} \underbrace{q_m(\br,\bOmega)}_{\text{query}:\;QW^Q_h}\underbrace{k_m({(r^\text{mesh}_\text{local})}_i, {(\theta^\text{mesh}_\text{local})}_i)}_{\text{keys}:\;KW^K_h},
	% & \approx \text{softmax}\left((QW_h^Q) (KW_h^K)^T / \sqrt{d_k} + M\right),
\end{equation}
where $d_k$ is some hyperparameter to determine the dimension of the key space, and $w_m$ is the $m$-th head of the multi-head attention mechanism.
Then the indicator function $\bm{1}_\mathcal{C}(\rmesh_i)$ is implemented by the mask $M$ together with the commonly used softmax function:
\begin{equation}
	\text{softmax}_i(\bm{x}) = \frac{e^{x_i+m_i}}{\sum_j e^{x_j+m_j}},
\end{equation}
where $m_i$ is $-10^{10}$ if the mesh point $\rmesh_i$ is not in the set $\mathcal{C}$ and $0$ otherwise, this makes the attention weights get zero when the mesh point is far away from the characteristic line.

Combining all the ingredients above together we obtain the following Alg.~\ref{alg:optical-depth-net}:
\begin{algorithm}[H]
	\caption{Optical Depth Network}\label{alg:optical-depth-net}
	\vspace{0.5em}
	$\Def \text{ OpticalDepthNet}(\br, \bOmega, \{\br_i^{\text{mesh}}\}, \{{(\mut^{\text{mesh}})}_i\}, \{{(\mus^{\text{mesh}})}_i\}, d_k=12, H=2)$:
	\begin{algorithmic}[1]
		\vspace{0.5em}
		\Comment{Relative position}
		\vspace{0.5em}
		\State${(\br^\text{mesh}_{\text{local}})}_i, m_i = \hyperref[alg:relative-position-encoding]{\text{RelativePositionEncoding}}(\br$, $\bOmega,
			\br_i^{\text{mesh}})$
		\vspace{0.5em}
		\Comment{Input projections to heads}
		\vspace{0.5em}
		\State$\bm{q}^h = \text{LinearNoBias}((\br,\bOmega))$ \hfill $\bm{q}^h\in \mathbb{R}^{d_k},\; h\in \{1,\dots,H\}$
		\vspace{0.5em}
		\State$\bm{k}_i^h = \text{LinearNoBias}\left({(\br^\text{mesh}_{\text{local}})}_i\right)$ \hfill $\bm{k}_i^j\in \mathbb{R}^{d_k}$
		\vspace{0.5em}
		\State$\bm{v}_i^h = \text{LinearNoBias}\left(\text{concat}\left({(\mus^{\text{mesh}})}_i,{(\mus^{\text{mesh}})}_i\right)\right)$ \hfill $\bm{v}_i^j\in \mathbb{R}^{d_v}$
		\vspace{0.5em}
		\Comment{Attention}
		\vspace{0.5em}
		\State$a_{i}^h = \text{softmax}_i\left(\frac{1}{\sqrt{d_k}}{(\bm{q}^h)}^{\transpose}\bm{k}_i^h+m_i\right)$
		\vspace{0.5em}
		\State$\bm{\tau}^h = \sum_i a_{i}^h \bm{v}_i^h$ \hfill $\bm{\tau}^h\in\R^{d_v}$
		\vspace{0.5em}
		\Comment{Output projection}
		\vspace{0.5em}
		\State$\tau_{-} = \text{LinearNoBias}\left(\text{concat}_h(\bm{\tau}^h)\right)$ \hfill $\tau_{-}\in \mathbb{R}^{d_\tau}$
		\vspace{0.5em}
		\Ret$\tau_{-}$
	\end{algorithmic}
\end{algorithm}

\begin{remark}
	In the implementation, we use the whole mesh points $\{\br_i^{\text{mesh}}\}$ to compute the optical depth $\tau_-(\br,\bOmega)$, which is a more general and flexible approach. However, in practice, one can also use a subset of the mesh points to compute the optical depth $\tau_-(\br,\bOmega)$, which can be more efficient and faster.
\end{remark}

\subsection{Scattering module}\label{sec:scattering-module}
Based on the expansion of operator $G$ in~\eqref{eq:greens-function-series}, we propose the following structure for the Scattering module, as shown in Fig.~\ref{fig:scattering-module},
\begin{equation}
	\begin{aligned}
		 & \bG^{0} = \bG^{\text{NN}}(\br,\bOmega,\br^{\prime},\bOmega^{\prime}),                          \\
		 & \bG^{\ell}  = \text{ScatteringBlock}_s(\bG^{\ell-1}) + \bG^{0}, \quad \ell = 1,\dots,N_{\ell},
	\end{aligned}
\end{equation}
where the $\text{ScatteringBlock}_\ell$, as shown in Fig.~\ref{fig:scattering-block}, is defined as
\begin{equation}\label{eq:scattering-block}
	\text{ScatteringBlock}_\ell(\bm{G}) = \text{LayerNorm}\Big(\sigma\Big(\bm{W}^{\ell} \bm{S}^{\top} \bm{G} + \bm{b}^{\ell}\Big)\Big).
\end{equation}

According to Eq.~\eqref{eq:fixed-point-iteration},
the scattering block can be viewed as an approximation the operator $\cL\cS$. $\bm{W}^{\ell} \in \mathbb{R}^{d_{\text{model}}\times d_{\text{model}}}$ and $\bm{b}^{\ell} \in \mathbb{R}^{d_{\text{model}}}$ are learnable weights and bias. Multiplying $G$ by $
	\bm{S} = \big[\, w_i\, p(\bOmega, \bOmega_i^*) \,\big]_{j=1}^{d_{\text{quad}}}
$ gives a discrete approximation of the scattering operator in~\eqref{eq:scattering-op}. $\sigma(x)$ denotes the activation function which is $\tanh(x)$ in our implementation.

The normalization layer follows~\cite{ba2016layernormalization,xu2019understanding} with trainable affine parameters $\bm{\gamma},\bm{\beta} \in \mathbb{R}^{d_{\text{model}}}$
\begin{equation}
	\text{LayerNorm}(\bm{x}) = \frac{\bm{x}-\alpha\bm{1}}{\sigma}\cdot\bm{\gamma} + \bm{\beta}, \quad\alpha=\sum_{i=1}^{d_{\text{model}}}\frac{x_i}{d_\text{model}}, \quad\sigma=\sqrt{\sum_{i=1}^{d_{\text{model}}}\frac{{(x_i-\alpha)}^2}{d_\text{model}}}.
\end{equation}
Layer normalization improves the stability of the recursive computation in~\eqref{eq:greens-function-iterations}. If the weight matrix $\bm{W}^\ell$ is ill-conditioned or nearly singular, repeated multiplications can accumulate errors and reduce convergence quality. Using layer normalization helps prevent these issues and makes the training process more robust and efficient. There are two common approaches to applying LayerNorm:
\begin{itemize}
	\item \textbf{Post-Norm}: Normalization is applied after the scattering operator. This straightforward method performs well in shallow networks.
	\item \textbf{Pre-Norm}: Normalization is applied before the scattering operator. This approach can alleviate gradient issues in deeper networks.
\end{itemize}
After evaluating both implementations, we adopt Post-Layer Normalization in the ScatteringBlock for its enhanced stability and simplicity.

\begin{figure}[H]
	\centering
	\includegraphics[width=0.9\textwidth]{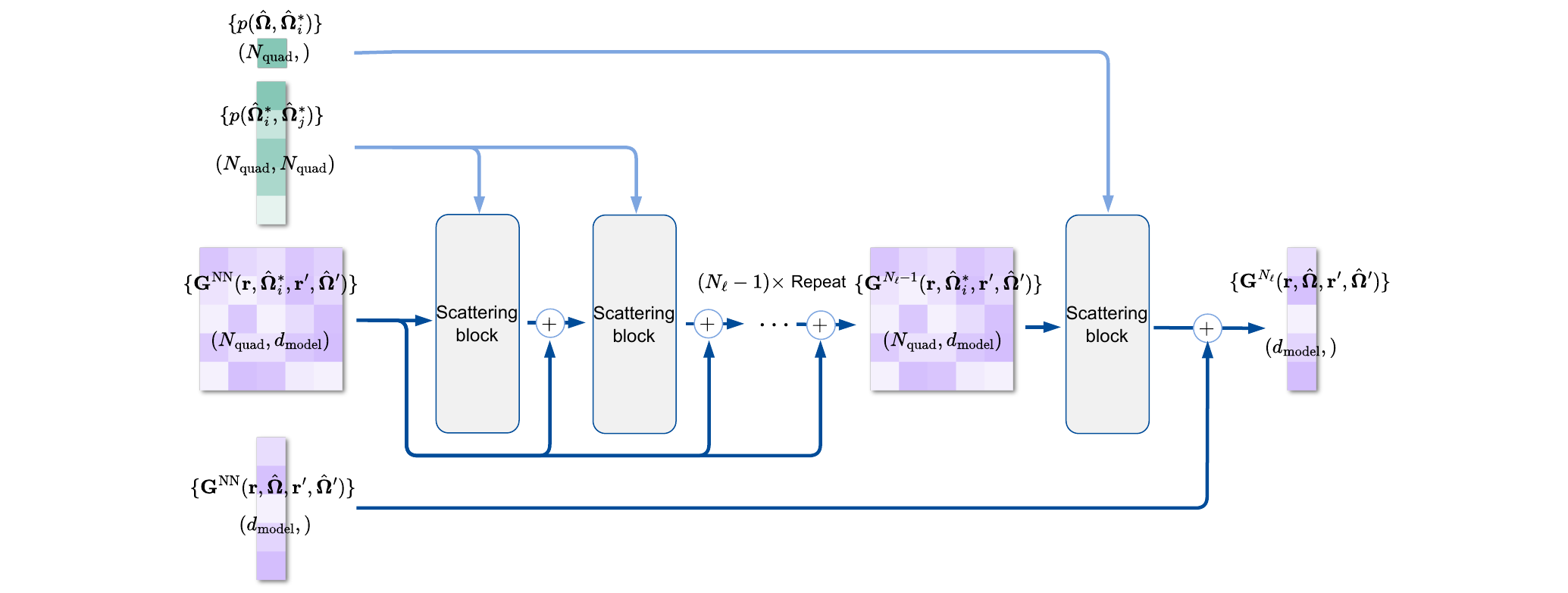}
	\caption{Scattering module. Stacked residual blocks with physics-informed $\cS$-approximation layers, employing tanh-activated operator transforms and adaptive layer normalization for stabilized learning.}\label{fig:scattering-module}
\end{figure}
\begin{figure}[H]
	\centering
	\includegraphics[width=1\textwidth]{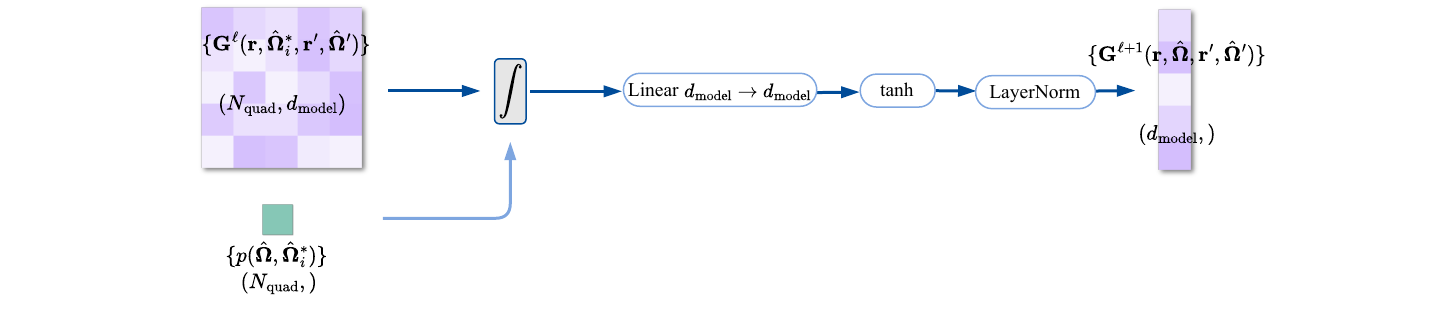}
	\caption{Scattering block.
		Representing simulating the result of the operator $\cS$ acting once.Including weighted summation of integral, linear projection and activation.}\label{fig:scattering-block}
\end{figure}

\noindent The detailed algorithm for Scattering module is provided in Alg.~\ref{alg:scattering-module}.
\begin{algorithm}[H]
	\caption{Scattering Module}\label{alg:scattering-module}
	\vspace{0.5em}
	$\Def \text{ ScatteringModule}(\bm{g}, \{\bm{g}^*_j\}, \{p(\bOmega,\bOmega^*_i)\},\{p(\bOmega^*_i,\bOmega^*_j)\}, \{w_j\}, N_s=2)$:
	\begin{algorithmic}[1]
		\vspace{0.5em}
		\Comment{Keep the initial value to be added at the each block}
		\vspace{0.5em}
		\State${(\bm{g}^*_j)}^{\text{init}} = \bm{g}^*_j$ \hfill ${(\bm{g}^*_j)}^{\text{init}}\in \mathbb{R}^{d_{\text{model}}}$
		\vspace{0.5em}
		\Comment{Scattering block with skip connections}
		\vspace{0.5em}
		\ForAll{$\ell\in[0,\ldots,N_\ell-2]$}
		\vspace{0.5em}
		\Statex{\hspace{1.5em}\color{Brown}\small\textit{\# \;Initial value is added at each block to simulate expansion}}
		\vspace{0.5em}
		\State$\bm{g}^*_i\gets{(\bm{g}^*_i)}^{\text{init}} + \hyperref[alg:scattering-block]{\text{ScatteringBlock}_\ell}(\{\bm{g}^*_j\},\{p(\bOmega^*_i,\bOmega^*_j), \{w_j\} \}) $ \hfill $\bm{g}^*_i\in \mathbb{R}^{d_{\text{model}}}$
		\vspace{0.5em}
		\EndFor%
		\vspace{0.5em}
		\State$\bm{g}\mathrel{+}=\hyperref[alg:scattering-block]{\text{ScatteringBlock}_{N_\ell-1}}(\{\bm{g}^*_i\},\{p(\bOmega,\bOmega^*_i)\}, \{\omega_i\})$ \hfill $\bm{g}\in \mathbb{R}^{d_{\text{model}}}$
		\vspace{0.5em}
		\Ret$\bm{g}$
	\end{algorithmic}
\end{algorithm}

\subsubsection{Scattering block}
Building on the series expansion in Eq.~\eqref{eq:soln-op-series}, we define a truncated state vector representation that captures the radiation field at discrete points along characteristic lines:

\begin{equation}
	\bG^{\ell} =
	\begin{pmatrix}
		G^\ell(\br-s_1 \bOmega, \bOmega; \br', \bOmega') \\
		\vdots                                           \\
		G^\ell(\br-s_{d_{\text{model}}} \bOmega, \bOmega; \br', \bOmega')
	\end{pmatrix}\in \mathbb{R}^{d_{\text{model}}},
\end{equation}
where $\bG^0$ is initialized by the Attenuation Module. This formulation yields an efficient recursive relation:
\begin{equation}
	G^{\ell+1} = \cL\cS G^\ell + G^0.
\end{equation}
$\cL\cS$ is composed of two physical processes: (1) angular scattering $\cS$; (2) lifting $\cL$ along the characteristic line. One can represent $\cL\cS$ into the following integral form:
\begin{equation}\label{LS}
	\cL\cS G^{\ell}(\br, \bOmega) = \int_0^{s_{-}(\br, \bOmega)} e^{-\tau(0,s')}
	\mus(\br-s' \bOmega) \underbrace{\frac{1}{S_d} \int_{\sS^{d-1}}
		p(\bOmega, \bOmega^*) G^{\ell}(\br-s' \bOmega, \bOmega^*)
		\diff{\bOmega^*}}_{\cS G^{\ell}(\br-s' \bOmega, \bOmega)} \diff{s'}.
\end{equation}

\paragraph{Numerical Implementation}
For computational tractability, we employ numerical quadrature:

\begin{equation}
	\cS G^\ell(\br-s' \bOmega, \bOmega) \approx \sum_{i=1}^{d_{\text{quad}}} w_i p(\bOmega, \bOmega_i^*) G^{\ell}(\br-s' \bOmega, \bOmega_i^*),
\end{equation}
where $d_{\text{quad}}$ quadrature points with weights $\{w_i\}$ approximate the angular integral.

In order to approximate the integration on the right hand side of \eqref{LS},  we use the discrete scattering approximation:
\begin{equation}
	\cL\cS G^{\ell}(\br, \bOmega) \approx \sum_{j=1}^{d_{\text{model}}} \tilde{w}_j^{\br, \bOmega} \mus e^{-\tau(0,s'_j)} \cS G^\ell(\br-s'_j \bOmega, \bOmega).
\end{equation}

%\paragraph{Learnable Matrix Representation}
To avoid recomputing characteristic line integrals, we parameterize the transport operator as a learnable weight matrix:
% \begin{equation}
%   \begin{aligned}
%     \cL\cS G^{\ell}(\br-s'_i\bOmega, \bOmega) & \approx
%     \sum_{j=1}^{d_{\text{model}}} \tilde{w}_j^{\br-s'_i \bOmega, \bOmega} \mus e^{-\tau(s'_i,s'_j)} \cS{G}^{\ell}(\br-(s'_j+s'_i) \bOmega, \bOmega)\\
%      & = \sum_{j=1}^{d_{\text{model}}} \underbrace{\tilde{w}_j^{\br-s'_i \bOmega, \bOmega} \mus e^{-\tau(0,s'_j)}}_{w^\cL_{ij}} \cS{G}^{\ell}(\br-s'_j \bOmega, \bOmega)\\
%     & = \bm{w}^\cL_{i} {\cS \bG}^{\ell},
%   \end{aligned}
% \end{equation}
\begin{equation}\label{eq:scattering-mus}
	\begin{aligned}
		\cL\cS G^{\ell}(\br-s'_i\bOmega, \bOmega) & \approx  \sum_{j=1}^{d_{\text{model}}} \underbrace{\tilde{w}_j^{\br-s'_i \bOmega, \bOmega} \mus e^{-\tau(0,s'_j)}}_{w^\cL_{ij}} \cS{G}^{\ell}(\br-s'_j \bOmega, \bOmega) \\
		                                          & = \bm{w}^\cL_{i} {\cS \bG}^{\ell},
	\end{aligned}
\end{equation}
% \begin{equation}
%   \begin{aligned}
%     \cL\cS G^{\ell}(\br-s'_i\bOmega, \bOmega) & \approx \sum_{j=1}^{d_{\text{model}}} \tilde{w}_j^{\br-s'_i \bOmega, \bOmega} \mus e^{-\tau(0,s'_j)} \cS{G}^{\ell}(\br-s'_j \bOmega, \bOmega)\\
%     & = \bm{w}^\cL_{i} {\cS \bG}^{\ell},
%   \end{aligned}
% \end{equation}
where $\bm{w}^\cL_{i} = (w^\cL_{i1}, w^\cL_{i2}, \cdots, w^\cL_{id_{\text{model}}}) \in \mathbb{R}^{d_{\text{model}}}$.
The complete system is represented through:

\begin{equation}
	\bG^{\ell} =
	\begin{pmatrix}
		G^{\ell}(\br-s_1 \bOmega, \bOmega) \\
		\vdots                             \\
		G^{\ell}(\br-s_{d_{\text{model}}} \bOmega, \bOmega)
	\end{pmatrix}, \quad
	\mathbf{W}^\ell =
	\begin{pmatrix}
		\bm{w}^\cL_1 \\
		\vdots       \\
		\bm{w}^\cL_{d_{\text{model}}}
	\end{pmatrix}.
\end{equation}

\paragraph{Neural Network Implementation}
In practice, we treat $\bm{W}^\ell$ as learnable parameters, allowing the network to automatically discover optimal coupling between spatial points during training. The scattering block simplifies to:

\begin{equation}\label{scattering_layer}
	\text{ScatteringBlock}_\ell(\bm{G}) = \text{LayerNorm}\Big(\sigma\Big(\bm{W}^{\ell} \bm{S}^{\top} \bm{G} + \bm{b}^{\ell}\Big)\Big),
\end{equation}
where $\bm{W}^{\ell} \in \mathbb{R}^{d_{\text{model}}\times d_{\text{model}}}$ encodes both the transport physics and numerical quadrature weights in a data-driven manner.

The detailed description of the Scattering block is provided in Alg.~\ref{alg:scattering-block}.
\begin{algorithm}
	\caption{Scattering Block}\label{alg:scattering-block}
	\vspace{0.5em}
	$\Def \text{ ScatteringBlock}(\{\bm{g}_i\}, \{p_i\}, \{\omega_i\}, d_{\text{model}}=16)$:
	\begin{algorithmic}[1]
		\vspace{0.5em}
		\Comment{Numerical integration as a weighted summation}
		\vspace{0.5em}
		\State$\bm{g}  = \sum_i \omega_i p_i \bm{g}_i$
		\vspace{0.5em}
		\Comment{Linear layer and activation}
		\vspace{0.5em}
		\State$\bm{g}\gets \text{tanh}\left(\text{Linear}( \bm{g})\right)$  \hfill $\bm{g}\in \mathbb{R}^{d_{\text{model}}}$
		\vspace{0.5em}
		\Comment{Layer normalization}
		\vspace{0.5em}
		\State$\bm{g} \gets \text{LayerNorm}(\bm{g})$
		\vspace{0.5em}
		\Ret$\bm{g}$
	\end{algorithmic}
\end{algorithm}

\begin{remark}\label{rmk:optical-depth-encoding}
	To encode $\mus e^{-\tau(0,s'_j)}$ into the scattering module, we construct the weights $\bm{W}^\ell$ in Eq.~\eqref{eq:scattering-block} as:
	\begin{equation}
		\mathbf{W}^\ell =
		\begin{pmatrix}
			\bm{w}^\cL_1 \\
			\vdots       \\
			\bm{w}^\cL_{d_{\text{model}}}
		\end{pmatrix}
		=
		\begin{pmatrix}
			\tilde{w}_j^{\br-s'_1 \bOmega, \bOmega}\circ \tau^{NN}_{-,s} \\
			\vdots                                                       \\
			\tilde{w}_j^{\br-s'_{d_\mathrm{model}} \bOmega, \bOmega} \circ \tau^{NN}_{-,s}
		\end{pmatrix},
	\end{equation}
	where $\tau^{NN}_{-,s} \approx ( \mus(\br-s'_1\bOmega) e^{-\tau(0,s'_1)}, \mus(\br-s'_2\bOmega) e^{-\tau(0,s'_2)}, \cdots, \mus(\br-s'_{d_{\mathrm{model}}}\bOmega) e^{-\tau(0,s'_{d_{\mathrm{model}}})} )$. We observe that $\tau(0,s'_j)$ can be modeled using an attention-based architecture similar to our Attenuation module. The full computational sequence - evaluating $\tau(0,s'_j)$, applying the exponential, and multiplying by $\mus$ - is approximated end-to-end by a neural network.

	In practice, this network is implemented within the Attenuation module alongside $\tau^{NN}_{-,t}$, with the actual output of $\mathrm{OpticalDepthNet}$ being the concatenation $\tau_- = \mathrm{Concat}(\tau^{NN}_{-,t}, \tau^{NN}_{-,s})$. This design is formalized in Alg.~\ref{alg:optical-depth-net}, which defines $\tau_-$ as:
	\begin{equation}
		\tau_- = \mathrm{Concat}\left(\tau^{NN}_{-,t}, \tau^{NN}_{-,s}\right) = \mathrm{OpticalDepthNet}\left(\br,\bOmega; \{\br^{\mathrm{mesh}}_i\}, \left\{(\mu_t^{\mathrm{mesh}}, \mu_s^{\mathrm{mesh}})_i\right\}\right) \in \mathbb{R}^{d_{\tau}},
	\end{equation}
	where $d_\tau > d_{\mathrm{model}}$.

	% Alternatively, these physical components could be explicitly modeled by using an attention-based architecture, similar to our Attenuation module, to determine the values of $\mus$ and $e^{-\tau(0,s')}$ along the direction $\br-s_j \bOmega$. While this explicit decomposition can improve interpretability, we empirically found that: (1) it significantly increased the total number of model parameters, and (2) despite the added complexity, it offered no measurable improvement in performance. Our experiments showed that directly learning $\bm{W}^\ell$ achieves comparable accuracy with greater computational efficiency.
\end{remark}

\section{Training}\label{sec:training}
In this section, we detail the training framework for our neural network model designed to approximate the Green's function in the Radiative Transfer Equation (RTE). Our approach leverages a delta function dataset, in which boundary conditions are approximated by Gaussian functions. This approach gives the model flexibility to handle boundary conditions not seen in training. This ability is similar to pre-training: it transfers useful knowledge from one task to others. The following subsections describe the construction of training dataset, the loss function based on the mean square error, and the evaluation methodology using relative error metrics.

\subsection{RTE dataset features}
\begin{table}[ht]
	\centering
	\begin{tabular}{lll}
		\toprule
		Features \& Shape                                 & Description                                                           \\
		\midrule
		\texttt{phase\_coords}: $[N_{\text{coords}},2d]$  & Phase coordinates $(\br, \bOmega)$                                    \\
		\texttt{boundary\_coords}: $[N_{\text{bc}},2d]$   & Boundary coordinates $(\br^{\prime}, \bOmega^{\prime})$               \\
		\texttt{boundary\_weights}: $[N_{\text{bc}}]$     & Boundary weights $\omega_i$                                           \\
		\texttt{position\_coords}: $[N_{\text{mesh}},2d]$ & Mesh points $(\br^{\text{mesh}})$                                     \\
		\texttt{velocity\_coords}: $[N_{\text{quad}},2d]$ & Angular quadrature points  $\bOmega^{*}$                              \\
		\texttt{velocity\_weights}: $[N_{\text{quad}}]$   & Angular quadrature weights $\omega_i$                                 \\
		\texttt{boundary}: $[N_{\text{bc}}]$              & Boundary $I(\br^{\prime}, \bOmega^{\prime})$                          \\
		\texttt{mu}: $[N_{\text{mesh}},2]$                & Cross section $\mus(\br^{\text{mesh}})$ and $\mut(\br^{\text{mesh}})$ \\
		\texttt{scattering\_kernel}: $[N_{\text{quad}}]$  & Scattering kernel    $p(\bOmega, \bOmega^*)$                          \\
		\bottomrule
	\end{tabular}
	\caption{DeepRTE features. $N_\text{coords}$ denotes the number of phase-space coordinates, $N_\text{bc}$ represents the number of boundary coordinates, $N_\text{mesh}$ corresponds to the number of coefficient coordinates, $N_\text{quad}$ specifies the number of velocity quadrature points.}\label{tab:rte_features}
\end{table}

To construct training dataset, each training sample include the phase coordinates, the inflow boundary function $I_-$, scattering cross section $\mus$, total cross section $\mut$, the phase function $p$, and the intensity $I$. While these physical quantities are theoretically continuous functions, we approximate them through discrete sampling points in practice.
Table~\ref{tab:rte_features} summarizes the features included in each training sample.

% \textcolor{red}{Table~\ref{tab:rte_features} summarizes the features included in each training sample. Each sample consists of $N_\text{coords}$ phase-space coordinates, ${\{(\br_i, \bOmega_i)\}}_{i=1}^{N_\text{coords}}$, and the corresponding intensity values, ${\{I(\br_i, \bOmega_i)\}}_{i=1}^{N_\text{coords}}$.
% Additionally, the dataset includes $N_\text{bc}$ boundary coordinates, and the corresponding boundary weights.
% The boundary conditions are represented by the boundary intensity values, ${\{I(\br^{\prime}_i, \bOmega^{\prime}_i)\}}_{i=1}^{N_\text{bc}}$.
% The dataset also contains $N_\text{mesh}$ position coordinates, $N_\text{quad}$ velocity coordinates, and the corresponding velocity weights.
% The coefficient $\mu$ is represented by $N_\text{mesh}$ values, ${\{\mu(\br^{\text{mesh}}_i)\}}_{i=1}^{N_\text{mesh}}$, and the scattering kernel is represented by $N_\text{quad}$ values, ${\{p(\bOmega_i, \bOmega^*_i)\}}_{i=1}^{N_\text{quad}}$.(removed)}

\paragraph{Mesh-free sampling}
The training dataset is generated by a fully mesh-free sampling approach.
All feature counts $N_{\text{coords}}$, $N_{\text{bc}}$, $N_{\text{mesh}}$ and $N_{\text{quad}}$ are chosen independently.
This enables adjustments based on problem complexity and computational resources.

% \subsection{Pencil-beam problem}
\subsection{Delta function training dataset}
Due to the linearity of the RTE solution operator with respect to the inflow boundary condition,
any practical boundary condition can be approximated as a linear combination of delta-like functions. Therefore, by specializing our training dataset to approximate delta functions (e.g., using Gaussians), and leveraging the linearity of both the solution operator $\mathcal{A}$ and our DeepRTE approximation $\ANN$ with respect to $I_-$, a model accurately predicting inflow boundary conditions of the form~\eqref{eq:rte-greens-function} inherently generalizes to other boundary conditions.

Inspired by the particle methods~\cite{raviart1983analysis, chertock2017practical}, we construct the the approximation of boundary condition $I_-$ by a linear combination of Dirac distributions (delta functions).
\begin{equation}
  I_{-}(\br, \bOmega) \approx I_{-}^{N_\text{bc}}(\br,\bOmega):=\sum_{i,j}^{N_\text{bc}} w_{ij}\delta_{\{\br_i\}}(\br)\delta(\bOmega-\bOmega_j), \quad (\br,\bOmega) \in \Gamma_{-},
\end{equation}
where $w_{ij}$ are given coefficients. This can be
done, for instance, in the sense of measures. Namely, for any test function $\phi\in C^0_0(\Gamma_-)$, the inner product $(I_{-},\phi)$ should be approximated by
\begin{equation}
  \int_{\Gamma_{-}} I_{-}(\br,\bOmega)\phi (\br,\bOmega) \diff{\br} \diff{\bOmega} \approx (I_{-}^{N_\text{bc}},\phi) = \sum_{i,j}^{N_\text{bc}} w_{ij}\phi(\br_i,\bOmega_j).
\end{equation}
Based on this we observe that determining the weights $w_{ij}$ can be achieved through solving a standard numerical quadrature problem. However, the above result is not convinient when one is interested in obtaining point values of the computed solution. In this respect, it is more useful to associate the particle function $I_{-}^{N_\text{bc}}$ with a continuous function $I_{-,\sigma}^{N_\text{bc}}$ which approximates the boundary $I_{-}$ in a more classical sense.

The regularization of particle solution is usually performed by taking a convolution product with a mollification kernel (or, so-called, cut-oﬀ function), $\zeta_\sigma$, that after a proper scaling takes into account the tightness of the particle discretization, namely,
\begin{equation}
  I_{-}(\br, \bOmega) \approx I_{-,\sigma}(\br,\bOmega)
  := (I_-*\zeta_\sigma)(\br, \bOmega)
  = \int_{\Gamma_-} I_-(\br', \bOmega') \zeta_\sigma(\br-\br', \bOmega - \bOmega') \diff{\br'}\diff{\bOmega'},
\end{equation}
where $\zeta_\sigma \in C^0(\Gamma_-)\cap L^1(\Gamma_-)$ is a mollification kernel satisfying (see~\cite{raviart1983analysis}):
% \begin{itemize}\setlength\itemsep{0em}
%   \item $\zeta_\sigma \in C^\infty$ (smoothness)
%   \item it is compactly supported
%   \item $\int \zeta_\sigma(\br,\bOmega) \diff{\br} \diff{\bOmega} = 1$ (normalization)
%   \item \highlight[id=ZYK]{$\zeta_\sigma(\br, \bOmega) = \frac{1}{\sigma^{2d-1}}\zeta_{\sigma}(\frac{\br, \bOmega}{\sigma})$}
% \end{itemize}
\begin{equation}
  \zeta_\sigma(\br, \bOmega):=\frac{1}{\sigma^{2d-1}}\zeta_{\sigma}\left(\frac{\br}{\sigma}, \frac{\bOmega}{\sigma}\right), \quad \quad  \int_{\Gamma_{-}} \zeta_\sigma(\br,\bOmega) \diff{\br} \diff{\bOmega} = 1.
\end{equation}
and $\sigma$ denotes a characteristic length of the kernel, which in our case
\begin{equation}
  \zeta_\sigma(\br-\br', \bOmega-\bOmega') := \zeta^{\sigma_{\br}}_{\{\br'\}}(\br)\zeta^{\sigma_{\bOmega}}(\bOmega-\bOmega'), \quad \text{with } \sigma:=\max{(\sigma_{\br},\sigma_{\bOmega})},
\end{equation}
The particle approximation of $I_-$ is then defined as
\begin{equation}\label{eq:particle-approx}
  I_-(\br, \bOmega) \approx
  I^{N_{\text{bc}}}_{-,\sigma}(\br,\bOmega)
  := (I_-^{N_\text{bc}}*\zeta_\sigma)(\br, \bOmega)
  = \sum_{i,j}^{N_\text{bc}} w_{ij}\zeta_\sigma(\br-\br_i,\bOmega-\bOmega_j),
\end{equation}
where $N_{\text{bc}}$ denotes the number of quadrature points in boundary $\Gamma_-$.

The accuracy of the particle method will thus be related to the moments of $\zeta_\sigma$ that are being conserved, and we say that the kernel is of order $k$ when:
\begin{equation}\label{eq:order_k}
  \left \{ % tex-fmt: skip
  \begin{aligned}
    &\int_{\Gamma_{-}} \zeta_\sigma(\br, \bOmega) \diff{\br} \diff{\bOmega} = 1, \\
    &\int_{\Gamma_{-}} \br^{\alpha_i}\bOmega^{\alpha_j} \zeta_\sigma(\br, \bOmega) \diff{\br} \diff{\bOmega} = 0, \quad \text{for all }\alpha = \alpha_i + \alpha_j\text{ such that } 1\leq |\alpha|\leq k-1, \\
    &\int_{\Gamma_{-}} \left|(\br,\bOmega)^T\right|^k |\zeta_\sigma(\br, \bOmega)| \diff{\br} \diff{\bOmega} < \infty.
  \end{aligned}
  \right. % tex-fmt: skip
\end{equation}
According to~\cite{chertock2017practical}, we have the following error estimate for the boundary condition approximation.

\begin{thm}\label{thm:bc-error-estimate}
  Let~\eqref{eq:order_k} are satisfied for some integer $k\geq 1$, Assume that $\zeta \in W^{m,p}(\Gamma_-)\cap W^{m,1}(\Gamma_-)$ for some integer $m>d$. Then, for $I_- \in W^{l,p}(\Gamma_-)$ with $l=\max{(k, m)}$. There exists a positive constant $C$, s.t.
  \begin{equation}\label{eq:bc_error_estimate}
    \|I_- - I^{N_{\text{bc}}}_{-,\sigma}\|_{L^p(\Gamma_-)}\leq C\left\{ \sigma^k \|I_-\|_{k,p,\Gamma_-} + \left(\frac{h}{\sigma}\right)^m \|I_-\|_{m,p,\Gamma_-}\right\},
  \end{equation}
  where $h > 0$ is the size of non-overlapping cubes covering $\Gamma_-$.
\end{thm}

The two terms in the above estimate may be balanced by choosing an appropriate size of $\sigma$.
Intuitively, it is clear that if the smoothing parameter $\sigma$ is too small in comparison to the minimal distance between particles, the approximate function will vanish away from the $\sigma$-neighborhood of the particles and is thus irrelevant.
On the other hand, large values of $\sigma$ will generate unacceptable smoothing errors.
Theoretically $\sigma$ is chosen so that the smoothing error and the
discretization error are of the same order and it is common to take $\sigma\sim O(\sqrt{h})$.
More discussions on the choice of $\sigma$ can be found in~\cite{raviart1983analysis, chertock2017practical}.

\textbf{Delta function training dataset.} So our strategy is to train the DeepRTE operator $\ANN$ with the following boundary conditions:
\begin{equation}
  \left\{\zeta_\sigma \mid \zeta_\sigma(\br-\br_i, \bOmega-\bOmega_j)
  = \zeta^{\sigma_{\br}}_{\{\br_i\}}(\br)\zeta^{\sigma_{\bOmega}}(\bOmega-\bOmega_j), \;(\br_i, \bOmega_j) \in \Gamma_-, \, i,j \in \{0,\ldots,N_{\text{bc}}-1\},\; (\sigma_{\br},\sigma_{\bOmega})\sim \Sigma\right\},
\end{equation}
where $\Sigma$ is some distribution. We call this the delta function training dataset.

According to the universal approximation capability of neural networks~\cite{hornik1989multilayer, weinanpriori, weinan2019barron}, we assume the generalization error of well-trained $\ANN_\theta$ on the delta function training dataset is small, i.e., assume there exits a parameter set $\theta^*$ of neural network such that
\begin{equation}
  \|\ANN_{\theta^*}\zeta_\sigma - \A\zeta_\sigma\|\leq \varepsilon, \quad \forall \zeta_\sigma \in \text{delta function testing dataset},
\end{equation}
we can then estimate the generalization error of $\ANN$ on any boundary condition $I_-$ by the following theorem:
\begin{thm}\label{thm:error-estimate}
  For $\mut$, $\mus$ and $p$ given and satisfying the asumptions in Theorem~\ref{thm:existence-uniqueness-l2},
  let $\A$ and $\ANN$ be the linear operators defined in~\eqref{eq:rte-op} and~\eqref{eq:greens-integral-nn}, respectively.
  For any $I_- \in L^2(\Gamma_-)$, let $I^{N_{\text{bc}}}_{-,\sigma}$ be the approximation of $I_-$ defined in~\eqref{eq:particle-approx}. If there exists a parameter set $\theta^*$ of neural network such that
  \begin{equation}
    \|\ANN_{\theta^*}\zeta_\sigma - \A\zeta_\sigma\|\leq \varepsilon, \quad \forall \zeta_\sigma \in \text{delta function testing dataset},
  \end{equation}
  then we have:
  \begin{equation}
    \begin{aligned}
      \|\A^{\text{NN}}_{\theta^*} I_- - \A I_-\|_{L^2(D\times\sS^{d-1})}\leq C\left\{ \sigma^k \|I_-\|_{H^k(\Gamma_-)} + \left(\frac{h}{\sigma}\right)^m \|I_-\|_{H^m(\Gamma_-)} + \varepsilon\right\}.
    \end{aligned}
  \end{equation}
\end{thm}

\begin{proof}
  We use $L^2$ instead of $L^2(D\times\sS^{d-1})$ to simplify the notation in the following proof.
  The solution operator $\A$ is linear since RTE is a linear equation w.r.t. boundary condition and let us first state the linearity of $\ANN_{\theta^*}$ which can be easily verified using the definition in~\eqref{eq:greens-integral-nn}:
  \begin{equation}
    \begin{aligned}
      \ANN_{\theta^*}(\alpha I_{-}^{(1)} + \beta I_{-}^{(2)})
      &=\int_{\Gamma_-} G^{\text{NN}}_{\theta^*} (\alpha I_{-}^{(1)} + \beta I_{-}^{(2)}) \diff{\br'} \diff{\bOmega'}, \\
      &=\alpha\int_{\Gamma_-} G^{\text{NN}}_{\theta^*} I_{-}^{(1)} \diff{\br'} \diff{\bOmega'} + \beta \int_{\Gamma_-} G^{\text{NN}}_{\theta^*} I_{-}^{(2)} \diff{\br'} \diff{\bOmega'}, \\
      & =\alpha \ANN_{\theta^*} I_{-}^{(1)} + \beta \ANN_{\theta^*} I_{-}^{(2)}, \quad \quad \forall I_{-}^{(1)}, I_{-}^{(2)}, \in L^2(\Gamma_-), \; \alpha, \beta \in \mathbb{R}.
    \end{aligned}
  \end{equation}
  Thus, we have
  \begin{equation}\label{eq:error_estimate}
    \|\ANN_{\theta^*} I_- - \mathcal{A} I_-\|_{L^2}
    \leq \|\ANN_{\theta^*} (I_- - I^{N_{\text{bc}}}_{-,\sigma})\|_{L^2} + \|\mathcal{A} (I_- - I^{N_{\text{bc}}}_{-,\sigma})\|_{L^2} + \|(\ANN_{\theta^*} - \mathcal{A}) I^{N_{\text{bc}}}_{-,\sigma}\|_{L^2} .
  \end{equation}

  For the first term, since $\A^{\text{NN}}_{\theta^*}$ defined by
  \begin{equation}
    \ANN_{\theta^*} I_{-} = \int_{\Gamma_-} G^{\text{NN}}_{\theta^*}(\br, \bOmega; \br', \bOmega') I_{-}(\br', \bOmega') \diff{\br'} \diff{\bOmega'},
  \end{equation}
  and our neural network $G^{\text{NN}}_{\theta^*}$ is continuous (as all the operations in the architecture are continuous function) on the compact domains $\Gamma_-$ and $D \times \mathbb{S}^{d-1}$, so
  \begin{equation}
    \int_{D \times \mathbb{S}^{d-1}\times \Gamma_-} |G^{\text{NN}}_{\theta^*}|^2 \diff{\br'} \diff{\bOmega'} \diff{\br} \diff{\bOmega} < \infty,
  \end{equation}
  which means $\ANN_{\theta^*}$ is a Hilbert-Schmidt operator.
  By the properties of Hilbert-Schmidt operators, $\A^{\text{NN}}_{\theta^*}$ is automatically bounded as a linear operator.
  Thus, we have
  \begin{equation}\label{eq:error_estimate_1}
    \|\ANN_{\theta^*} (I_- - I^{N_{\text{bc}}}_{-,\sigma})\|_{L^2} \leq C_1 \|I_- - I^{N_{\text{bc}}}_{-,\sigma}\|_{L^2(\Gamma_-)}.
  \end{equation}
  For the second term, we can also obtain, according to Theorem~\ref{thm:existence-uniqueness-l2},
  \begin{equation}\label{eq:error_estimate_3}
    \|\mathcal{A} (I_- - I^{N_{\text{bc}}}_{-,\sigma})\|_{L^2} \leq C_2 \|I_- - I^{N_{\text{bc}}}_{-,\sigma}\|_{L^2(\Gamma_-)}.
  \end{equation}
  Finally, for the last term, by the definition of $I^{N_{\text{bc}}}_{-,\sigma}$ in~\eqref{eq:particle-approx} and linearity of $\A$, $\ANN_{\theta^*}$, we have
  \begin{equation}\label{eq:error_estimate_2}
    \begin{aligned}
      \|(\ANN_{\theta^*} - \mathcal{A}) I^{N_{\text{bc}}}_{-,\sigma}\|_{L^2}
      &= \left\|(\ANN_{\theta^*} - \mathcal{A}) \sum_{i,j}^{N_{\text{bc}}} w_{ij}\zeta_\sigma\right\|_{L^2} \\
      &\leq \sum_{i,j}^{N_{\text{bc}}} |w_{ij}| \|(\ANN_{\theta^*} - \mathcal{A}) \zeta_\sigma\|_{L^2} \\
      &\leq \varepsilon \sum_{i,j}^{N_{\text{bc}}} |w_{ij}|:= C_3 \varepsilon.
    \end{aligned}
  \end{equation}
  Combining~\eqref{eq:error_estimate},~\eqref{eq:error_estimate_1},~\eqref{eq:error_estimate_2} and~\eqref{eq:error_estimate_3}, we obtain
  \begin{equation}
    \begin{aligned}
      \|\A^{\text{NN}}_{\theta^*} I_- - \A I_-\|_{L^2}
      &\leq (C_1 + C_2) \|I_- - I^{N_{\text{bc}}}_{-,\sigma}\|_{L^2(\Gamma_-)} + \varepsilon \sum_{i,j}^{N_{\text{bc}}} |w_{ij}|, \\
      &\leq C\left\{ \sigma^k \|I_-\|_{H^k(\Gamma_-)} + \left(\frac{h}{\sigma}\right)^m \|I_-\|_{H^m(\Gamma_-)} + \varepsilon\right\},
    \end{aligned}
  \end{equation}
  where the last inequality is due to Theorem~\ref{thm:bc-error-estimate}. This completes the proof.
\end{proof}

In our implementation, we use the following Gaussian mollification kernel
\begin{equation}
	\zeta_\sigma(\br-\br', \bOmega-\bOmega') := \delta^{\sigma_{\br}}_{\{\br'\}}(\br)\delta^{\sigma_{\bOmega}}(\bOmega-\bOmega'),
\end{equation}
where
\begin{equation}\label{eq:delta_defination}
	\delta_{\{\br'\}}^{\sigma}(\br) = \frac{1}{\sigma \sqrt{\pi}} \exp\left( -\frac{{(\br-\br')}^2}{\sigma^2} \right), \quad
	\delta^{\sigma}(\bOmega-\bOmega') = \frac{1}{\sigma \sqrt{\pi}} \exp\left( -\frac{(\bOmega-\bOmega')^2}{\sigma^2} \right),
\end{equation}
where $(\br', \bOmega') \in \Gamma_-$.

\subsection{Loss and optimization}
The loss function is defined as the mean squared error (MSE) between the predicted intensity $I^{\text{NN}}_{\bm{\theta}}$ and the corresponding ground truth $I$:
\begin{equation}\label{eq:mse}
	\ell(I^\text{NN}_{\bm{\theta}}, I)= \frac{1}{N_\text{coords}}\sum_{i=1}^{N_{\text{coords}}} \left| I^{\text{NN}}_{\bm{\theta}}(\br_i,\bOmega_i) - I(\br_i, \bOmega_i) \right|^2,
\end{equation}
\begin{equation}\label{eq:loss}
	\mathcal{L}(\bm{\theta}) = \frac{1}{N}\sum_{n=1}^N \ell(I^\text{NN}_{\bm{\theta},n}, I_n),
\end{equation}
where $N_{\text{coords}}$ is the number of phase coordinates. However, to conserve computational resources during training, inspired by stochastic gradient descent principles, we sample only $N_\text{col}$ collocation points within the phase space at each gradient descent step:

\begin{equation}\label{eq:collocation_points} \ell(I^\text{NN}_{\bm{\theta}}, I)= \frac{1}{N_\text{col}}\sum_{i=1}^{N_{\text{col}}} \left| I^{\text{NN}}_{\bm{\theta}}(\br_i,\bOmega_i) - I(\br_i, \bOmega_i) \right|^2,
\end{equation}

In practice, training a neural network involves minimizing a loss function with respect to its parameters, $\bm{\theta}$. This optimization problem is formulated as:
\begin{equation}
	\min_{\bm{\theta}} \mathcal{L}(\bm{\theta}),
\end{equation}
and is typically solved using methods such as stochastic gradient descent (SGD) or advanced variants like Adam~\cite{kingma2017,rakhlin2012}.

\subsection{Accuracy and evaluation}
In practical applications, the density is a crucial and widely adopted metric for evaluation. The density is obtained by integrating the intensity $I(\br,\bOmega)$ with respect to $\bOmega$ over the $d$-dimensional unit sphere $\sS^{d-1}$, and is mathematically expressed as:
\begin{equation}\label{eq:density}
	\Phi(\br) = \frac{1}{S^{d-1}}\int_{\sS^{d-1}} I(\br, \bOmega) \diff{\bOmega},
\end{equation}

During the model's training phase, the target loss function is the mean squared error (MSE) of $I(\br,\bOmega)$ (see~\eqref{eq:loss}).
However, when evaluating the model's performance, the mean squared error is calculated based on the density.
Specifically, we define this MSE as:
\begin{equation}
	\text{MSE} = \frac{1}{N_{\text{mesh}}}\sum_{i=1}^{N_{\text{mesh}}} | \Phi^{\text{NN}}(\br_i) - \Phi(\br_i) |^2.
\end{equation}

Additionally, one can compute the relative error and we use the root mean square percentage error (RMSPE), defined as:
\begin{equation}
	\text{RMSPE} = \sqrt{\frac{\displaystyle\sum^{N_{\text{mesh}}}_{i=1} |\Phi^{\text{NN}}(\br_i) - \Phi(\br_i)|^2}{\displaystyle\sum^{N_{\text{mesh}}}_{i=1}|\Phi(\br_i)|^2}}\cdot 100\%.
\end{equation}

\section{Experiments and results}\label{sec:experiments}
In this section, we present the numerical experiments and results to verify the effectiveness of DeepRTE.
In practice, the RTE~\eqref{eq:rte} is a $6$-dimensional equation in the position variable $\br\in\mathbb{R}^3$ and angular variable $\bOmega\in\sS^2$ corresponding to $d=3$.
It can be reduced to lower dimensional equations.
In the Cartesian coordinate system, let
\begin{equation}
	\bOmega=(c,s,\zeta), \quad c =
		{\left(1-\zeta^{2}\right)}^{\frac{1}{2}} \cos\theta, \quad s =
		{\left(1-\zeta^{2}\right)}^{\frac{1}{2}} \sin\theta, \quad \text{for }|\zeta| \leq 1.
\end{equation}
Suppose that $\mut, \mus$ only depend on $x$, $y$ and $I$ is uniform along the $z$-axis. The functions
\begin{equation}
	\tilde{I}(x, y,\zeta,\theta) = \frac{1}{2}[I(x,y,z,c,s,\zeta) + I(x,y,z,c,s,-\zeta)],
\end{equation}
and
\begin{equation}
	\tilde{p}(\zeta,\theta,\zeta^*, \theta^*) = \frac{1}{2}[p(c,s,\zeta, c^*,s^*,\zeta^*) + p(c,s,\zeta, c^*,s^*,-\zeta^*)],
\end{equation}
are independent of $z$ and even in $\zeta$. Thus, Eq.~\eqref{eq:rte-with-bc} reduces to the following equation:
\begin{equation}
	\left(c\partial_x \tilde{I}(x,y,\zeta,\theta)+s\partial_y
	\tilde{I}(x,y,\zeta,\theta)\right)+\mut \tilde{I}(x,y,\zeta,\theta)=\frac{\mus}{2\pi}\int_{0}^{1}
	\int_0^{2\pi} \tilde{p}(\zeta, \theta, \zeta^*,\theta^*) \tilde{I}(x,y,\zeta^*,\theta^*) \diff{\theta^*}\diff{\zeta^*},
\end{equation}
with inflow boundary condition:
\begin{equation}
	\tilde{I}(x,y,\bOmega)=\tilde{I}_{-}(x,y,\bOmega), \quad \text{for }
	\bOmega\cdot\bm{n}<0,\quad (x, y)\in\partial D.
\end{equation}

The phase function $p$ is essential in the radiative transport equation and typically depends only on the angle between the incoming and outgoing directions. A common model for anisotropic scattering is the normalized Henyey-Greenstein (H-G) phase function, given by
\begin{equation}
	p(\bOmega,\bOmega^*) = p(\bOmega\cdot\bOmega^*) = \frac{1-g^2}{{\Bigl(1+g^2-2g\,(cc^*+ss^*+\zeta\zeta^*)\Bigr)}^{\frac{3}{2}}}.
\end{equation}
The asymmetry parameter $g$ can be adjusted to control the relative amounts of forward and backward scattering in $p$: $g=0$ corresponds to isotropic scattering, and $g=1$ gives highly peaked forward scattering. Although we use the H-G function here, the algorithm can be easily adapted to other scattering models.

\subsection{Dataset construction}\label{sec:dataset-construction}

The datasets for training and evaluating are generated using conventional numerical methods, i.e., sweeping method~\cite{koch1991solution,zeyao2004parallel}.
The source code is written in MATLAB and Python and is available at \href{https://github.com/mazhengcn/rte-dataset}{rte-dataset}.

\paragraph{Computational domain and discretization}
We conduct our numerical experiments on a rectangular spatial domain $D = [0, 1]\times[0, 1]\subset \mathbb{R}^2$ and angular variable space $\mathbb{S}^{2}$.
The number of discrete mesh points in each direction is $40$, and these points are uniformly distributed.
% Grid points are randomly sampled from discrete grid points following a uniform distribution.

\paragraph{Angular quadrature points}
The angular quadrature points are the positive roots of the
standard Legendre polynomial of degree $2N$ on the interval $[-1,1]$.
Here we take $N=3$, the quadrature points and weights of the
quadrature set are presented in Table~\ref{tab:quadrature}.

\begin{table}[htbp]
	\centering
	\begin{tabular}{ccccc}
		\toprule
		$\zeta_i$ & $\theta_i$ & $c_i$     & $s_i$     & $4 w_i$   \\
		\midrule
		0.2386192 & $\pi$/12   & 0.9380233 & 0.2513426 & 0.1559713 \\
		0.2386192 & 3$\pi$/12  & 0.6866807 & 0.6866807 & 0.1559713 \\
		0.2386192 & 5$\pi$/12  & 0.2513426 & 0.9380233 & 0.1559713 \\
		0.6612094 & $\pi$/8    & 0.6930957 & 0.2870896 & 0.1803808 \\
		0.6612094 & 3$\pi$/8   & 0.2870896 & 0.6930957 & 0.1803808 \\
		0.9324695 & $\pi$/4    & 0.2554414 & 0.2554414 & 0.1713245 \\
		\bottomrule
	\end{tabular}
	\caption{Angular quadrature points and weights of the quadrature set.
		$\zeta_i$ and $\theta_i$ are the quadrature points of the velocity
		space, $c_i$ and $s_i$ are the corresponding cosine and sine
		values, and $w_i$ is the quadrature weight.}\label{tab:quadrature}
\end{table}

\paragraph{Cross section}
The cross section $\mu_t(x,y)$ and $\mu_s(x,y)$ are defined based on the subdomain $D_\mu = [0.4, 0.6] \times [0.4, 0.6]$. The detailed definitions are:
\begin{equation} \label{eq:coefficients_definition}
	\begin{aligned}
		\mu_t(x,y) & =
		\begin{cases}
			U_t, \quad \text{where } U_t \sim \mathcal{U}(5,7) & \text{if } (x,y) \in D_\mu    \\
			10                                                 & \text{if } (x,y) \notin D_\mu
		\end{cases} \\
		\mu_s(x,y) & =
		\begin{cases}
			U_s, \quad \text{where } U_s \sim \mathcal{U}(2,4) & \text{if } (x,y) \in D_\mu    \\
			5                                                  & \text{if } (x,y) \notin D_\mu
		\end{cases}
	\end{aligned}
\end{equation}

\paragraph{Training dataset boundary conditions}
The training dataset uses delta-function inflow boundary conditions, approximated by Gaussian functions:
% The smoothed delta function, $\delta_{x^{\prime}}^\sigma(x)$, used to construct these beams is defined as:
% \begin{equation} \label{eq:smoothed_delta_redefined}
%   \delta_{{x^{\prime}}}^\sigma(x)=\frac{1}{\sigma \sqrt{\pi}}
%   \exp\Bigl(-\frac{{(x-x^{\prime})}^{2}}{\sigma^{2}}\Bigr).
% \end{equation}
% The incoming boundary conditions for the training dataset are:
\begin{equation}\label{eq:bc-condition}
	\centering
	\left \{ % tex-fmt: skip
	\begin{aligned}
		 & I_-(x=0,y,c>0,s;x_l^{\prime}=0,
		y_l^{\prime},c_l^{\prime},s_l^{\prime})
		=\delta_{\{y_l^{\prime}\}}^{\sigma_{\br}}(y)\delta^{\sigma_{\bOmega}}(c - c_l')\delta^{\sigma_{\bOmega}}(s - s_l'),
		\\
		 & I_-(x=1,y,c<0,s;x_r^{\prime}=1,
		y_r^{\prime},c_r^{\prime},s_r^{\prime})
		=\delta_{\{y_r^{\prime}\}}^{\sigma_{\br}}(y)\delta^{\sigma_{\bOmega}}(c - c_r')\delta^{\sigma_{\bOmega}}(s - s_r'),
		\\
		 & I_-(x,y=0,c,s>0;x_b^{\prime},y_b^{\prime}=0, c_b^{\prime},
		s_b^{\prime})
		=\delta_{\{x_b^{\prime}\}}^{\sigma_{\br}}(x)\delta^{\sigma_{\bOmega}}(c - c_b')\delta^{\sigma_{\bOmega}}(s - s_b'),
		\\
		 & I_-(x,y=1,c,s<0;x_t^{\prime},y_t^{\prime}=1,c_t^{\prime},
		s_t^{\prime})
		=\delta_{\{x_t^{\prime}\}}^{\sigma_{\br}}(x)\delta^{\sigma_{\bOmega}}(c - c_t')\delta^{\sigma_{\bOmega}}(s - s_t'),
	\end{aligned}
	\right. % tex-fmt: skip
\end{equation}
where the variables $y_l'$ (left), $y_r'$ (right), $x_b'$ (bottom), and $x_t'$ (top) are uniformly sampled from discrete mesh points along the corresponding boundary.
The angular variables $c_l'$, $c_r'$, $c_b'$, $c_t'$, $s_l'$, $s_r'$, $s_b'$, and $s_t'$ are uniformly sampled from discrete velocity directions that satisfy the inflow boundary conditions. The parameters $\sigma_{\br}$ and $\sigma_{\bOmega}$ are drawn from uniform distributions $\sigma_{\br} \sim \mathcal{U}(0.005\sqrt{2},0.02\sqrt{2})$ and $\sigma_{\bOmega} \sim \mathcal{U}(0.005\sqrt{2},,0.01\sqrt{2})$, respectively.

\paragraph{Dataset normalization}
In training process, the intensity $I$ was normalized to the range $[0, 1]$ using the formula $(I - I_{\text{min}}) / (I_{\text{max}} - I_{\text{min}})$, where $I_{\text{min}}$ and $I_{\text{max}}$ represent the minimum and maximum intensity values observed in the dataset, respectively. This normalization helps stabilize the training process and improve model convergence.

% \begin{itemize}
%     \item $y_l'$ and $y_r'$ are the y-coordinates for the center of beams on the left ($x_0=0$) and right ($x_0=1$) edges, respectively. Similarly, $x_b'$ and $x_t'$ are the x-coordinates for beams on the bottom ($y_0=0$) and top ($y_0=1$) edges. These source coordinates are uniformly sampled from discrete grid points along the respective boundary.
%     \item $c'$ and $s'$ (representing $c_0, s_0$ in the definition of $P$) define the central angular direction of the incident beam.
%     \item $\sigma_{\br}$ and $\sigma_{\bOmega}$ control the spatial and angular spread of the Gaussian beams. They are sampled from uniform distributions: $\sigma_{\br} \sim \mathcal{U}(0.005\sqrt{2},\,0.02\sqrt{2})$ and $\sigma_{\bOmega} \sim \mathcal{U}(0.005\sqrt{2},\,0.01\sqrt{2})$.
% \end{itemize}

\paragraph{Scattering regimes}
The model was trained and tested across three distinct scattering regimes, defined by the scattering asymmetry parameter $g \in [0,1]$ (where $g=0$ denotes isotropic scattering and $g=1$ denotes fully forward-peaked scattering). The regimes are as follows:
\begin{itemize}\label{li:regimes}
	\item \textbf{Near isotropy}: \( g \sim \mathcal{U}(0, 0.2) \), characterized by nearly isotropic scattering, where light or neutrons scatter in all directions equally.
	\item \textbf{Moderate anisotropy}: \( g \sim \mathcal{U}(0.4, 0.6) \), representing a middle ground where scattering is neither fully isotropic nor strongly directional.
	\item \textbf{Highly forward peaking}: \( g \sim \mathcal{U}(0.7, 0.9) \),  where particles (photons/neutrons) scatter predominantly forward, similar to a spotlight. This presents modeling challenges because: (1) particle paths become highly correlated; (2) their positions and scattering directions are closely linked; and (3) minor changes in boundary conditions can drastically affect the results.
\end{itemize}
Table~\ref{tab:dataset} summarizes the parameters used in the dataset construction.
\begin{table}[htbp]
	\centering
	\begin{tabular}{lp{0.3\textwidth}ll}
		\toprule
		Category                             & Parameters                                      & Symbol                   & Value/Range                                     \\ \midrule
		\multirow{2}{*}{Spatial domain}      & Domain                                          & $D$                      & ${[0,1]}^2$                                     \\
		                                     & Subdomain                                       & $D_\mu$                  & ${[0.4,0.6]}^2$                                 \\
		\midrule
		\multirow{4}{*}{Cross section}       & \multirow{2}{*}{Total}                          & \multirow{2}{*}{$\mut$}  & $\mathcal{U}(5,7)$ in $D_\mu$                   \\
		                                     &                                                 &                          & $10$ in $D\backslash D_\mu$                     \\
		                                     & \multirow{2}{*}{Scattering}                     & \multirow{2}{*}{$\mus$}  & $\mathcal{U}(2,4)$ in $D_\mu$                   \\
		                                     &                                                 &                          & $5$ in $D\backslash D_\mu$                      \\
		\midrule
		\multirow{2}{*}{Discretization}      & \# of mesh points                               & $N_{\text{mesh}}$        & 40                                              \\
		                                     & \# of angular quadrature points                 & $N_\text{quad}$          & 24                                              \\ \midrule
		\multirow{5}{*}{Boundary conditions} & Beam spatial center coordinates                 & $y_l', y_r', x_b', x_t'$ & Sampled from mesh points                        \\
		                                     & \multirow{2}{*}{Beam angular quadrature points} & $c_l', c_r', c_b', c_t'$ & \multirow{2}{*}{Sampled from quadrature points} \\
		                                     &                                                 & $s_l', s_r', s_b', s_t'$ &                                                 \\
		                                     & Beam spatial std dev                            & $\sigma_{\br}$           & $\sqrt{2}\,\mathcal{U}(0.005, 0.02)$            \\
		                                     & Beam angular std dev                            & $\sigma_{\bOmega}$       & $\sqrt{2}\,\mathcal{U}(0.005, 0.01)$            \\
		\midrule
		\multirow{3}{*}{Scattering}          & \multirow{3}{*}{Asymmetry parameter}            & \multirow{3}{*}{$g$}     & $\mathcal{U}(0,0.2)$                            \\
		                                     &                                                 &                          & $\mathcal{U}(0.4,0.6)$                          \\
		                                     &                                                 &                          & $\mathcal{U}(0.7,0.9)$                          \\
		\bottomrule
	\end{tabular}
	\caption{Summary of hyperparameters for dataset construction. This table details the computational domain, cross sections, discretization, boundary conditions, and other parameters used.}\label{tab:dataset}
\end{table}

\subsection{Training details}

\paragraph{Training configuration}

\begin{table}[htbp]
	\centering
	\begin{tabular}{ll}
		\toprule
		Hyperparameters          & Value                      \\ \midrule
		Optimizer                & Adam                       \\
		Learning rate schedule   & Cosine annealing           \\
		Initial learning rate    & $1 \times 10^{-3}$         \\
		Batch size               & 8                          \\
		Epochs                   & 5000                       \\
		\# of training data      & 800                        \\
		\# of validation data    & 200                        \\
		\# of collocation points & 128                        \\
		Hardware                 & 4$\times$ NVIDIA 4090 GPUs \\
		\bottomrule
	\end{tabular}
	\caption{Summary of model training hyperparameters. This table details the optimizer and training hyperparameters, including learning rate, dataset settings, and hardware information.}
	\label{tab:training_details}
\end{table}
The model was trained using the Adam optimizer with a cosine annealing learning rate schedule.
The training process involved 800 samples per scattering regime, and its performance was evaluated on a separate validation set of 200 samples per regime.
Training was conducted on four NVIDIA 4090 GPUs and proceeded for 5000 epochs.
The initial learning rate was set to $1 \times 10^{-3}$, the batch size was 8, and the number of collocation points used in the loss computation (see Eq.~\eqref{eq:collocation_points}) was 128.
Table~\ref{tab:training_details} summarizes the training hyperparameters.

\paragraph{Network architecture}
The neural network hyperparameters of deeprte are detailed in Table~\ref{tab:parameters}.

\begin{table}[htbp]
	\centering
	\begin{tabular}{lll}
		\toprule
		Module Name                  & Hyperparameters                          & Value \\ \midrule
		\multirow{7}{*}{Attenuation} & \texttt{num\_layer}
		$N_{\text{mlp}}$             & 4                                                \\
		                             & \texttt{hidden\_dim}  $d_{\text{mlp}}$   & $128$ \\
		                             & \texttt{output\_dim}  $d_{\text{model}}$ & $16$  \\
		                             & \texttt{num\_head}    $N_{\text{head}}$  & $2$   \\
		                             & \texttt{key\_dim}     $d_k$              & $32$  \\
		                             & \texttt{value\_dim}   $d_v$              & $32$  \\
		                             & \texttt{output\_dim}  $d_{\tau}$         & $2$   \\
		\midrule
		\multirow{2}{*}{Scattering}  & \texttt{num\_block}   $N_{\ell}$         & $2$   \\
		                             & \texttt{latent\_dim}  $d_{\text{model}}$ & $16$  \\
		\bottomrule
	\end{tabular}
	\caption{DeepRTE network hyperparameters. The network consists of two modules: the attenuation module and the scattering module. The hyperparameters include the number of layers, hidden dimensions, output dimensions, number of heads, key dimensions, value dimensions, and latent dimensions.}\label{tab:parameters}
\end{table}

\subsection{Results}
\subsubsection{Accuracy validation}
\label{sec:acc}
This section evaluates the model's accuracy and computational efficiency. Accuracy is measured using the mean squared error (MSE) and root mean squared percentage error (RMSPE) between predictions and ground truth. All experimental findings presented in this subsection are derived from a validation dataset with identical distribution to the training dataset.

\begin{table}[htbp]
	\centering
	\begin{tabular}{@{}cccccc@{}}
		\toprule
		Model                    & \# of parameters       & Scattering regime      & \(g\) range & MSE (\(\times 10^{-10}\)) & RMSPE (\%) \\ \midrule
		\multirow{3}{*}{DeepRTE} & \multirow{3}{*}{37954} & Near isotropy          & (0, 0.2)    & 5.630                     & 2.827      \\ \cmidrule(l){3-6}
		                         &                        & Moderate anisotropy    & (0.4, 0.6)  & 5.453                     & 2.759      \\ \cmidrule(l){3-6}
		                         &                        & Highly forward peaking & (0.7, 0.9)  & 7.223                     & 3.181      \\ \bottomrule
	\end{tabular}
	\caption{Accuracy validation of DeepRTE model across three distinct scattering regimes. All validation experiments were conducted on a dataset sharing the same distribution as the training dataset. The table demonstrating that the model achieves high accuracy with RMSPE consistently below 3.2\% across all regimes.}\label{tab:accuracy}
\end{table}

Across all regimes, DeepRTE achieved an MSE not exceeding \(7.3 \times 10^{-10}\) and an RMSPE below 3.2\%, indicating high accuracy. Specifically, in the near isotropy regime, the MSE was \(5.630 \times 10^{-10}\) with an RMSPE of 2.827\%, reflecting excellent performance in nearly isotropic conditions. For moderate anisotropy regime, the MSE was slightly lower at \(5.453 \times 10^{-10}\), with an RMSPE of 2.759\%, showing comparable accuracy. Notably, in the highly forward peaking regime, where modeling is typically more challenging due to the directional nature of photon transport, the model recorded an MSE of \(7.223 \times 10^{-10}\) and an RMSPE of 3.181\%, demonstrating robust performance despite the increased complexity.

Table~\ref{tab:accuracy-figs} provides visual comparisons of $\Phi$ predictions for representative cases across three scattering regimes. For each regime, we display: (1) 3D ground truth, (2) 2D ground truth projection, (3) model prediction, and (4) absolute error between prediction and ground truth.

The Table~\ref{tab:time_compare} presents a runtime comparison between DeepRTE and traditional methods, highlighting the efficiency gains achieved by our approach. The fast sweeping method was implemented in two variants: a CPU implementation (Python with NumPy and SciPy) measured on an Intel(R) Xeon(R) Silver 4410Y (12 cores), which requires $193.0$ s, and a highly optimized GPU implementation (JAX with data sharding across devices, multi-GPU parallelism, and \texttt{jax.jit} for JIT (Just-in-Time) compilation, see Alg.~\ref{alg:fast-sweeping-gpu}) measured on $4\times$ NVIDIA GeForce RTX 4090, which requires $16.7$ s ($11.6\times$ speedup over the CPU). By contrast, DeepRTE performs inference in $2.3$ s on $4\times$ RTX 4090, achieving a $ 7.3\times$ speedup over the fast sweeping method on GPU and an $83.9\times$ speedup relative to the CPU baseline.

These results show that DeepRTE effectively captures the physics of scattering across different parameter settings.
The model achieves low error rates even in complex forward scattering scenarios.
With only $37,954$ parameters, DeepRTE ensures computational efficiency.
These characteristics make it well-suited for large-scale simulations requiring fast processing.

\begin{table}[htbp]
	\centering
	\begin{tabular}{@{}lccc@{}}
		\toprule
		Method                     & Device                        & Count    & Time (s) \\ \midrule
		Fast sweeping method (CPU) & Intel(R) Xeon(R) Silver 4410Y & 12 cores & 193.0    \\
		Fast sweeping method (GPU) & NVIDIA GeForce RTX 4090       & 4 GPUs   & 16.7     \\
		\midrule
		DeepRTE (GPU)              & NVIDIA GeForce RTX 4090       & 4 GPUs   & 2.3      \\ \bottomrule
	\end{tabular}
	\caption{Runtime comparison. The conventional fast sweeping method requires 193.0 s on an Intel(R) Xeon(R) Silver 4410Y (12 cores) and 16.7 s on $4\times$ NVIDIA GeForce RTX 4090 GPUs ($11.6\times$ over CPU). By contrast, DeepRTE completes inference in 2.3 s on $4\times$ RTX 4090, achieving a $7.3\times$ speedup over the GPU fast sweeping method and an $83.9\times$ speedup relative to the CPU baseline.}
	\label{tab:time_compare}
\end{table}

\begin{algorithm}[!htbp]
	\caption{Multi-GPU Fast-Sweeping Method}\label{alg:fast-sweeping-gpu}
	\small
	$\Def \text{ FastSweeping}(\mu_t, \mu_s, \{c_k\}, \{s_k\}, \{w_{k'}\}, p, \Delta x, \Delta y, D=4, \text{tol}=10^{-8}, \text{max\_iter}=1000)$:
	\begin{algorithmic}[1]
		\Comment{Data sharding: partition angular flux and scattering kernel across $D$ devices}
		\State$M \gets n_v / D$
		\Comment{Shard angular flux and scattering kernel; replicate cross-sections}
		\ForAll{device $d \in \{1,\dots,D\}$}
		\State$I^{d,[\ell]}_{i, j, k} \gets \texttt{jax.zeros}(n_x, n_y, M)$ \hfill $I^{d,[\ell]}_{i,j,k} \in \mathbb{R}^{n_x \times n_y \times M}$
		\State$p^d_{k,k'} \gets p_{k,k'}$ for $k \in [(d-1)M, dM)$ and all $k'$
		\Comment{Replicate material cross-sections to all devices}
		\State$\mu_t^d, \mu_s^d \gets \mu_t, \mu_s$
		\EndFor
		\Comment{Source iteration with convergence check}
		\ForAll{iteration $\ell \in [0, \text{max\_iter})$}
		\Statex{\hspace{1.5em}\color{Brown}\small\textit{\# \;Step 1: All-gather angular fluxes from all devices (cross-device communication)}}
		\State$I_{i,j,k'}^{[\ell]} \gets \text{All-Gather}(\{I_{i,j, k}^{d,[\ell]}\}_{d=1}^D)$ \hfill $k' \in [0, n_v)$
		\ForAll{device $d$ \textbf{in parallel}}
		\Statex{\hspace{1.5em}\color{Brown}\small\textit{\# \;Step 2: Compute local scattering moments on each device (local computation, no communication)}}
		\State$S^{d,[\ell]}_{i,j,k} \gets \displaystyle\sum_{k'=0}^{n_v-1} w_{k'}\, p^d_{k,k'}\,I^{[\ell]}_{i,j,k'}$ \hfill $k \in [(d-1)M, dM)$
		\Statex{\hspace{1.5em}\color{Brown}\small\textit{\# \;Step 3: Perform local angular sweeps on each device (local computation, no communication)}}
		\ForAll{$i,j$ \textit{in sweep order}}
		\State$I^{d,[\ell+1]}_{i,j,k} \gets \dfrac{c_k\dfrac{I^{d,[\ell]}_{i-1,j,k}}{\Delta x} + s_k\dfrac{I^{d,[\ell]}_{i,j-1,k}}{\Delta y} + (\mu_s)_{i,j}\,S^{d,[\ell]}_{i,j,k}}{\dfrac{c_k}{\Delta x} + \dfrac{s_k}{\Delta y} + (\mu_t)_{i,j}}$ \hfill $k \in [(d-1)M, dM)$
		\EndFor
		\EndFor
		\Comment{Check convergence}
		\If{$\text{RelativeError}(I^{[\ell+1]}, I^{[\ell]}) < \text{tol}$}
		\State\textbf{break}
		\EndIf
		\EndFor
		\Comment{Gather all angular fluxes from devices}
		\State$I^{[\ell]} \gets \text{All-Gather}(\{I^{d,[\ell]}\}_{d=1}^D)$
		\Ret$I^{[\ell]}$
	\end{algorithmic}
\end{algorithm}

\begin{table}[htbp]
	\centering
	\begin{tabular}{>{\centering\arraybackslash}m{0.9cm}cccc}
		\toprule
		$g$                         & $\Phi_{\text{label}}$                                             & $\Phi_{\text{label}}$                                          & $\Phi_{\text{predict}}$                                      & $|\Phi_{\text{label}}- \Phi_{\text{predict}}|$                 \\
		\midrule
		\raisebox{5em}{$(0,0.2)$}   & \includegraphics[width=0.22\textwidth]{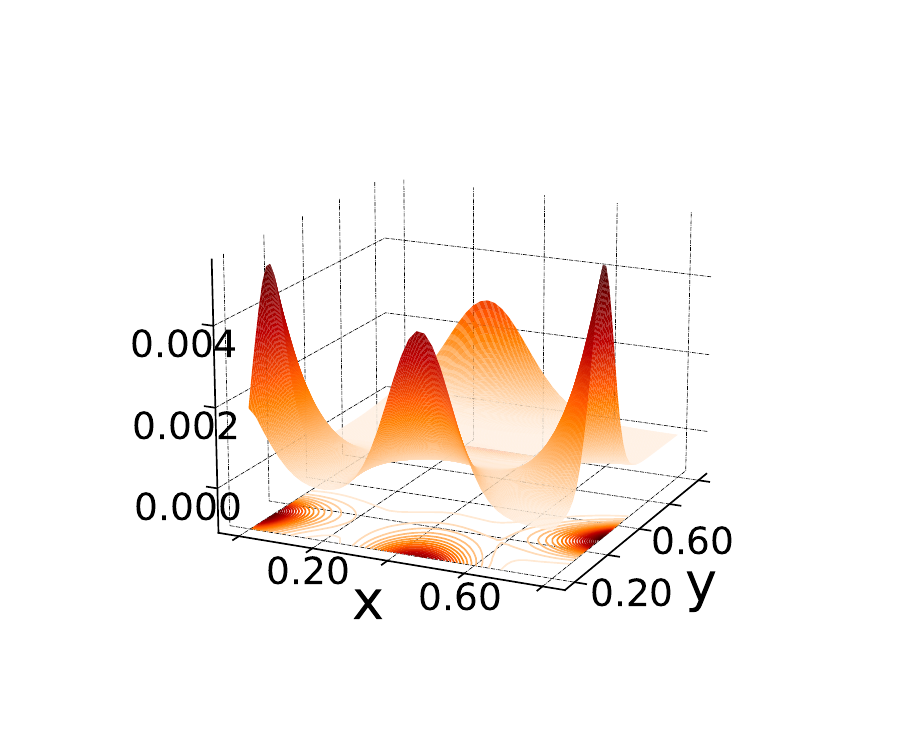} & \includegraphics[width=0.20\textwidth]{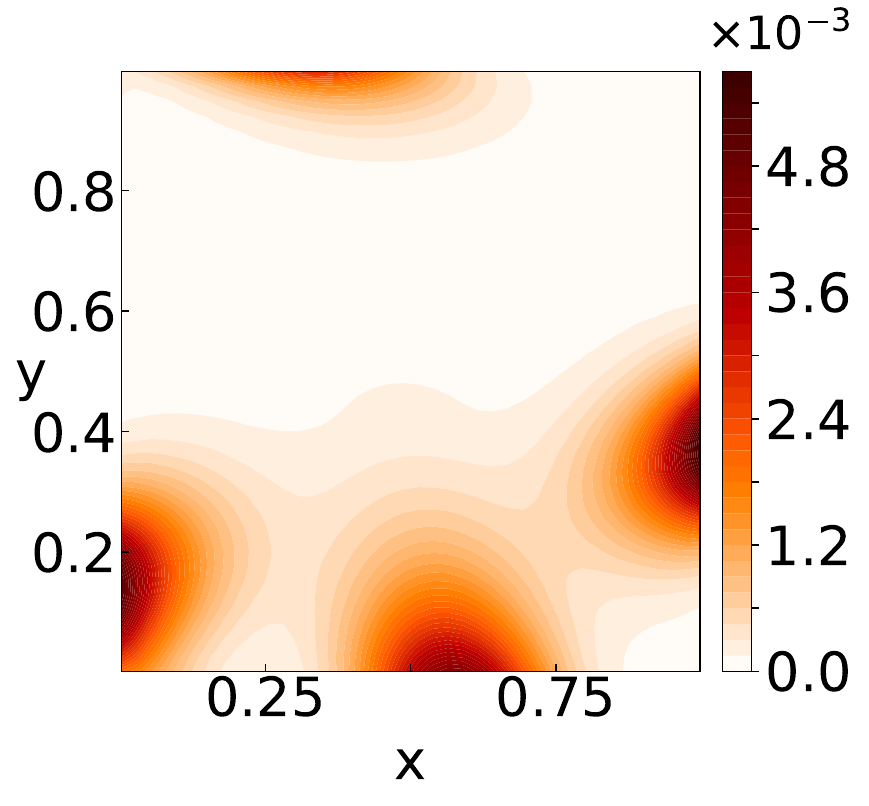} & \includegraphics[width=0.20\textwidth]{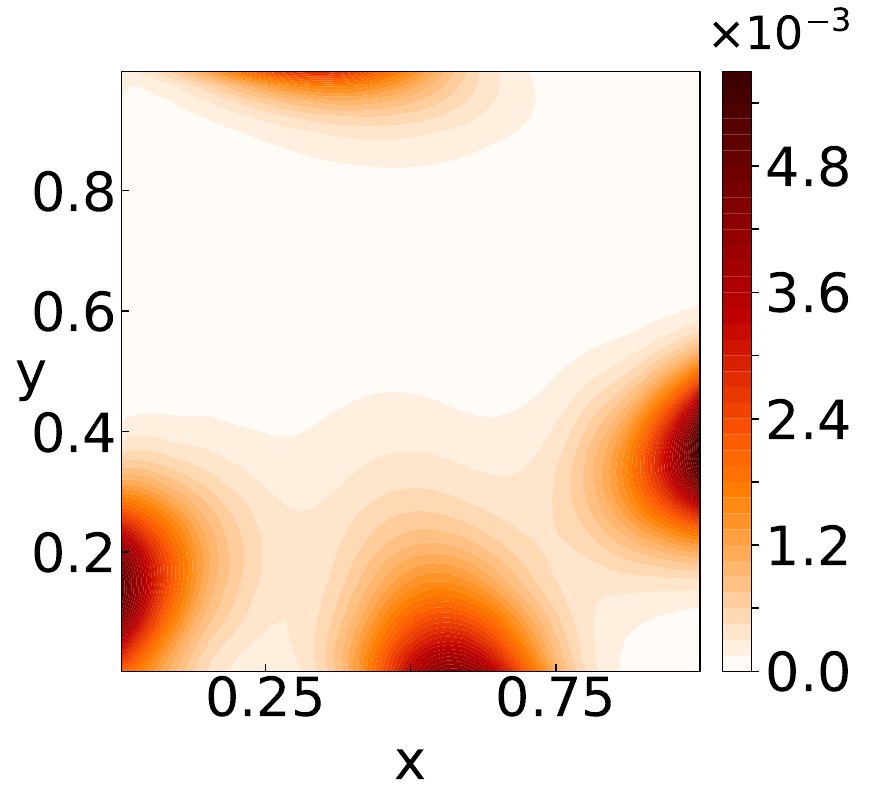} & \includegraphics[width=0.20\textwidth]{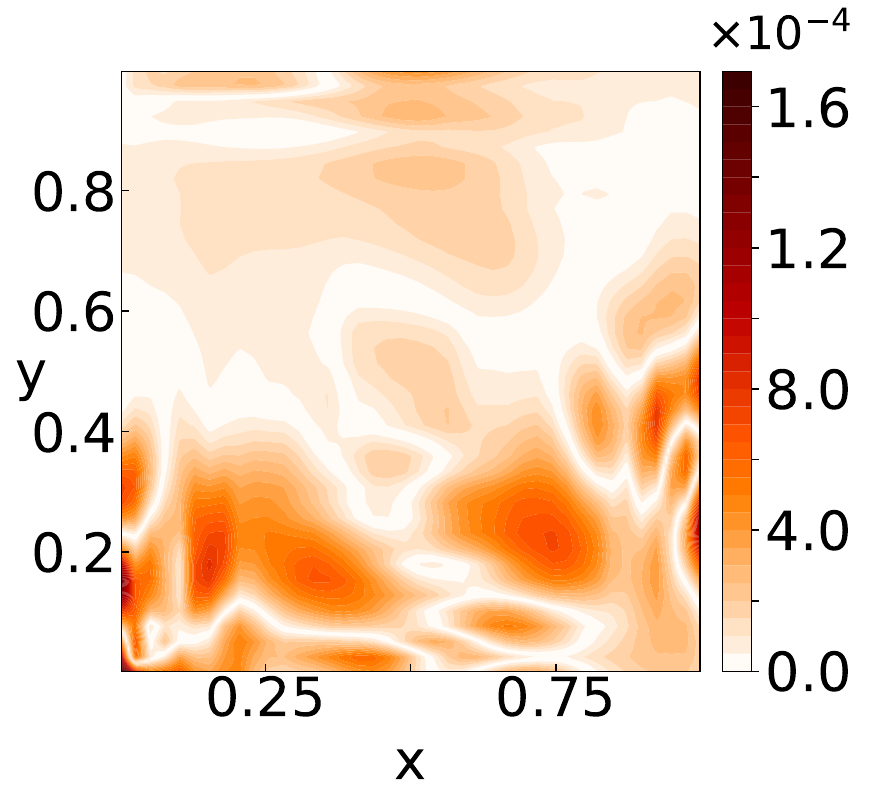} \\
		\raisebox{5em}{$(0.4,0.6)$} & \includegraphics[width=0.22\textwidth]{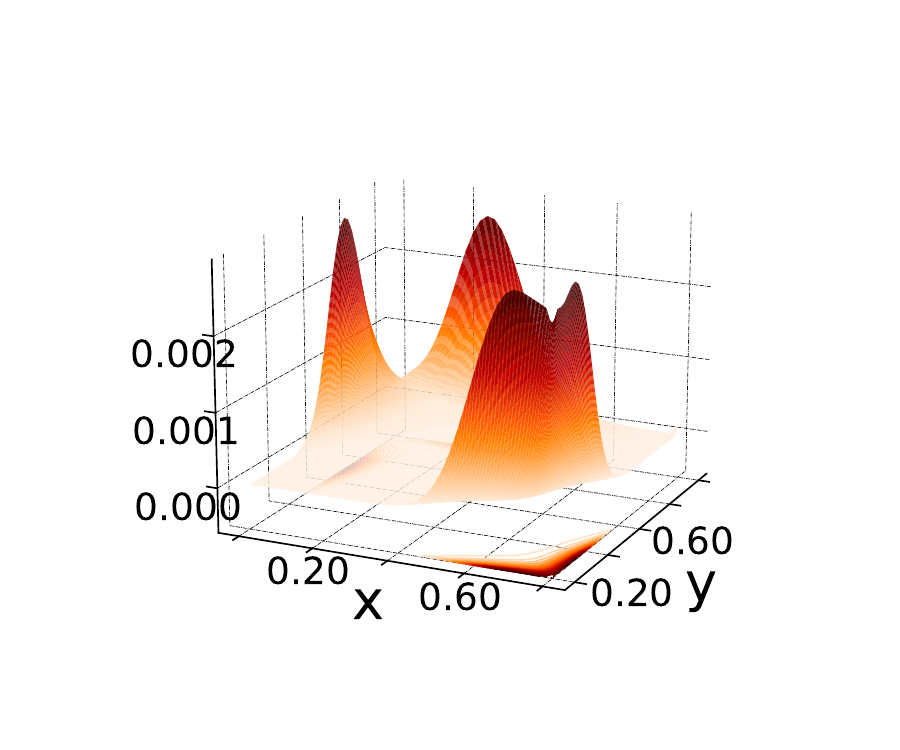} & \includegraphics[width=0.20\textwidth]{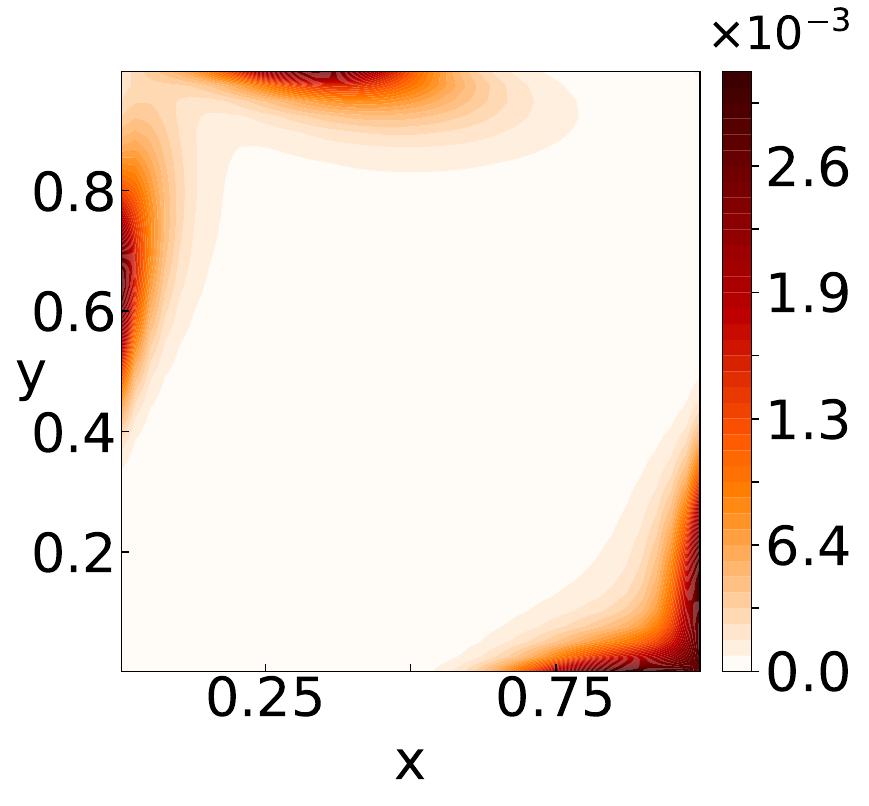} & \includegraphics[width=0.20\textwidth]{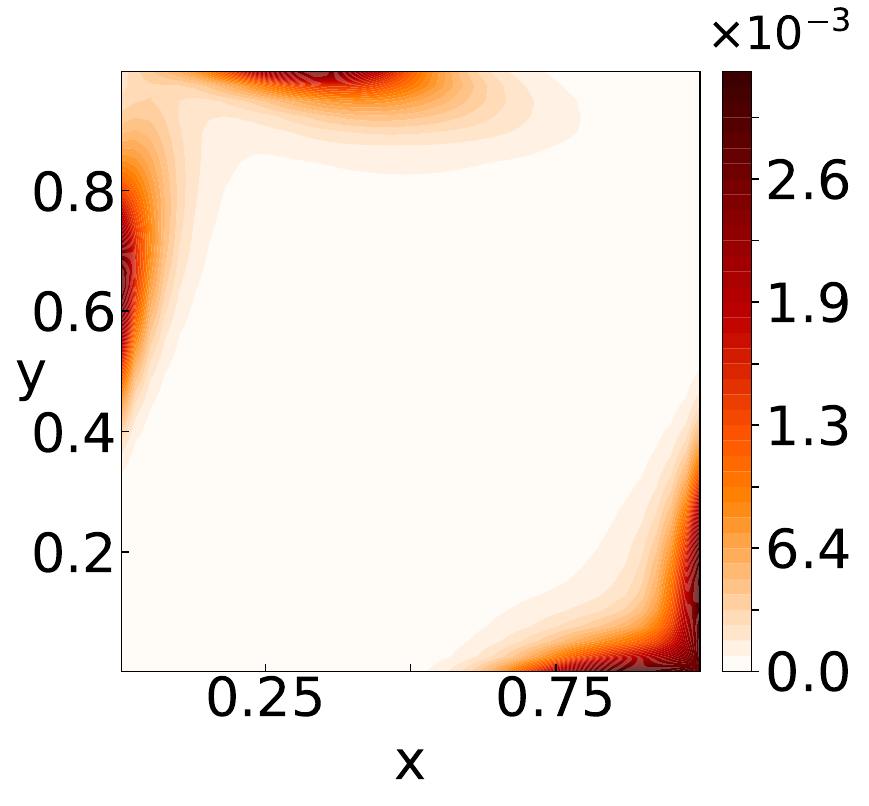} & \includegraphics[width=0.20\textwidth]{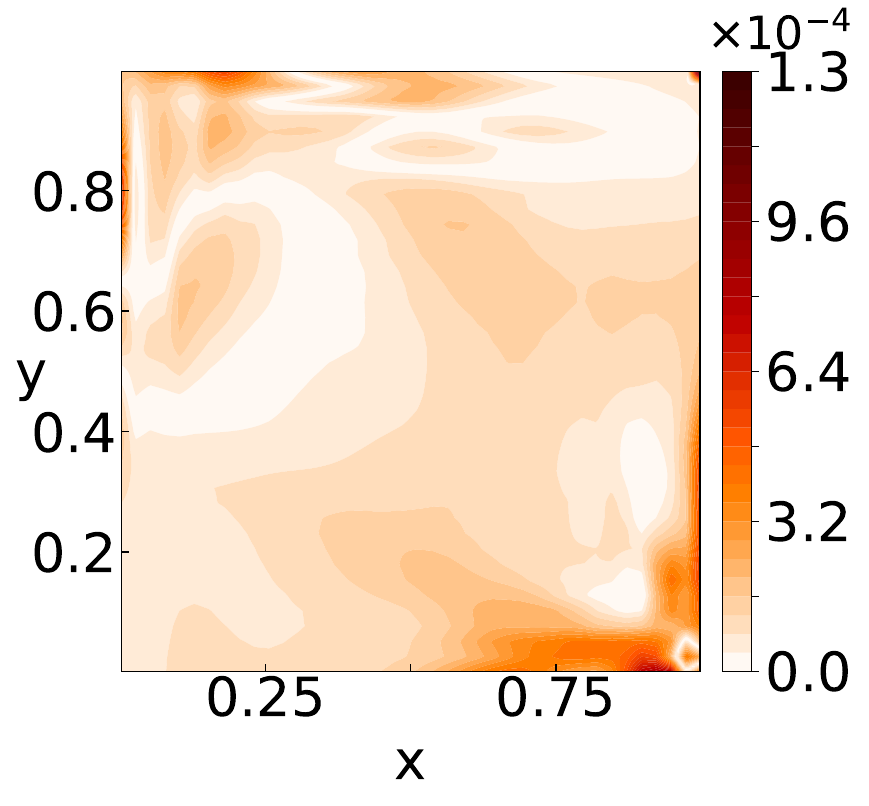} \\
		\raisebox{5em}{$(0.7,0.9)$} & \includegraphics[width=0.22\textwidth]{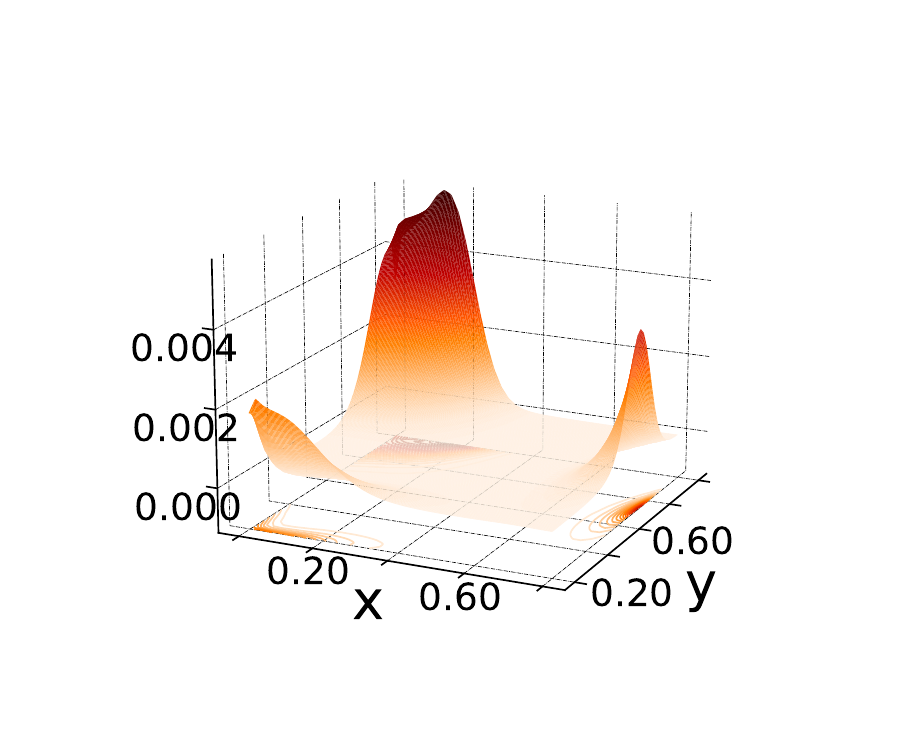} & \includegraphics[width=0.20\textwidth]{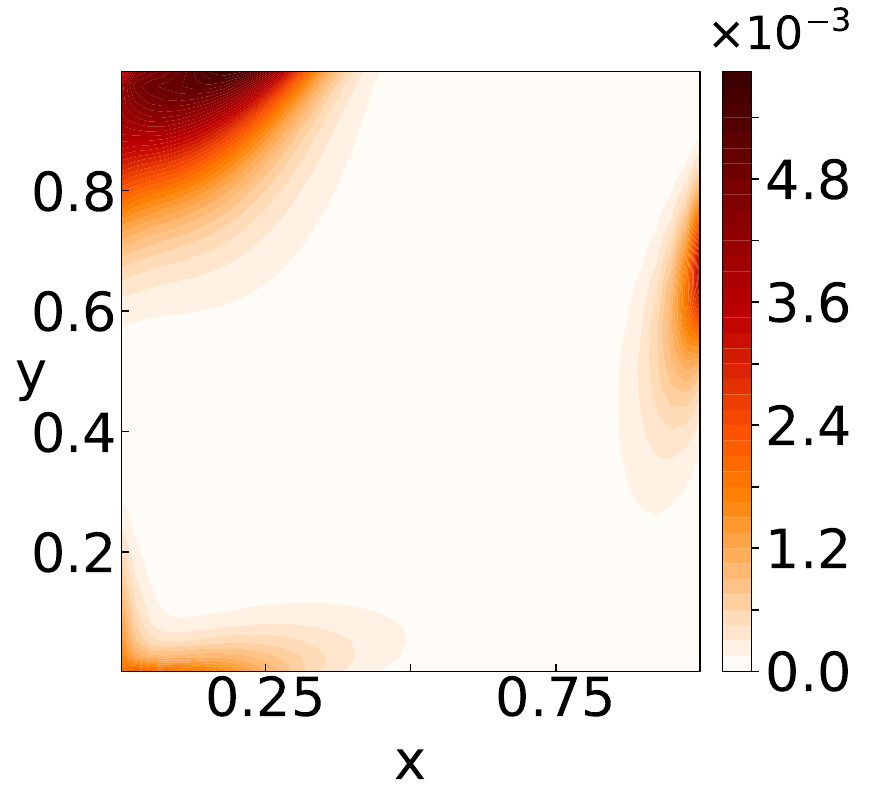} & \includegraphics[width=0.20\textwidth]{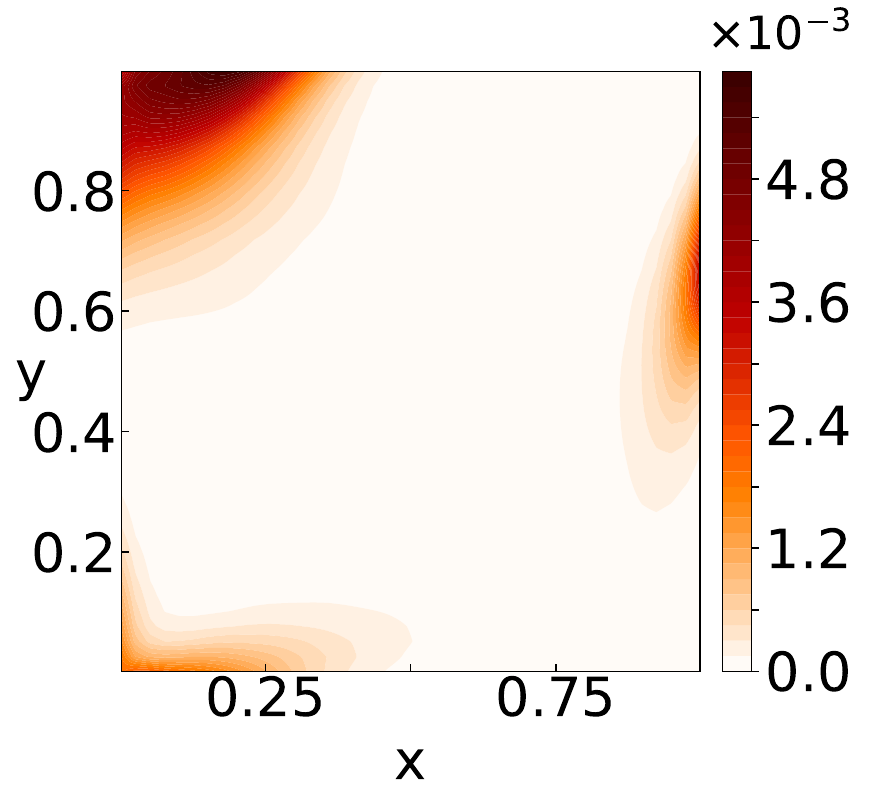} & \includegraphics[width=0.20\textwidth]{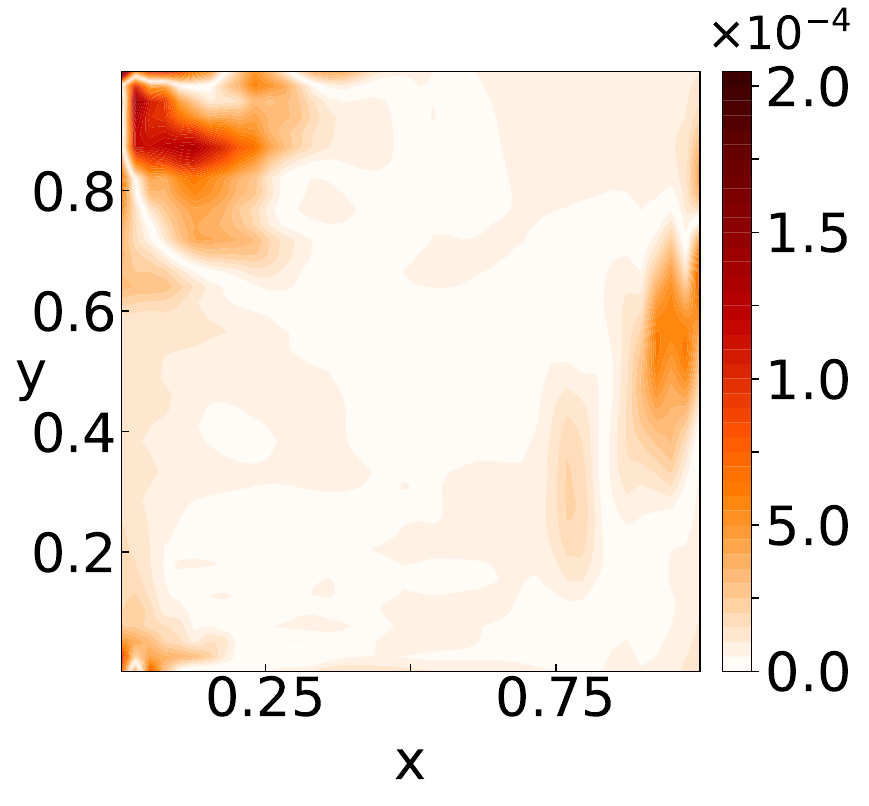} \\
		\bottomrule
	\end{tabular}\caption{Visual comparison of density $\Phi$ predictions across scattering regimes: ground truth (3D and 2D views), model predictions, and absolute errors for representative validation cases.}\label{tab:accuracy-figs}
\end{table}

\begin{figure}
	\centering
	\includegraphics[width=\textwidth]{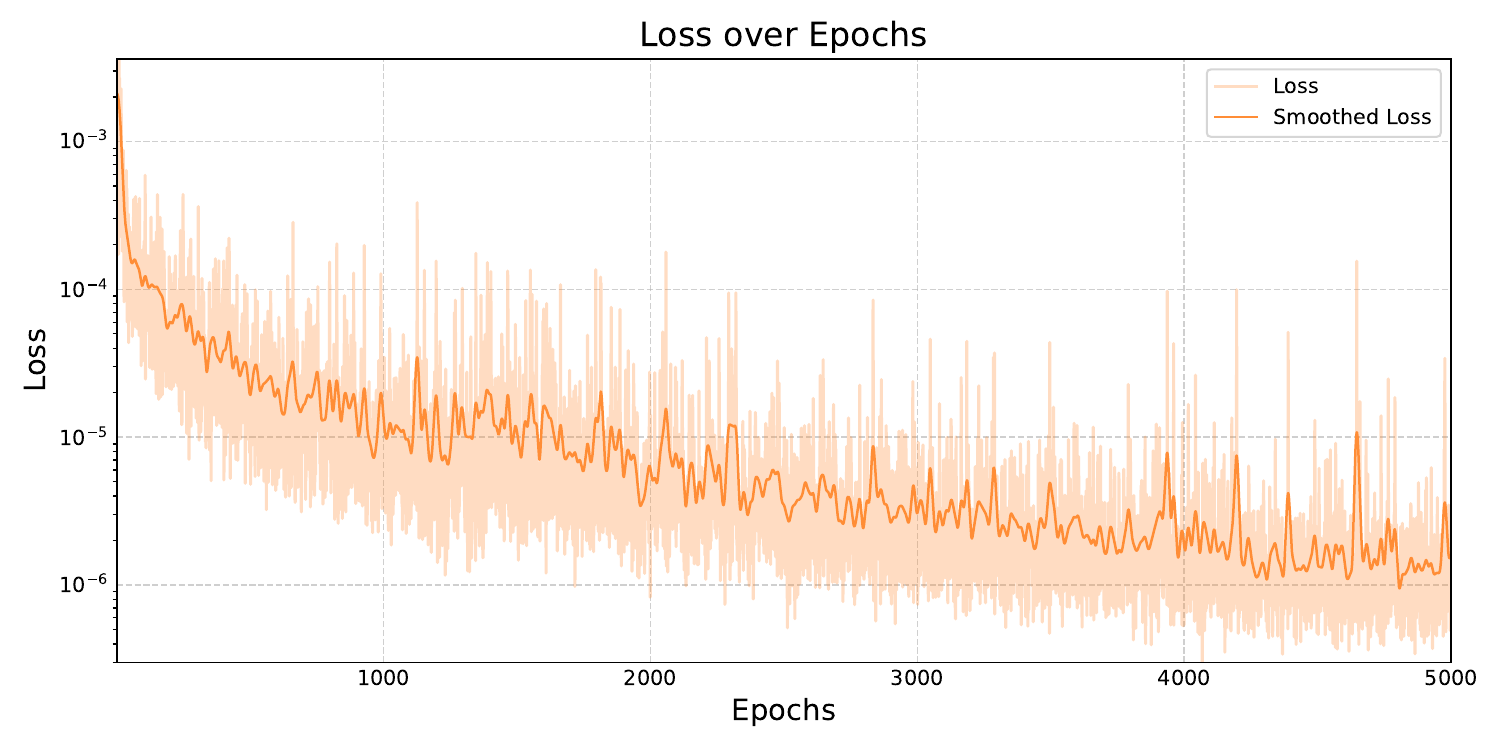}
	\caption{Loss curve. Obtained from training the model on a delta function training dataset, illustrating the convergence behavior over epochs.}\label{fig:loss}
\end{figure}

\subsubsection{Linearity validation}\label{sec:linearity_val}

Our DeepRTE model architecture is specifically designed to preserve the linearity property with respect to the boundary function $I_{-}$.
According to \eqref{eq:greens-integral-nn}, the neural network $G^{\text{NN}}$ used in DeepRTE represents the kernel (i.e., the Green's function) of the solution operator $\ANN$ in integral form. Therefore, the overall input-output mapping DeepRTE model remains linear with respect to $I_{-}$ at both continuous and discrete levels.

To verify this claim, we validate the linearity of DeepRTE through numerical experiments that test both additivity and homogeneity properties. All tests below are conducted with fixed cross sections $\mu_t,\mu_s$ and phase function $p$ with the following settings:
\begin{equation}
  \mu_t(x,y) =
  \begin{cases}
    6, & \text{if } (x,y) \in D_\mu \\
    10 & \text{if } (x,y) \notin D_\mu
  \end{cases},
  \quad
  \mu_s(x,y) =
  \begin{cases}
    3, & \text{if } (x,y) \in D_\mu \\
    5 & \text{if } (x,y) \notin D_\mu
  \end{cases},
  \quad
  g=0.5, \quad \text{where } D_\mu = [0.4, 0.6]^2.
\end{equation}
For simplicity, we write the DeepRTE operator as $\ANN$ instead of $\ANN[\mu_t,\mu_s,p]$. The two fundamental linearity properties are defined as follows:

\textbf{1. Additivity}: The DeepRTE predicted solution corresponding to the sum of two boundary conditions should equals the sum of individual solutions, i.e.,
\begin{equation}\label{eq:additivity-val}
  \ANN(I_{-}^{(1)} + I_{-}^{(2)}) = \ANN I_{-}^{(1)} + \ANN I_{-}^{(2)}.
\end{equation}
In this test, we set two independent boundary conditions $I_{-}^{(1)}$ and $I_{-}^{(2)}$ chosen from our validation dataset as follows, that is, the incoming boundary conditions are in the form of~\eqref{eq:bc-condition} with variance $\sigma_{\br}=0.01$ and $\sigma_{\bOmega}=0.0075$:
\begin{equation}
  I_{-}^{(1)}:
  \begin{cases}
    y_l'=0.3875, y_r'=0.6125, x_b'=0.6125, x_t'=0.3875, \\
    c_l'=0.6931, c_r'=-0.6931, c_b'=0.2871, c_t'=-0.2871, \\
    s_l'=-0.2871, s_r'=0.2871, s_b'=0.6931, s_t'=-0.6931.
  \end{cases}
\end{equation}
and
\begin{equation}
  I_{-}^{(2)}:
  \begin{cases}
    y_l'=0.1375, y_r'=0.8625, x_b'=0.8625, x_t'=0.1375, \\
    c_l'=0.2871, c_r'=-0.2871, c_b'=-0.6931, c_t'=0.6931, \\
    s_l'=0.6931, s_r'=-0.6931, s_b'=0.2871, s_t'=-0.2871.
  \end{cases}
\end{equation}
We evaluate the MSE error between the left and right sides of~\eqref{eq:additivity-val}, i.e.,
% \begin{equation}
%   \|\ANN (I_{-}^{(1)}{+}I_{-}^{(2)}) - ( \ANN I_{-}^{(1)} + \ANN I_{-}^{(2)})\|_2= 1.581\times 10^{-19}.
% \end{equation}
\begin{equation}
  \ell(\ANN (I_{-}^{(1)}{+}I_{-}^{(2)}),\ANN I_{-}^{(1)} + \ANN I_{-}^{(2)})=1.581\times 10^{-19}.
\end{equation}
This verifies that the two sides of~\eqref{eq:additivity-val} are nearly identical, thus confirming the additivity property of DeepRTE.

We also compute the MSE and RMSPE between our model predicted solution $\ANN (I_{-}^{(1)}{+}I_{-}^{(2)})$ and the ground truth solution $\A (I_{-}^{(1)}{+}I_{-}^{(2)})$ which is solved using numerical solver. The MSE is $1.358\times10^{-8}$ and the RMSPE is $2.83\%$ which shows that our model can accurately predict the solution for the combined boundary condition $I_{-}^{(1)}{+}I_{-}^{(2)}$. The Fig.~\ref{fig:additivity} visualizes the results of this additivity test using the density defined $\Phi$ in~\eqref{eq:density}.
\begin{figure}[htbp]
  \centering
  \begin{subfigure}[b]{0.30\textwidth}
    \centering
    \includegraphics[width=\textwidth]{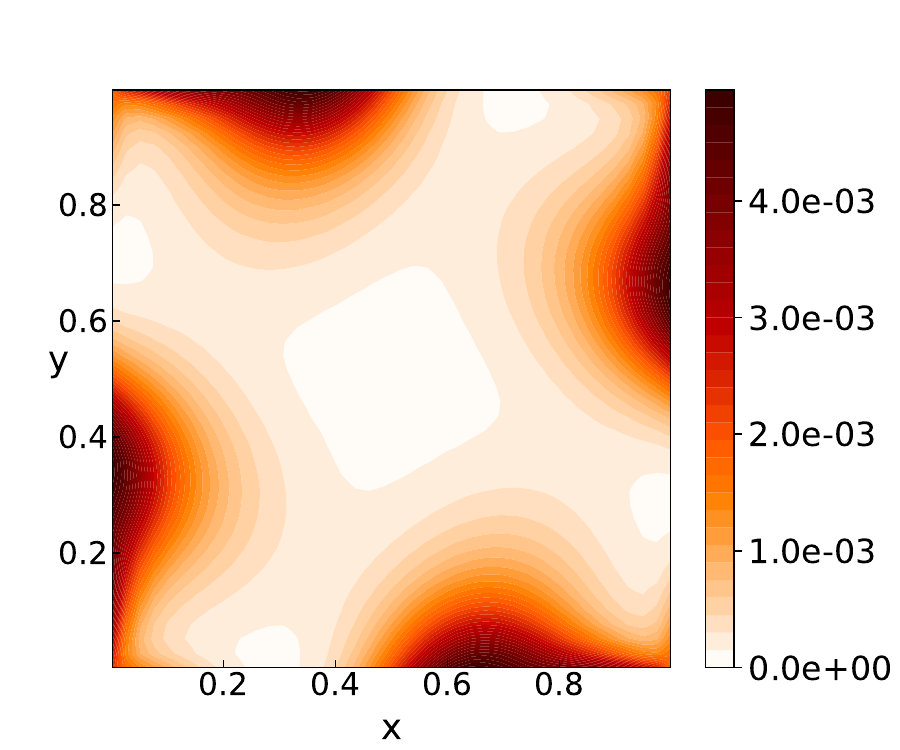}
    \caption{Density $\Phi$ of $\ANN (I_{-}^{(1)} + I_{-}^{(2)})$.}
  \end{subfigure}
  \hfill
  \begin{subfigure}[b]{0.3\textwidth}
    \centering
    \includegraphics[width=\textwidth]{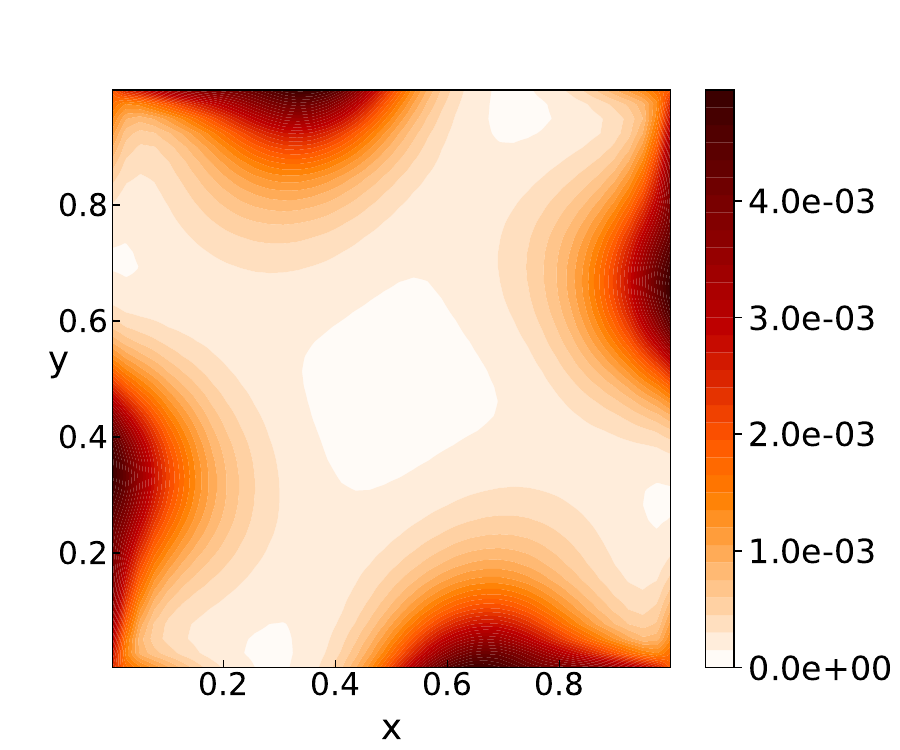}
    \caption{Density $\Phi$ of $\ANN I_{-}^{(1)} + \ANN I_{-}^{(2)}$}
  \end{subfigure}
  \hfill
  \begin{subfigure}[b]{0.3\textwidth}
    \centering
    \includegraphics[width=\textwidth]{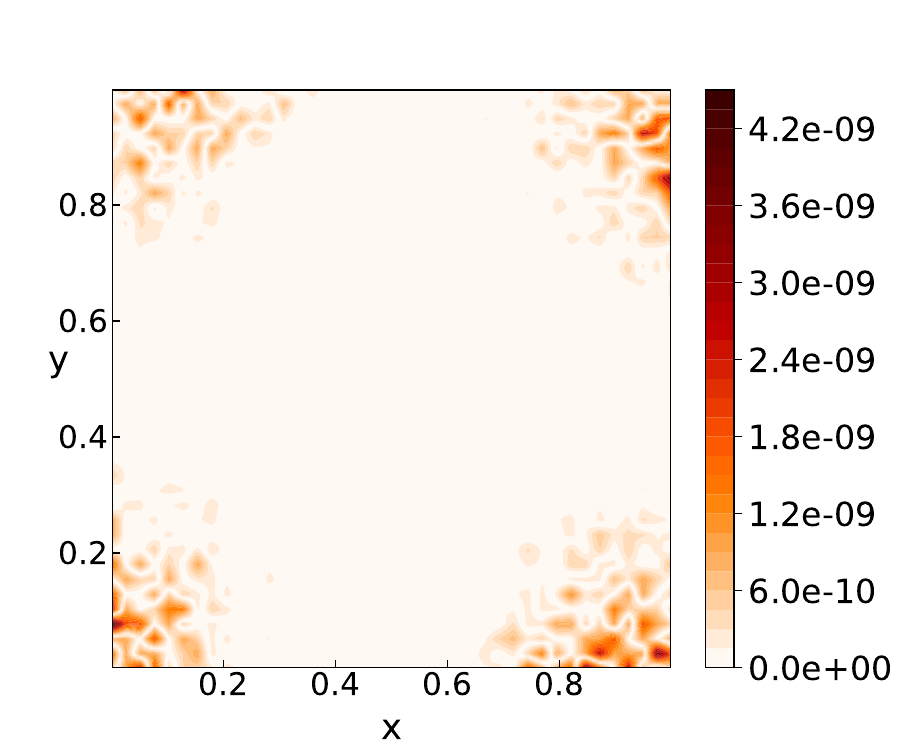}
    \caption{Absolute error}
  \end{subfigure}
  \caption{Additivity verification of DeepRTE with $g=0.5$, $\mus=3$ and $\mut=6$.}
  \label{fig:additivity}
\end{figure}

To be more general, we also validate the additivity property on 100 samples of validation dataset with identical distribution to the training dataset with $g=0.5$ mentioned in above section (see Sec.~\ref{sec:acc}). For every sample in dataset with boundary conditions $I_{-}^{(1)}$, we randomly choose another sample with the boundary conditions $I_{-}^{(2)}$, construct new boundary condition $I_{-}^{(1)}{+}I_{-}^{(2)}$ and then use DeepRTE to evaluate the intensity of this three cases, finally compare $\ANN (I_{-}^{(1)}{+}I_{-}^{(2)})$ and $( \ANN I_{-}^{(1)}{+}\ANN I_{-}^{(2)})$. The MSE and RMSPE between two sides of~\eqref{eq:additivity-val} are $1.421\times 10^{-19}$ and $9.20\times10^{-6}\%$ accordingly, see Table~\ref{tab:linearity}.

\textbf{2. Homogeneity}: The predicted solution corresponding to a scaled boundary condition is the scaled solution:
\begin{equation}\label{eq:homogeneity-val}
  \ANN(\alpha I_{-}) = \alpha \ANN I_{-}, \quad \text{for any } \alpha \in \mathbb{R}.
\end{equation}
Similar to the additivity test, one of the example boundary condition $I_{-}$ chosen from our validation dataset as follows,
\begin{equation}
  I_{-}:
  \begin{cases}
    y_l'=0.3875, y_r'=0.6125, x_b'=0.6125, x_t'=0.3875, \\
    c_l'=0.6931, c_r'=-0.6931, c_b'=0.2871, c_t'=-0.2871, \\
    s_l'=-0.2871, s_r'=0.2871, s_b'=0.6931, s_t'=-0.6931,
  \end{cases}
\end{equation}
with variance $\sigma_{\br}=0.01$, $\sigma_{\bOmega}=0.0075$ and
\begin{equation}
  \alpha=5.
\end{equation}
The MSE error between the left and right sides of~\eqref{eq:homogeneity-val} is:
% \begin{equation}
%   \|\ANN (\alpha I_{-}) - \alpha \ANN I_{-}\|_2= 5.179\times 10^{-13}.
% \end{equation}
\begin{equation}
  \ell(\ANN (\alpha I_{-}), \alpha \ANN I_{-}) = 5.179\times 10^{-13}.
\end{equation}
This verifies that the two sides of~\eqref{eq:homogeneity-val} are nearly identical, thus confirming the homogeneity property of
DeepRTE.

We also report the MSE and RMSPE between predicted $\A^\text{NN} (\alpha I^-)$ and ground truth solution solved by numerical solver $\A I^-$ is $5.842\times10^{-9}$ and $2.19\%$ accordingly. This shows that our model can accurately predict the solution for the scaled boundary condition $\alpha I^-$. The Fig.~\ref{fig:homogeneity} visualizes the results of this homogeneity test using the density $\Phi$ defined  in~\eqref{eq:density}
\begin{figure}[htbp]
  \centering
  \begin{subfigure}[b]{0.3\textwidth}
    \centering
    \includegraphics[width=\textwidth]{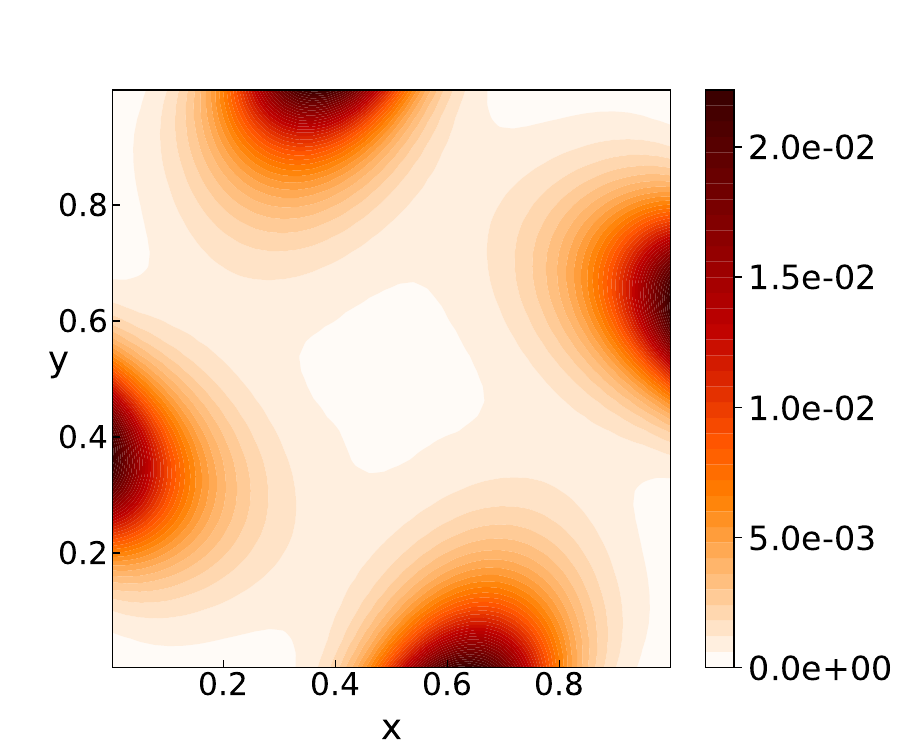}
    \caption{Density $\Phi$ of $\ANN (\alpha I_{-})$.}
  \end{subfigure}
  \hfill
  \begin{subfigure}[b]{0.3\textwidth}
    \centering
    \includegraphics[width=\textwidth]{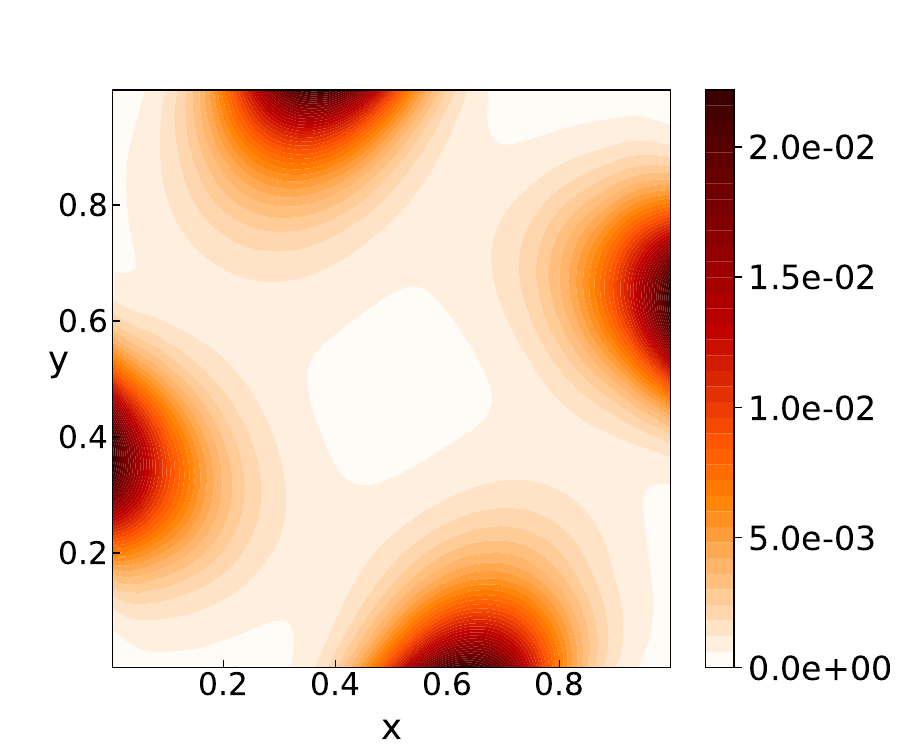}
    \caption{Density $\Phi$ of $\alpha \ANN I_{-}$}
  \end{subfigure}
  \hfill
  \begin{subfigure}[b]{0.3\textwidth}
    \centering
    \includegraphics[width=\textwidth]{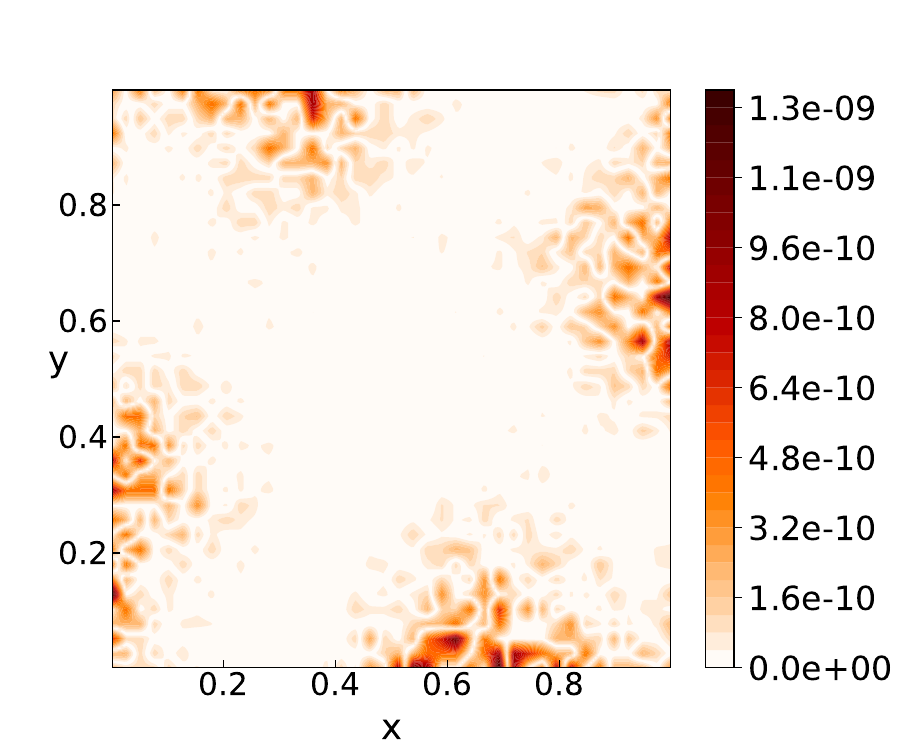}
    \caption{Absolute error}
  \end{subfigure}
  \caption{Homogeneity verification with $g=0.5$, $\mus=3$, $\mut=6$ and $\alpha=5$.}\label{fig:homogeneity}
\end{figure}

To be more general, we also validate the homogeneity property on $100$ samples of validation dataset with identical distribution to the training dataset with $g=0.5$ mentioned in above section (see Sec.~\ref{sec:acc}). In homogeneity test, scale factor $\alpha \sim \mathcal{U}(0,5)$, we multiply this random scale factor with each origin boundary condition $I_-$, and compare the $\ANN (\alpha I^-)$ and $\alpha \ANN I^-$. The MSE and RMSPE between two sides of~\eqref{eq:homogeneity-val} are $2.804\times 10^{-13}$ and $1.01\times 10^{-5}\%$ accordingly, see Table~\ref{tab:linearity}.
\begin{table}[htbp]
  \centering
  \begin{tabular}{lcc}
    \toprule
    & MSE  & RMSPE($\%$)  \\
    \midrule
    Additivity & $1.421\times 10^{-19}$  & $9.20\times 10^{-6}$ \\
    Homogeneity & $2.804\times 10^{-13}$  & $1.01\times 10^{-5}$ \\
    \bottomrule
  \end{tabular}
  \caption{Additivity and homogeneity verification in $100$ samples of validation dataset with identical distribution to the training dataset with $g=0.5$ mentioned in above section (see Sec.~\ref{sec:acc}).}
  \label{tab:linearity}
\end{table}

\textbf{3. Linearity combination}:
The linearity of $\ANN$ established above implies that the generalization error for an arbitrary inflow $I_-$ can be decomposed exactly as in the proof of Theorem~\ref{thm:error-estimate} (see also Eq.~\ref{eq:error_estimate}). In this subsection we design a numerical experiment to empirically verify this decomposition and to assess the relative magnitudes of the constituent terms.

To numerically compare the error, for example, we consider a smooth inflow imposed only on the left boundary and zero elsewhere:
\begin{equation}
  \left \{ % tex-fmt: skip
  \begin{aligned}
    & I_-(x=0,y,c>0,s) = \sin(\pi y), \\
    & I_-(x=1,y,c<0,s) = 0,           \\
    & I_-(x,y=0,c,s>0) = 0,           \\
    & I_-(x,y=1,c,s<0) = 0.
  \end{aligned}
  \right. % tex-fmt: skip
\end{equation}

Following Eq.~\eqref{eq:particle-approx} and Eq.~\eqref{eq:delta_defination}, we only need to approximate the left boundary condition as follows:
\begin{equation}
  I_-(x=0,y,c>0,s) = \sin(\pi y) \approx \sum_{i=0}^{39} w_{i} \delta^{\sigma_{\br}}_{\{y_i\}}(y),
\end{equation}
where $\{y_i\}$ are the centers of uniform discrete cells, $g=0.5$, $\sigma_{\br} = 0.01$. And according to ~\cite{chertock2017practical} Section 2.1,
\begin{equation}
  w_i=\frac{1}{40}\sin(\pi y_i).
\end{equation}

To empirically verify the error estimate in Theorem~\ref{thm:error-estimate}, we use numerical solution as the reference solution $I^{\text{ref}}$ and measure the following errors (note all $L^2$ errors are computed by replacing integrals with quadrature over the discrete spatial--angular grid), starting from the regularized delta function approximation error:
\begin{equation}
  \|I_{-} - I^{40}_{-,0.01}\|_{L^2} \approx \frac{1}{160}\|I_{(x=0,y,c>0,s)} - \sum_{i=0}^{39} w_{i} \delta^{\sigma_{\br}}_{\{y_i\}}(y)\|_{\ell^2}=4.799\times 10^{-5},
\end{equation}
and then
\begin{equation}
  \|\ANN_{\theta^*}(I_{-} - I^{40}_{-,0.01})\|_{L^2}\approx 2.190\times 10^{-7}, \quad
  \|\A(I_{-} - I^{40}_{-,0.01})\|_{L^2} \approx 2.208\times 10^{-7},
\end{equation}
and finally the generalization error
\begin{equation}
  \|(\ANN_{\theta^*} - \mathcal{A}) I^{40}_{-,0.01}\|_{L^2}
  =\|\sum_{i=0}^{39}w_i(\ANN_{\theta^*} - \mathcal{A}) \delta^{\sigma_{\br}}_{\{y_i\}}(y)\|_{L^2} \leq \sum_{i=0}^{39}|w_i|\|(\ANN_{\theta^*} - \mathcal{A}) \delta^{\sigma_{\br}}_{\{y_i\}}(y)\|_{L^2} \approx 8.385\times 10^{-5},
\end{equation}
So the total error is bounded by $8.385\times 10^{-5}$.
We also calculate the direct inference error as
\begin{equation}
  \|(\ANN_{\theta^*} - \mathcal{A}) I_{-}\|_{L^2} \approx 1.888\times 10^{-5}.
\end{equation}
This is lower than our theoretical bound which verifies our results.

\subsubsection{Transfer learning and zero-shot performance}

In this section, we explore the DeepRTE framework's capabilities in transfer learning and Zero-Shot performance. Transfer learning enables a model to leverage knowledge from one task for improved performance on a related task, while Zero-Shot capability allows immediate application to entirely new conditions without retraining. To demonstrate these properties for the Radiative Transfer Equation (RTE)—specifically its ability to adapt to novel boundary conditions with minimal or no additional training—we train the model on datasets featuring delta function boundary conditions. The goal is to learn the RTE's fundamental physical principles. Subsequently, we evaluate the model on boundary conditions sampled from an entirely distinct data distribution (different from those in the training set), assessing its generalization capacity and Zero-Shot performance.
% —its ability to predict solutions for entirely new boundary conditions without further training.

To evaluate the model, we consider the following three specific boundary condition cases:
\begin{itemize}
	\item \textbf{Case I (Constant boundary conditions):} This case tests the model's generalization capability from the delta function boundary conditions used in training to a uniform source. It serves as a baseline for evaluating fundamental generalization.
	\item \textbf{Case II (Trigonometric boundary conditions):} This case introduces spatial changes using a sine function, assessing the model's ability to handle smooth, periodic, and spatially dependent boundary conditions.
	\item \textbf{Case III (Velocity dependent boundary conditions):} This case evaluates the model's performance with more complex, nonlinear boundary conditions that introduce directional dependence.
\end{itemize}

\paragraph{Case I. Constant boundary conditions}
The boundary conditions of the problem are
\begin{equation}
	\left \{
	\begin{aligned}
		 & I_-(x=0,y,c>0,s) =1,  \\
		 & I_-(x=1,y,c<0,s) = 0, \\
		 & I_-(x,y=0,c,s>0) =0,  \\
		 & I_-(x,y=1,c,s<0) =0,
	\end{aligned}
	\right.
\end{equation}
\begin{figure}[htbp]
	\centering
	\begin{subfigure}[b]{0.24\textwidth}
		\centering
		\includegraphics[width=\textwidth]{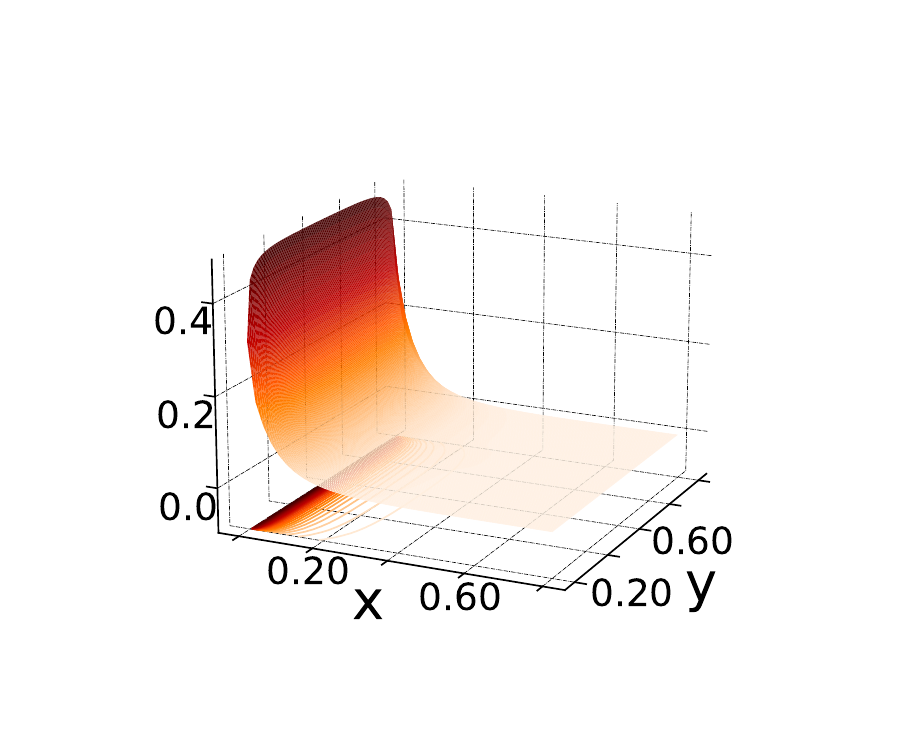}
		\caption{$\Phi_{\text{label}}$ (3D)}\label{fig:caseI_label_3d}
	\end{subfigure}
	\hfill
	\begin{subfigure}[b]{0.24\textwidth}
		\centering
		\includegraphics[width=\textwidth]{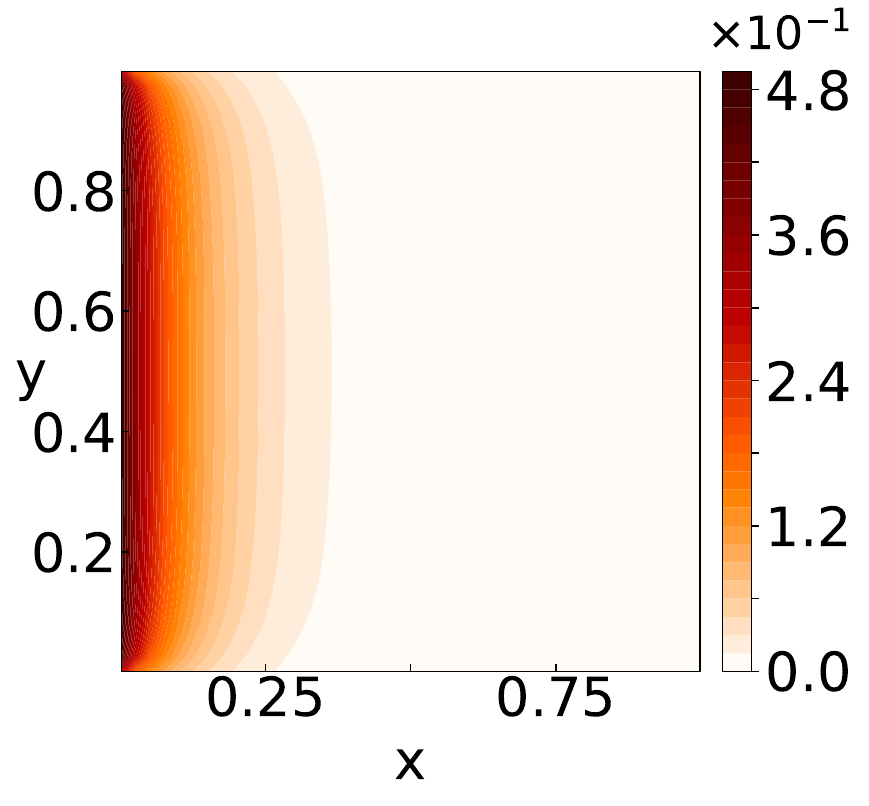}
		\caption{$\Phi_{\text{label}}$}\label{fig:caseI_label}
	\end{subfigure}
	\hfill
	\begin{subfigure}[b]{0.24\textwidth}
		\centering
		\includegraphics[width=\textwidth]{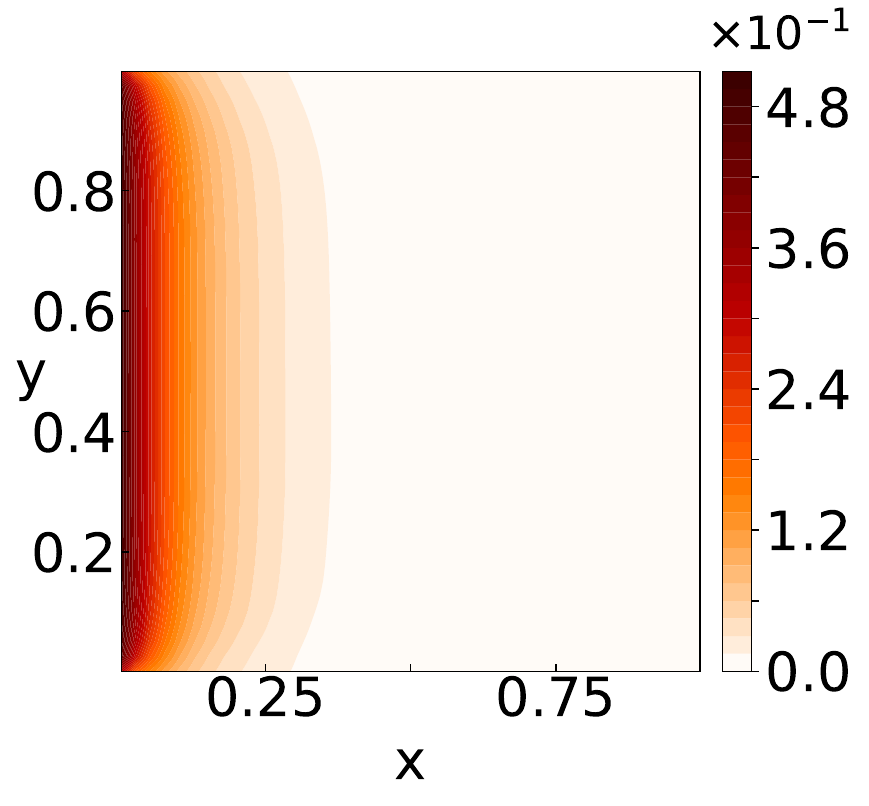}
		\caption{$\Phi_{\text{predict}}$}\label{fig:caseI_predict}
	\end{subfigure}
	\hfill
	\begin{subfigure}[b]{0.24\textwidth}
		\centering
		\includegraphics[width=\textwidth]{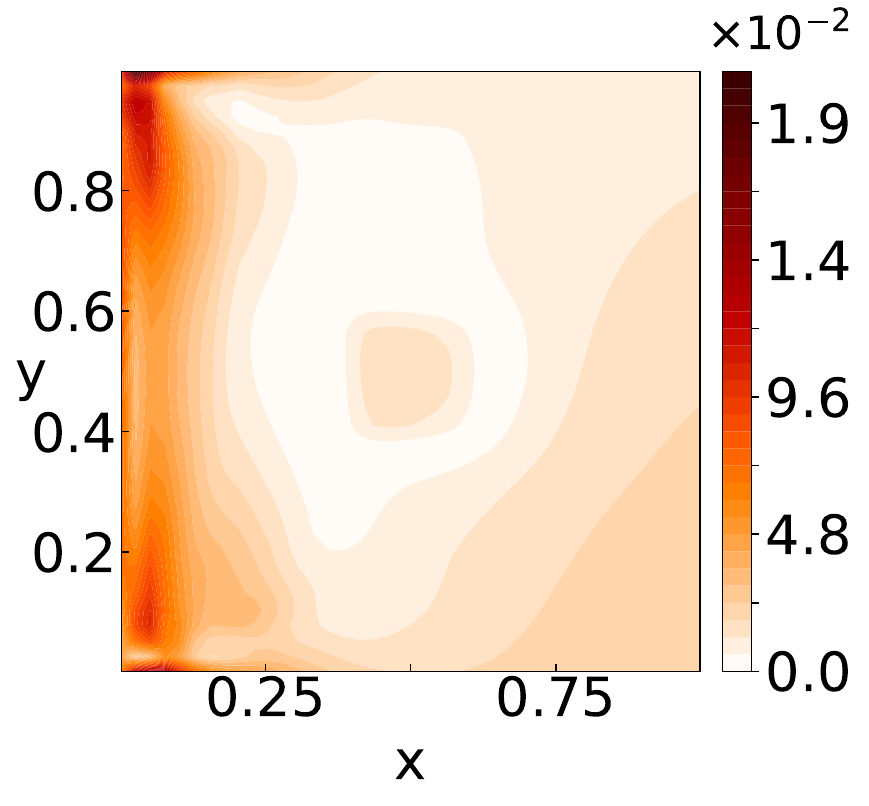}
		\caption{Absolute error}\label{fig:caseI_error}
	\end{subfigure}
	\caption{Evaluation dataset with scattering kernel coefficient
		$g\in (0,0.2)$. The RMSPE of the predict solution is 2.573\%}
\end{figure}

\paragraph{Case II.\@ Trigonometric boundary conditions}
The boundary conditions of the problem are
\begin{equation}
	\left \{
	\begin{aligned}
		 & I_-(x=0,y,c>0,s) =a_L\sin{k_L y}+5,   \\
		 & I_-(x=1,y,c<0, s) = a_R\sin{k_R y}+5, \\
		 & I_-(x,y=0,c,s>0) =a_B\sin{k_B x}+5,   \\
		 & I_-(x,y=1,c,s<0) =a_T\sin{k_T x}+5,
	\end{aligned}
	\right.
\end{equation}

where
\begin{equation}
	a_L, a_R, a_B, a_T \sim \mathcal{U}(-5,5), \quad
	k_L, k_R, k_B, k_T \sim\mathcal{U}(-10, 10).
\end{equation}
\begin{figure}[htbp]
	\centering
	\begin{subfigure}[b]{0.24\textwidth}
		\centering
		\includegraphics[width=\textwidth]{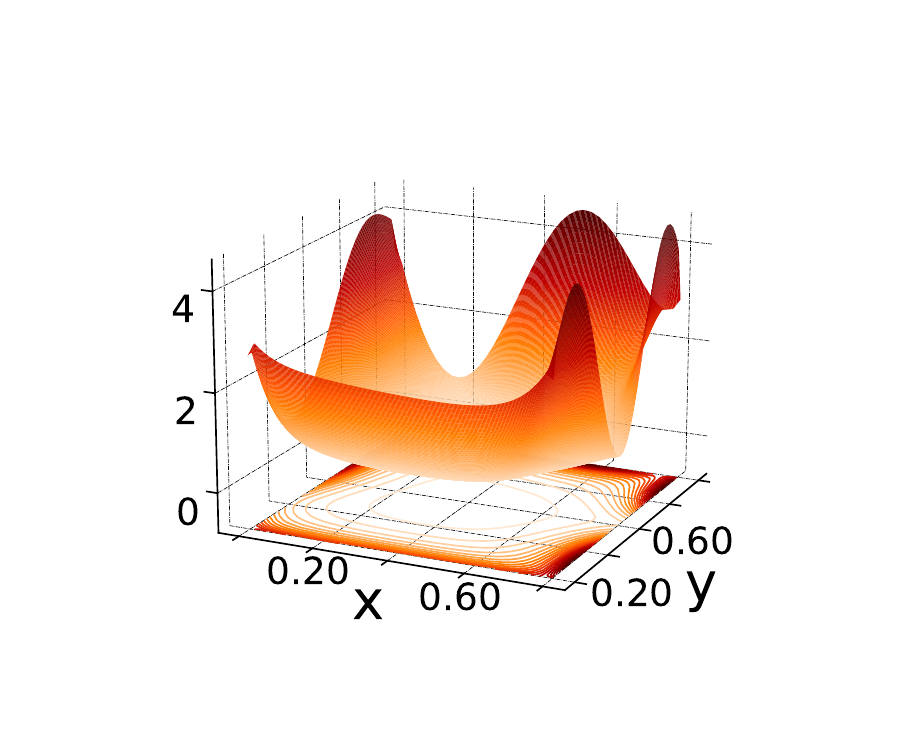}
		\caption{$\Phi_{\text{label}}$ (3D)}\label{fig:caseII_label_3d}
	\end{subfigure}
	\hfill
	\begin{subfigure}[b]{0.24\textwidth}
		\centering
		\includegraphics[width=\textwidth]{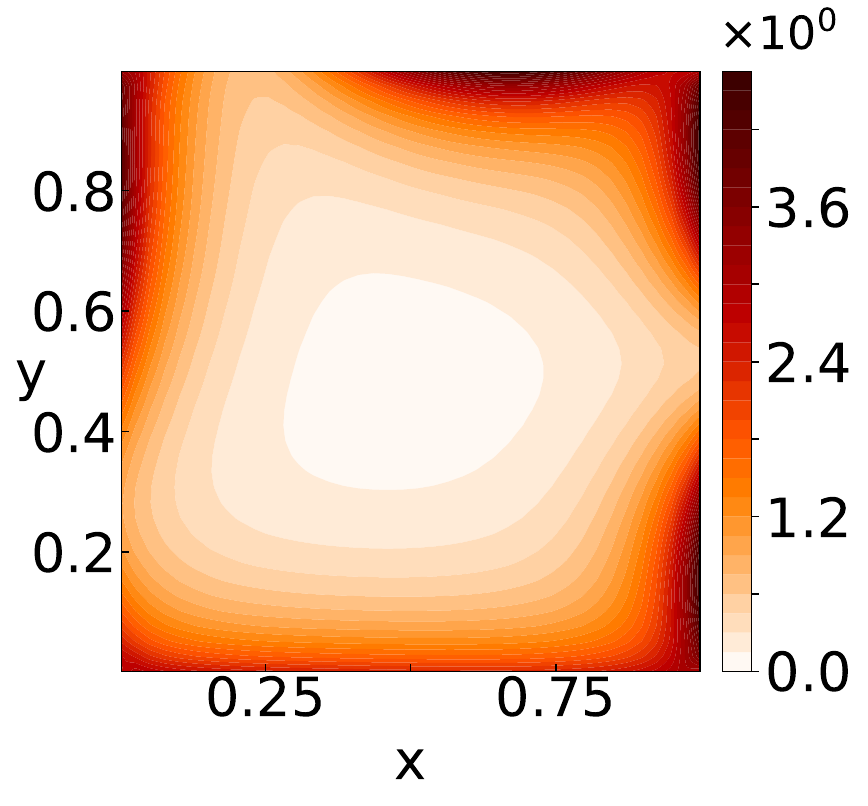}
		\caption{$\Phi_{\text{label}}$}\label{fig:caseII_label}
	\end{subfigure}
	\hfill
	\begin{subfigure}[b]{0.24\textwidth}
		\centering
		\includegraphics[width=\textwidth]{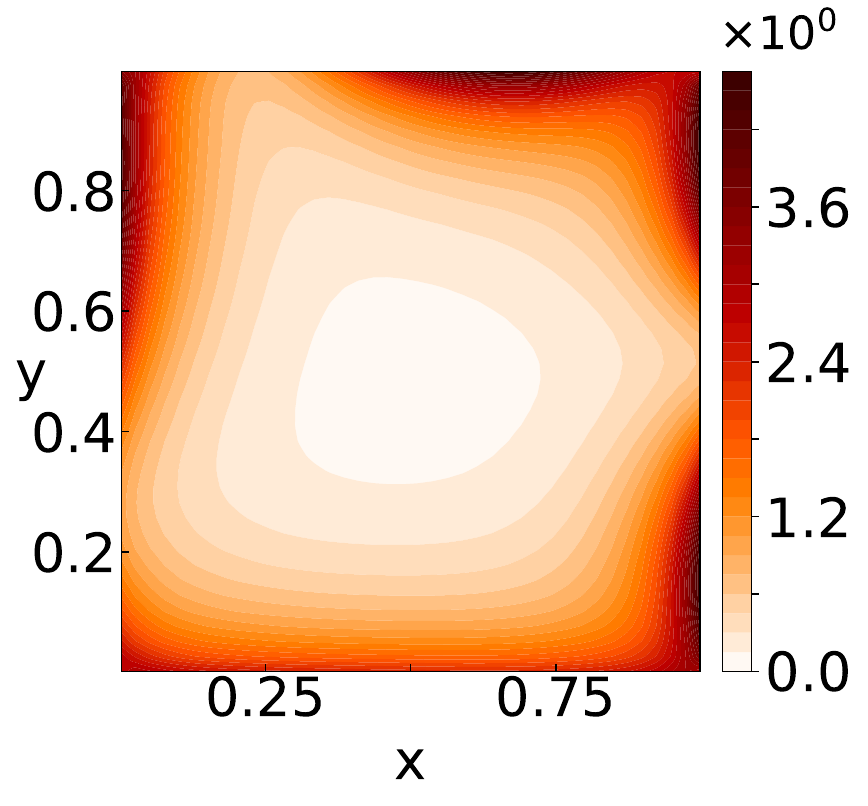}
		\caption{$\Phi_{\text{predict}}$}\label{fig:caseII_predict}
	\end{subfigure}
	\hfill
	\begin{subfigure}[b]{0.24\textwidth}
		\centering
		\includegraphics[width=\textwidth]{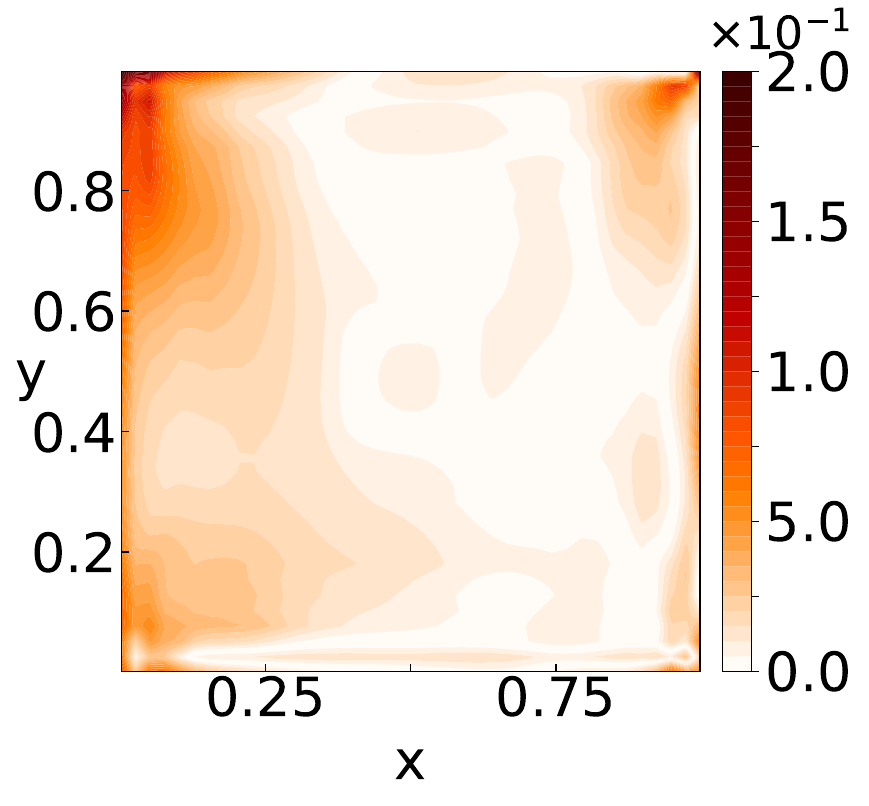}
		\caption{Absolute error}\label{fig:caseII_error}
	\end{subfigure}
	\caption{Evaluation data set with scattering kernel coefficient
		$g\in (0,0.2)$. The RMSPE of the predict solution is 1.968\%}
\end{figure}

\paragraph{Case III.\@ Velocity dependent boundary conditions}The boundary conditions of the problem are
\begin{equation}
	\left \{
	\begin{aligned}
		 & I_-(x=0,y,c>0,s) =(a_{Lr}\sin{k_{Lr} y}+5)(a_{Lv}\sin{k_{Lv} c}+1)(a_{Lv}\sin{k_{Lv} s}+1),  \\
		 & I_-(x=1,y,c<0,s) = (a_{Rr}\sin{k_{Rr} y}+5)(a_{Rv}\sin{k_{Rv} c}+1)(a_{Rv}\sin{k_{Rv} s}+1),
		\\
		 & I_-(x,y=0,c,s>0) =(a_{Br}\sin{k_{Br} x}+5)(a_{Bv}\sin{k_{Bv} c}+1)(a_{Bv}\sin{k_{Bv} s}+1),
		\\
		 & I_-(x,y=1,c,s<0) =(a_{Tr}\sin{k_{Tr} x}+5)(a_{Tv}\sin{k_{Tv} c}+1)(a_{Tv}\sin{k_{Tv} s}+1),
	\end{aligned}
	\right.
\end{equation}

where
\begin{equation}
	\begin{aligned}
		a_{Lr}, a_{Rr}, a_{Br}, a_{Tr} & \sim \mathcal{U}(-5,5),   \\
		a_{Lv}, a_{Rv}, a_{Bv}, a_{Tv} & \sim \mathcal{U}(-1,1),   \\
		k_{Lr}, k_{Rr}, k_{Br}, k_{Tr} & \sim\mathcal{U}(-10, 10), \\
		k_{Lv}, k_{Rv}, k_{Bv}, k_{Tv} & \sim\mathcal{U}(-6, 6).
	\end{aligned}
\end{equation}
\begin{figure}[htbp]
	\centering
	\begin{subfigure}[b]{0.24\textwidth}
		\centering
		\includegraphics[width=\textwidth]{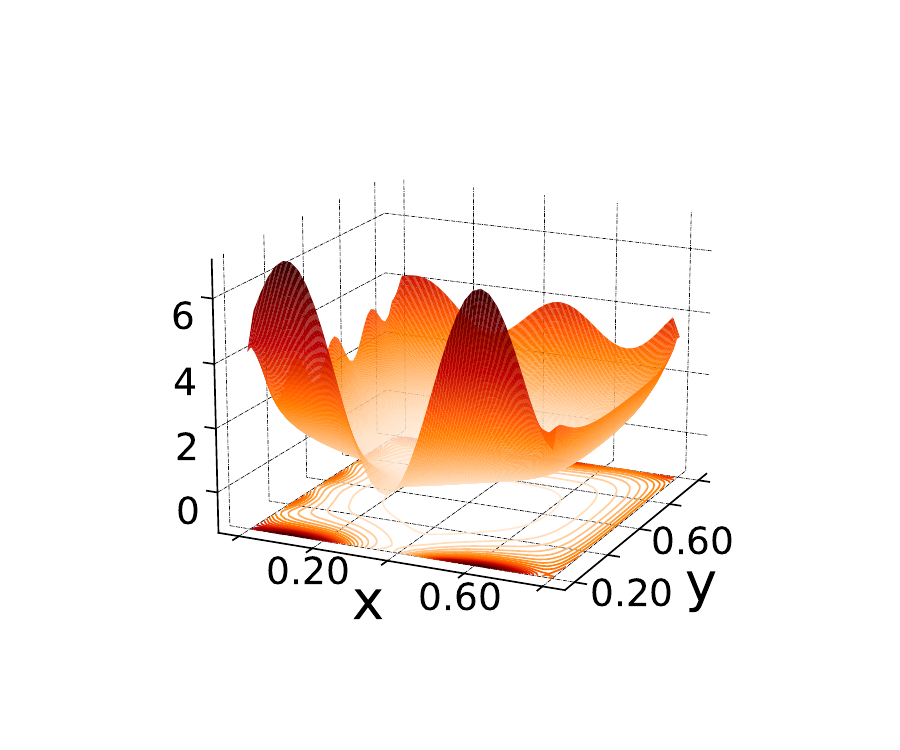}
		\caption{$\Phi_{\text{label}}$ (3D)}\label{fig:caseIII_label_3d}
	\end{subfigure}
	\hfill
	\begin{subfigure}[b]{0.24\textwidth}
		\centering
		\includegraphics[width=\textwidth]{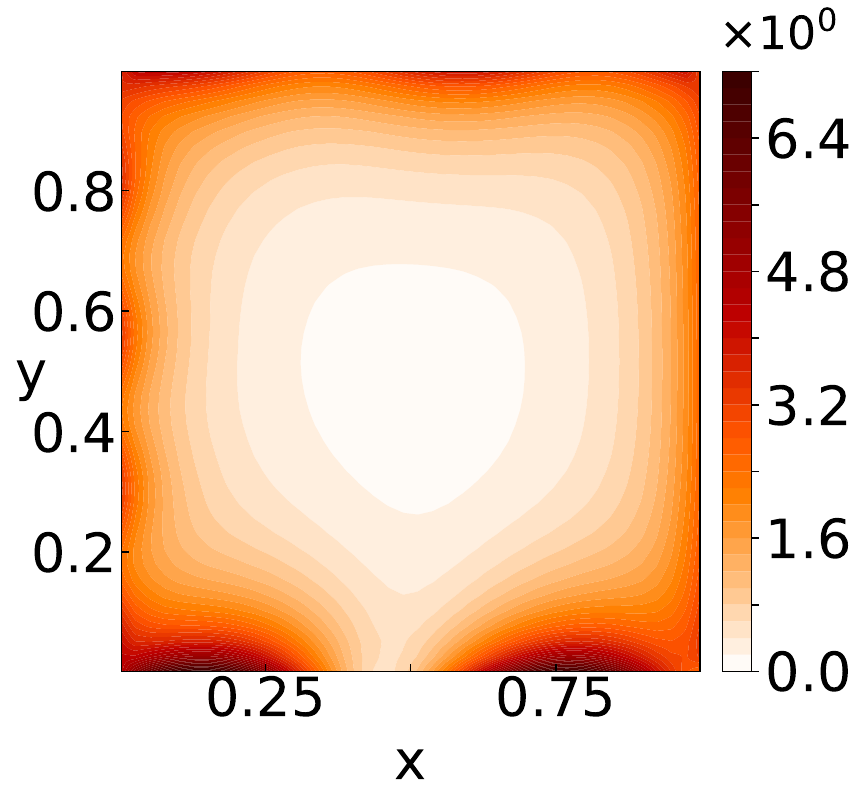}
		\caption{$\Phi_{\text{label}}$}\label{fig:caseIII_label}
	\end{subfigure}
	\hfill
	\begin{subfigure}[b]{0.24\textwidth}
		\centering
		\includegraphics[width=\textwidth]{test_sin_rv_g0.1_pre.pdf}
		\caption{$\Phi_{\text{predict}}$}\label{fig:caseIII_predict}
	\end{subfigure}
	\hfill
	\begin{subfigure}[b]{0.24\textwidth}
		\centering
		\includegraphics[width=\textwidth]{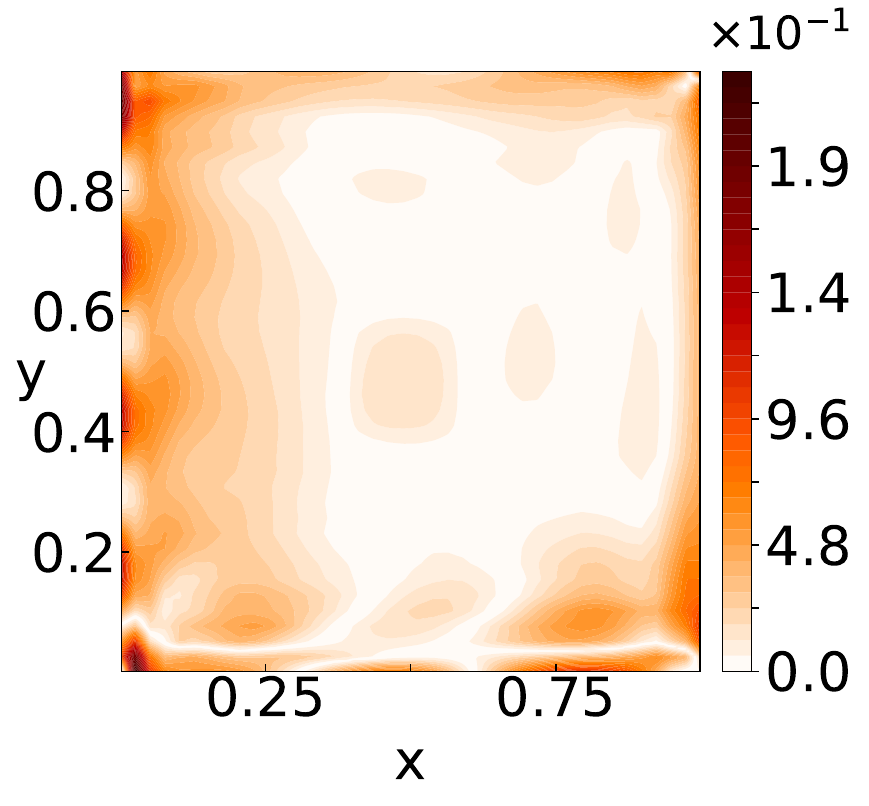}
		\caption{Absolute error}\label{fig:caseIII_error}
	\end{subfigure}
	\caption{Evaluation data set with scattering kernel
		coefficient $g\in (0,0.2)$. The RMSPE of the predict
		solution is 1.908\%}
\end{figure}

\begin{table}[htbp]
	\centering
	\begin{tabular}{@{}cccc@{}}
		\toprule
		\textbf{}                 & Test dataset    & MSE                    & RMSPE (\%) \\ \midrule
		\multirow{3}{*}{Case I}   & $g\in(0,0.2)$   & $4.390 \times 10^{-6}$ & 1.833      \\
		                          & $g\in(0.4,0.6)$ & $5.184 \times 10^{-6}$ & 1.994      \\
		                          & $g\in(0.7,0.9)$ & $1.474 \times 10^{-5}$ & 3.193      \\ \midrule
		\multirow{3}{*}{Case II}  & $g\in(0,0.2)$   & $4.931 \times 10^{-4}$ & 1.653      \\
		                          & $g\in(0.4,0.6)$ & $5.798 \times 10^{-4}$ & 1.827      \\
		                          & $g\in(0.7,0.9)$ & $2.870 \times 10^{-3}$ & 3.572      \\ \midrule
		\multirow{3}{*}{Case III} & $g\in(0,0.2)$   & $1.065 \times 10^{-3}$ & 2.383      \\
		                          & $g\in(0.4,0.6)$ & $1.127 \times 10^{-3}$ & 2.452      \\
		                          & $g\in(0.7,0.9)$ & $1.853 \times 10^{-3}$ & 3.069      \\ \bottomrule
	\end{tabular}
	\caption{Transfer learning and Zero-Shot performance. The model, trained on delta function boundary conditions, is tested on constant (Case I), trigonometric (Case II), and combined trigonometric (Case III) boundary conditions for $g$ in ranges (0,0.2), (0.4,0.6), and (0.7,0.9). Performance is measured using mean square error (MSE) and root mean square percentage error (RMSPE), demonstrating robust generalization with errors below 3.2\% for Case I, 3.6\% for Case II, and 3.1\% for Case III.}
\end{table}

\paragraph{Out-of-distribution possibility for the scattering kernel}
One last experiment we conducted is to test the performance of DeepRTE when the scattering kernel is outside the distribution of our pretraining dataset.
We test the performance as $g$ approaches $1$ and present the results for $g = 0.99$ under the following experimental setup:
\begin{equation}
	\mu_t(x,y) =
	\begin{cases}
		6, & \text{if } (x,y) \in D_\mu    \\
		10 & \text{if } (x,y) \notin D_\mu
	\end{cases},
	\quad
	\mu_s(x,y) =
	\begin{cases}
		3, & \text{if } (x,y) \in D_\mu    \\
		5  & \text{if } (x,y) \notin D_\mu
	\end{cases},
	\quad \text{where } D_\mu = [0.4, 0.6]^2,
\end{equation}
and the incoming boundary conditions are:
\begin{equation}
	\left \{
	\begin{aligned}
		 & I_-(x=0,y,c>0,s) =(5\sin{2\pi y}+5)(\sin{\pi c}+1)(\sin{\pi s}+1),  \\
		 & I_-(x=1,y,c<0,s) = (5\sin{2\pi y}+5)(\sin{\pi c}+1)(\sin{\pi s}+1),
		\\
		 & I_-(x,y=0,c,s>0) =(5\sin{2\pi x}+5)(\sin{\pi c}+1)(\sin{\pi s}+1),
		\\
		 & I_-(x,y=1,c,s<0) =(5\sin{2\pi x}+5)(\sin{\pi c}+1)(\sin{\pi s}+1).
	\end{aligned}
	\right.
\end{equation}

\begin{figure}[htbp]
	\centering
	\begin{subfigure}[t]{0.32\textwidth}
		\centering
		\includegraphics[width=\linewidth]{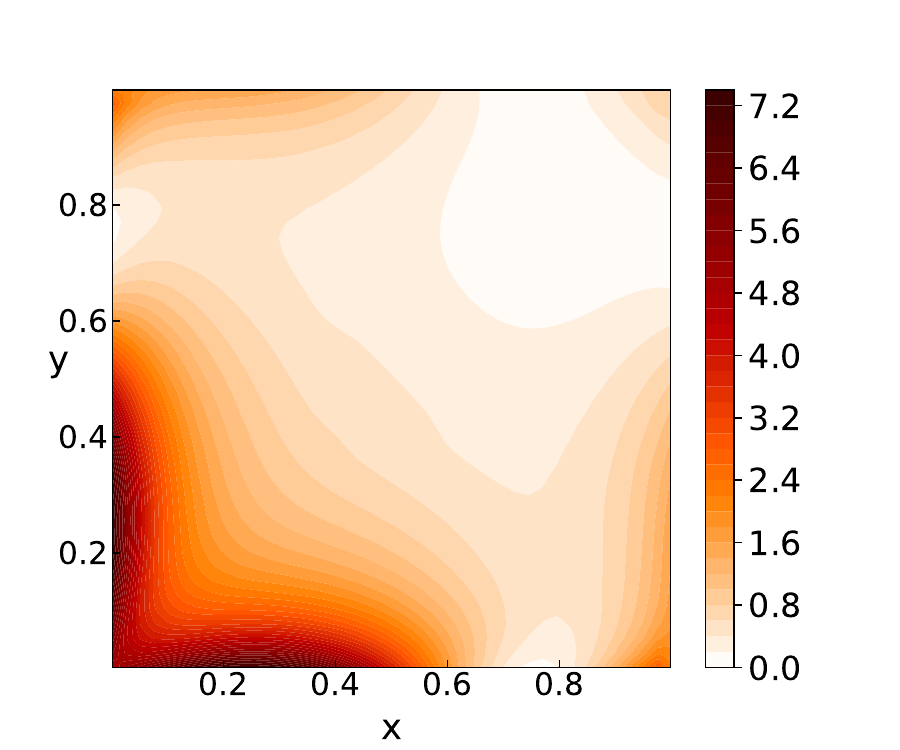}
		\caption{$\Phi_{\text{label}}$}
		\label{fig:deeprte-g099-label}
	\end{subfigure}\hfill
	\begin{subfigure}[t]{0.32\textwidth}
		\centering
		\includegraphics[width=\linewidth]{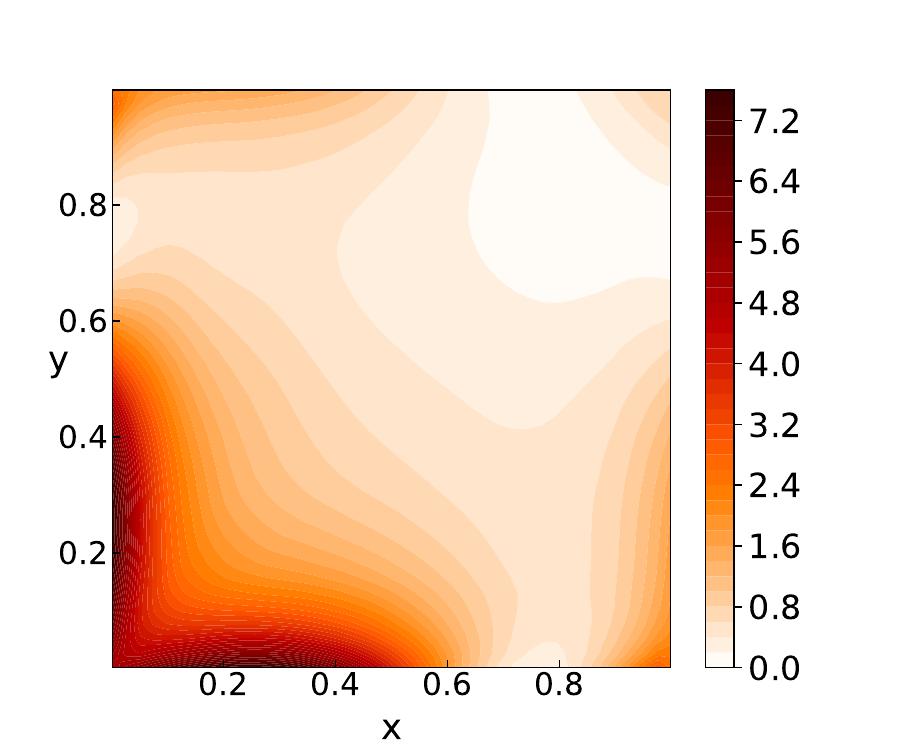}
		\caption{$\Phi_{\text{predict}}$}
		\label{fig:deeprte-g099-predict}
	\end{subfigure}\hfill
	\begin{subfigure}[t]{0.32\textwidth}
		\centering
		\includegraphics[width=\linewidth]{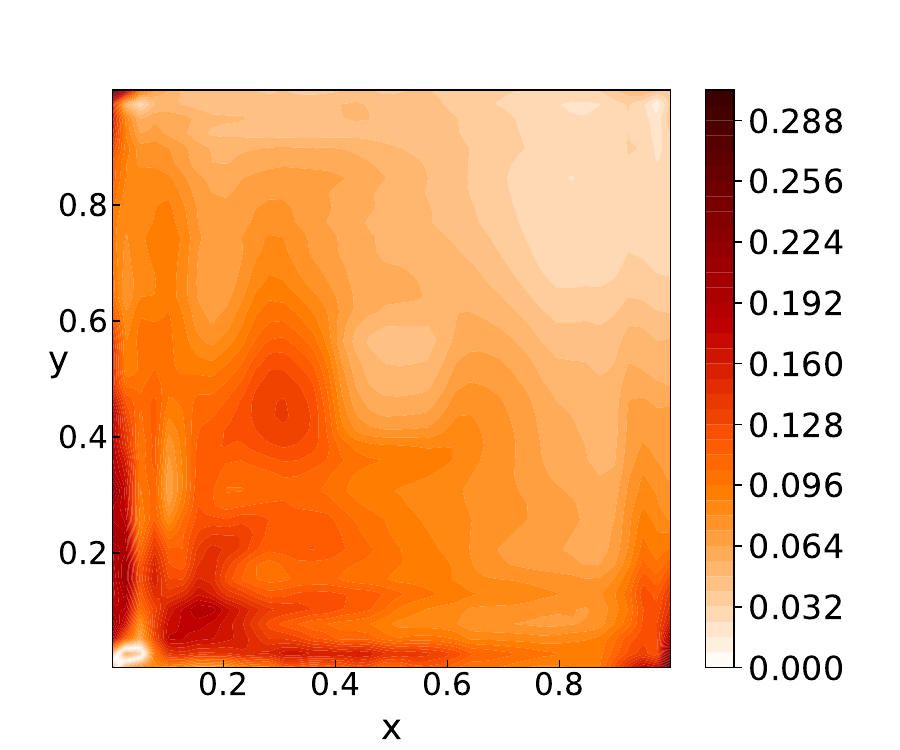}
		\caption{Absolute error}
		\label{fig:deeprte-g099-error}
	\end{subfigure}
	\caption{DeepRTE evaluation ($g=0.99$).}
	\label{fig:deeprte-g099}
\end{figure}

The figure above illustrates the reference solution $\Phi_{\text{label}}$, the predicted solution $\Phi_{\text{predict}}$, and the absolute error between them. Despite the challenging scenario with $g$ extremely close to 1, DeepRTE demonstrates remarkable accuracy. The quantitative metrics, including MSE of $0.00764$ and RMSPE of $4.257\%$, confirm DeepRTE's effectiveness and robustness in handling near-unity anisotropy. This performance highlights DeepRTE's capability to accurately solve the radiative transfer equation even under highly-peaked regime.

In conclusion, DeepRTE's transfer learning and zero-shot capabilities demonstrate its strength and effectiveness across different conditions for solving the radiative transport equation.
By applying the core physical principles learned during training, DeepRTE can accurately predict solutions for entirely new boundary conditions, even without additional training.
This ability to generalize and adapt to new problems highlights the potential of deep learning frameworks to advance scientific computing.

\subsection{Mesh dependence and cross-resolution performance}

An important characteristic of any numerical method for solving partial differential equations is its performance across various mesh resolutions. In this section, we test the mesh independence property of DeepRTE. This capability allows parameters learned on fine meshes to be effectively applied to coarser meshes without significant loss of accuracy.

\begin{table}[htbp]
	\centering
	\begin{tabular}{cccc}
		\toprule
		Test dataset & Mesh resolution & MSE                    & RMSPE (\%) \\
		\midrule
		\multirow{3}{*}{Evaluation Dataset}
		             & $40\times 40$   & $5.453\times 10^{-10}$ &
		$2.759$                                                              \\
		             & $20\times 20$   & $8.235 \times 10^{-9}$ &
		$10.006$                                                             \\
		             & $10\times 10$   & $9.476 \times 10^{-8}$ &
		$34.346$                                                             \\
		\midrule
		\multirow{3}{*}{Case I}
		             & $40\times 40$   & $4.390 \times 10^{-6}$ &
		$1.833$                                                              \\
		             & $20\times 20$   & $1.876 \times 10^{-5}$ &
		$3.758$                                                              \\
		             & $10\times 10$   & $1.243 \times 10^{-4}$ &
		$9.276$                                                              \\
		\midrule
		\multirow{3}{*}{Case II}
		             & $40\times 40$   & $4.931 \times 10^{-4}$ &
		$1.653$                                                              \\
		             & $20\times 20$   & $1.792 \times 10^{-2}$ &
		$9.952$                                                              \\
		             & $10\times 10$   & $3.687 \times 10^{-2}$ &
		$13.798$                                                             \\
		\midrule
		\multirow{3}{*}{Case III}
		             & $40\times 40$   & $1.065 \times 10^{-3}$ &
		$2.383$                                                              \\
		             & $20\times 20$   & $1.175 \times 10^{-2}$ &
		$8.132$                                                              \\
		             & $10\times 10$   & $4.511 \times 10^{-2}$ &
		$15.477$                                                             \\
		\bottomrule
	\end{tabular}
	\caption{Performance of DeepRTE across different resolutions and test sets. The Evaluation Dataset shows very high precision, with an MSE of $5.45\times10^{-10}$ on the $40\times 40$ grid. In contrast, Cases I-III show a consistent increase in error on coarser grids. This increase is mainly because of the limitations of numerical integration. With the second-order midpoint rule, halving the grid resolution should theoretically make the integration error four times larger, which matches our findings.}\label{tab:mesh_independence_multiple_testsets}
\end{table}

For the evaluation dataset and Case I, the root mean square percentage error rises from $1.8$--$2.8\%$ on the $40\times 40$ grid to $3.8$--$10\%$ on the $20\times 20$ grid. Cases II and III exhibit a larger but still manageable increase, with errors growing from $1.7$--$2.4\%$ to $8$--$10\%$ on the coarser grid.

This fourfold error increase is expected due to the theoretical limitations of the second-order midpoint rule used for the outermost numerical integration. This represents a key finding: DeepRTE has successfully learned the fundamental solution for problems with delta-function inflow boundary conditions, and the observed errors primarily originate from the numerical integration method rather than the neural network's approximation of the solution. Consequently, accuracy can be improved in the future by adopting more advanced numerical integration methods in place of the current midpoint approach. We plan to explore this further.

DeepRTE's ability to maintain good accuracy across different mesh resolutions suggests that the network is learning the basic physical principles of the radiative transport equation, not just fitting to a specific grid setup. This mesh independence has several important benefits:
\begin{enumerate}
	\item Computational efficiency: Users can train the model on
	      fine meshes for high accuracy, but use it on coarser
	      meshes for faster calculations when lower accuracy is acceptable.
	\item Flexibility: The same trained model can be used for
	      problems with different mesh needs without retraining.
\end{enumerate}

\section{Ablation study and comparison with baseline model}\label{sec:ablation-study}
Our DeepRTE architecture is designed to mirror the mathematical structure of the RTE solution by decomposing it into an attenuation component that models transport along characteristics (operators $\J$ and $\cL$) and a scattering component that captures scattering behavior ($\cS$), as described in Section~\ref{sec:architecture}.
Each component is implemented via distinct neural modules within DeepRTE: the Attenuation module (Section~\ref{sec:attenuation-module}) and the Scattering module (Section~\ref{sec:scattering-module}).
To validate this design choice, we conduct carefully designed numerical experiments to assess whether each learned component corresponds directly to its analytical counterpart.
As a control, we also compare DeepRTE against a baseline model that does not incorporate any RTE-specific inductive biases.

We compare DeepRTE against the Multi-Input Operator (MIO)~\cite{jin2022mionet} (an operator learning framework derived from DeepONet to accept multiple input functions) as our baseline model, which employs a monolithic architecture without specialized design for the RTE. We choose MIO as the baseline for the following reasons:
\begin{itemize}
	\item MIO is one of the most commonly used operator learning frameworks that can handle multi-function inputs, making it suitable for learning the RTE solution operator that depends on multiple input functions ($I_-$, $\mut$, $\mus$, and $p$).
	\item MIO (DeepONet) can theoretically approximate any continuous operator, including the RTE solution operator, given sufficient training data and network capacity~\cite{lu2021learning}.
	\item MIO does not incorporate any physics-informed inductive biases specific to the RTE, allowing us to isolate the impact of DeepRTE's architecture that reflects the RTE's mathematical structure.
\end{itemize}
While FNO~\cite{li2020fourier} is another potential baseline, it is not suitable for this task due to several fundamental limitations:
\begin{itemize}
	\item FNO is designed for learning operators that map functions on regular grids to functions on the same grid, which does not align with the RTE's requirements for handling angular variables and complex boundary conditions.
	\item FNO relies on Fourier transforms to capture global (non-local) information, but this approach faces significant challenges in the high-dimensional RTE solution space (spatial and angular dimensions), leading to prohibitive computational and memory costs. In our preliminary tests, these costs made FNO training infeasible within our computational budget.
	\item The RTE operates on phase space involving both spatial and angular variables, whereas FNO typically applies to spatial domains only. Adapting FNO to handle angular dependencies would require non-trivial modifications (e.g., determining whether to apply Fourier transforms in the angular domain), which are not straightforward and may not be effective for this specific problem.
\end{itemize}

Due to these reasons, we focus our comparison on MIO as the baseline model. We first validate the effectiveness of the attenuation and scattering modules in DeepRTE through ablation studies. We then compare the overall performance of DeepRTE against MIO on the RTE solution operator learning task.

\subsection{Ablation study of Attenuation module}

The attenuation module models the transport behavior of the RTE without scattering effects.
This behavior is characterized by two key processes: the attenuation of boundary conditions through the total cross section $\mut$ (operator $\J$, Eq.~\eqref{eq:attenuation-op}) and the attenuation of volumetric sources through the scattering cross section $\mus$ (lifting operator $\mathcal{L}$, Eq.~\eqref{eq:lifting-op}).

The fundamental challenge lies in the fact that the solution operator $\A$'s dependence on $\mus$ and $\mut$ is both nonlinear and exhibits \emph{non-local behavior exclusively along characteristic lines}.
Conventional architectures like MIO only capture general non-local dependencies through branch networks and inner products with trunk networks, which becomes computationally expensive as spatial and angular dimensions increase. More critically, this approach fails to accurately capture the RTE's characteristic-based attenuation behavior and may not generalize effectively.

In contrast, DeepRTE incorporates a specialized attenuation module that explicitly models this non-local dependence along characteristic lines (which remain one-dimensional regardless of problem dimensionality) through carefully designed attention mechanisms, as detailed in Section~\ref{sec:attenuation-module}.
To validate the effectiveness of this physics-informed design, we conduct two targeted experiments: one examining the attenuation operator $\J$ and another focusing on the lifting operator $\cL$.

\textbf{1. Ablation study of operator $\J$.}
We isolate the attenuation operator $\J$ by considering a pure attenuation scenario without scattering, governed by~\eqref{eq:rte-sweep}:
\begin{equation}
	\left\{ % tex-fmt: skip
	\begin{aligned}
		\bOmega \cdot \nabla J_b + \mut J_b & = 0,     \\
		J_b|_{\Gamma_{-}}                   & = I_{-},
	\end{aligned}
	\right. % tex-fmt: skip
\end{equation}
which has the analytical formulation~\eqref{eq:attenuation-op}:
\begin{equation}
	\J: (I_{-}; \mut) \mapsto J_b = e^{-\tau_{\br,\bOmega}(0,s_{-}(\br,\bOmega))} I_{-}\left(\br-s_{-}(\br, \bOmega)
	\bOmega, \bOmega\right).
\end{equation}

According to our architecture design in Section~\ref{sec:attenuation-module}, we use the Attenuation module to learn the mapping $\J^{\text{NN}}$ that approximates $\J$ as follows:
\begin{equation}
	J_b(\br,\bOmega) \approx \J^{\text{NN}}I_{-} := \int_{\Gamma_{-}} G^{\text{NN}}_{\J}(\br,\bOmega, \br',\bOmega'; \mut) I_-(\br',\bOmega')\diff{\br'}\diff{\bOmega'},
\end{equation}
where $G^{\text{NN}}_{\J}$ is analogous to the general attenuation Green's function~\eqref{G_equation}, but simplified to depend only on $\tau^{\text{NN}}_{-,t}$ as defined in~\eqref{eq:optical-depth-net}. Since the attenuation operator $\J$ depends exclusively on $\mut$ (and not on $\mus$), we formulate:
\begin{equation}
	G_{\J}^{\text{NN}}(\br,\bOmega, \br', \bOmega';\mut) = \text{MLP}(\br,\bOmega, \br', \bOmega', \tau^{\text{NN}}_{-,t}) \in \mathbb{R}^1.
\end{equation}

In our numerical experiments, we implement this by removing all scattering-related components: we set the number of scattering blocks $N_\ell = 0$ and reduce the dimension of $\mu$ from $2$ to $1$, effectively disabling attention mechanisms that depend on $\mu_s$.

To validate that our architecture actually capture the attenuation effect, we compare it with the MIO that doesn't have any RTE specialized architecture. The MIO network we use here is quite standard as, see~\cite{jin2022mionet}:
\begin{equation}
	J_b(\br,\bOmega) \approx \text{MIO}_{\J}(\mut,I_{-}) := (\text{MLP}^{\text{branch}}_1(\mut)\odot\text{MLP}^{\text{branch}}_2(I_-))\cdot \text{MLP}^{\text{trunk}}(\br,\bOmega).
\end{equation}

For dataset construction, except for $\mus$ and $g$, all other settings remain unchanged (see Table~\ref{tab:dataset}). We use the same numerical solver as in~\ref{sec:dataset-construction} except that we disable scattering by setting $\mus = 0$ and set $g$ to be irrelevant (not used).
Following these settings, we generate $1000$ samples for training and $100$ i.i.d. samples for testing.

\textbf{2. Ablation study of operator $\cL$.} Similar to $\J$, we isolate the lifting operator $\cL$ by considering a pure lifting scenario without boundary inflow, governed by~\eqref{eq:lifting}:
\begin{equation}
	(\bOmega\cdot\nabla + \mut)J = \mus I, \quad  J|_{\Gamma_{-}} = 0,
\end{equation}
whose analytical formulation is given by~\eqref{eq:lifting-op}:
\begin{equation}
	\cL:(I;\mus,\mut) \mapsto J = \int_0^{s_{-}(\br, \bOmega)} e^{-\tau_{\br,\bOmega}(0,s)}\mus(\br-s\bOmega)I(\br-s\bOmega, \bOmega)\diff{s}.
\end{equation}

According to our architecture design in Section~\ref{sec:attenuation-module}, we use the Attenuation module to learn the mapping $\cL^{\text{NN}}$ that approximates $\cL$ as follows:
\begin{equation}
	J(\br,\bOmega) \approx \cL^{\text{NN}}I := \int_{D\times\sS^{d-1}}
	G_{\cL}^{\text{NN}}(\br,\bOmega,\br',\bOmega';\mut,\mus)I(\br',\bOmega')\diff{\br'}\diff{\bOmega'},
\end{equation}
where $G^{\text{NN}}_{\cL}$ is the attenuation Green's function~\eqref{G_equation}. Since the lifting operator $\cL$ depends on both $\mut$ and $\mus$, according to~\eqref{G_equation} we formulate:
\begin{equation}
	G_{\cL}^{\text{NN}}(\br,\bOmega, \br', \bOmega';\mut,\mus) = \bG^{\text{NN}}(\br,\bOmega, \br', \bOmega', \tau^{\text{NN}}_{-})\bm{W}\in\mathbb{R}^1,
\end{equation}
where $\tau^{\text{NN}}_{-}$ is defined in~\ref{alg:optical-depth-net}, also see Remark~\ref{rmk:optical-depth-encoding}. The output dimension of $\bG^{\text{NN}}$ is $d_\text{model}$ and $\bm{W} \in \mathbb{R}^{d_{\text{model}}\times 1}$ is a learnable weight matrix that projects outputs into a single output channel, like the $\bm{W}$ in~\eqref{scattering_layer}.

Similar to $\J$, we compare it with the $\text{MIO}_{\cL}$ that doesn't have any RTE specialized architecture:
\begin{equation}
	J(\br,\bOmega) \approx \text{MIO}_{\cL}(\mut,\mus;I) := (\text{MLP}^{\text{branch}}_1(\mut)\odot\text{MLP}^{\text{branch}}_2(\mus)\odot\text{MLP}^{\text{branch}}_3(I))\cdot \text{MLP}^{\text{trunk}}(\br,\bOmega).
\end{equation}

For dataset construction, except for $g$ (not needed) and boundary conditions $I_- = 0$ (set to 0), all other settings remain unchanged (see Table~\ref{tab:dataset}).
For the source term $I(\br,\bOmega)$, we generate a two-dimensional Gaussian random field in the frequency domain using spectral synthesis~\cite{liu2019advances, ruan1998efficient}.
We use the same numerical solver as in Sec.~\ref{sec:dataset-construction} except that we replace the scattering with $\mus I$.
Following these settings, we generate $1000$ samples for training and $100$ i.i.d. samples for testing.

\textbf{Results.} The results in Table~\ref{tab:j_l_compare} demonstrate that our Attenuation module $\J^{\text{NN}}$ and $\cL^{\text{NN}}$ substantially outperform $\text{MIO}_{\J}$ and $\text{MIO}_\cL$ across all metrics. Specifically, $\J^{\text{NN}}$ and $\cL^{\text{NN}}$ achieve over an order of magnitude improvement in both training and test MSE, while requiring $34$ times fewer parameters ($37,282$ vs. $4956160$ and $139,362$ vs. $4,792,320$). The test RMSPE is reduced from $44.251\%$, $25.350\%$ for MIO to $2.339\%$, $1.327\%$ for our approach. This stark performance difference validates that the design of our Attenuation module effectively captures the characteristic-based transport behavior inherent in operator $\J$ and $\cL$, whereas the generic MIO architecture fails to learn this specialized mapping efficiently.
\begin{table}
	\centering
	\begin{tabular}{@{}cccccccc@{}}
		\toprule
		\multirow{2}{*}{Op}    & \multirow{2}{*}{Model} & \multirow{2}{*}{\# of Parameters} & \multicolumn{3}{c}{Training} & \multicolumn{2}{c}{Evaluation}                                                   \\
		\cmidrule(l){4-6} \cmidrule(l){7-8}
		                       &                        &                                   & MSE                          & RMSPE (\%)                     & \# Epochs & MSE                    & RMSPE (\%) \\
		\midrule
		\multirow{2}{*}{$\J$}  & Our $\J^{\text{NN}}$   & 37,282                            & $3.631\times 10^{-7}$        & 2.78                           & 200,000   & $4.561 \times 10^{-8}$ & 2.339      \\
		                       & $\text{MIO}_{\J}$      & 4,956,160                         & $1.983\times 10^{-6}$        & 6.459                          & 200,000   & $1.088 \times 10^{-4}$ & 44.251     \\
		\midrule
		\multirow{2}{*}{$\cL$} & Our $\cL^{\text{NN}}$  & 139,362                           & $9.682 \times 10^{-6}$       & 0.909                          & 200,000   & $3.359 \times 10^{-5}$ & 1.327      \\
		                       & $\text{MIO}_{\cL}$     & 4,792,320                         & $8.528 \times 10^{-4}$       & 8.565                          & 200,000   & $1.224 \times 10^{-2}$ & 25.350     \\
		\bottomrule
	\end{tabular}
	\caption{Ablation study of attenuation module. Training and evaluation of $\J$ and $\cL$ operators comparied to MIO (without any RTE specialized architecture) .}\label{tab:j_l_compare}
\end{table}

\subsection{Ablation study of Scattering module}

To examine the necessity of the scattering operator $\cS$ (Eq.~\ref{eq:scattering-op}), we employ a proof-by-contradiction approach. We hypothesize that $\cS$ is unnecessary and test this by substituting the scattering layer—the component that implements $\cS$—with a multi-layer perceptron of equivalent size.
We denote this modified architecture as $\ANN_{\text{w/o }\cS}$.

The only difference between $\ANN_{\text{w/o }\cS}$ and full DeepRTE $\ANN$ lies in the scattering block implementation. Instead of using Eq.~\ref{scattering_layer}, we disable the scattering operator $S^\top$:
\begin{equation}
	\text{ScatteringBlock}_\ell(\bm{G}) = \text{LayerNorm}\Big(\sigma\Big(\bm{W}^{\ell} \bm{G} + \bm{b}^{\ell}\Big)\Big),
\end{equation}
where this is no scatting operator $\cS$ involved, and $\bm{W}^{\ell}$ and $\bm{b}^{\ell}$ are learnable parameters of the MLP. The number of parameters in this MLP is kept approximately equal to that of the original scattering block to ensure a fair comparison.

The dataset and training settings are the same as the original DeepRTE (see Section~\ref{sec:acc}). To evaluate the model's performance under different anisotropy conditions, the parameter $g$ in the training set is sampled uniformly from $0$ to $0.9$. Except this $g$, the training set for $\text{DeepRTE}_{\cS}$ uses the same configuration as the original DeepRTE model. For testing, we test on the same three scattering regimes~\ref{li:regimes}.

If our hypothesis were correct, this modified model would successfully fit the training data and achieve comparable test performance to the original DeepRTE. Conversely, poor training fit and large test errors would refute our assumption and demonstrate the necessity of the scattering operator $\cS$. The experimental results, presented in Table~\ref{tab:ablation_g}, strongly support the latter conclusion. Notice the poor performance of $\ANN_{\text{w/o }\cS}$ on the dataset where scattering highly forward peaked (strong anisotropic). This significant degradation in performance compared to the full DeepRTE model clearly indicates that the scattering operator $\cS$ is essential for accurately modeling the RTE solution operator.
\begin{table}[htbp]
	\centering
	\begin{tabular}{@{}lcccc@{}}
		\toprule
		Model                                                        & Scattering regime      & $g$ range   & MSE                    & RMSPE (\%) \\
		\midrule
		\multirow{3}{*}{Without scattering: $\ANN_{\text{w/o }\cS}$} & Near isotropy          & $(0,\,0.2)$ & $2.657\times10^{-2}$   & $8.251$    \\
		                                                             & Moderate anisotropy    & $(0.4,0.6)$ & $6.852\times10^{-2}$   & $11.395$   \\
		                                                             & Highly forward peaking & $(0.7,0.9)$ & $7.257\times10^{-2}$   & $13.108$   \\
		\midrule
		\multirow{3}{*}{With scattering: $\ANN$}                     & Near isotropy          & $(0,\,0.2)$ & $1.065 \times 10^{-3}$ & 2.383      \\
		                                                             & Moderate anisotropy    & $(0.4,0.6)$ & $1.127 \times 10^{-3}$ & 2.452      \\
		                                                             & Highly forward peaking & $(0.7,0.9)$ & $1.853 \times 10^{-3}$ & 3.069      \\
		\bottomrule
	\end{tabular}
	\caption{Performance across scattering anisotropy regimes (Heney–Greenstein parameter $g$).}
	\label{tab:ablation_g}
\end{table}

\begin{remark}
	We note that in our DeepRTE model, the lifting operator $\cL$~\eqref{eq:lifting-op} appears in both the attenuation module (see~\eqref{eq:L-op} and since $\mus$ is included in $\tau_{-}^{\text{NN}}$, see Remark~\ref{rmk:optical-depth-encoding}) and the scattering module (see~\eqref{eq:scattering-mus}). This design choice reflects our architectural principle of separating operations based on their mathematical domain: the attenuation module handles operators acting on spatial variables along characteristics (captured by attention mechanisms), while the scattering module handles operators acting on angular variables (captured by angular integration). This separation simplifies the model architecture and facilitates extensibility through repeating modular blocks.
\end{remark}

\subsection{Overall comparision with MIO}

To evaluate the generalization performance of MIO, we trained it on a dataset where the parameter $g$ of the scattering kernel was uniformly distributed in the range $[0, 0.2]$. This training dataset is the same to the one used for DeepRTE to ensure a fair comparison.
Furthermore, we varied the scale of MIO from small to large to examine how different MIO configurations affected the results on both the training and validation datasets. The results of this comparison are presented in Table~\ref{compare_table}.

\begin{table}
	\centering
	\begin{tabular}{@{}llccccc@{}}
		\toprule
		\multirow{2}{*}{Models}                  &
		\multirow{2}{*}{$\#$ of parameters}      &
		\multicolumn{3}{c}{Training datasets}    &
		\multicolumn{2}{c}{Evaluation datasets}
		\\ \cmidrule(l){3-7}
		                                         &                                    &
		\begin{tabular}[c]{@{}c@{}}MSE
		\end{tabular}           &
		\begin{tabular}[c]{@{}c@{}}RMSPE \\ (\%)
		\end{tabular} & \multicolumn{1}{l}{$\#$ of epochs} &
		\begin{tabular}[c]{@{}c@{}}MSE
		\end{tabular}           &
		\begin{tabular}[c]{@{}c@{}}RMSPE \\ (\%)
		\end{tabular}                                                      \\ \midrule
		\multicolumn{1}{c}{\multirow{3}{*}{MIO}} &
		\multicolumn{1}{c}{244960}               & $1.690\times 10^{-4}$
		                                         & $65.60\%$
		                                         & 100000                             & $1.984\times
		10^{-4}$                                 & $72.49\%$
		\\
		\multicolumn{1}{c}{}                     &
		\multicolumn{1}{c}{\num{4351744}}        & $4.119\times 10^{-6}$
		                                         & $10.57\%$
		                                         & 100000                             & $1.319\times
		10^{-4}$                                 & $59.09\%$
		\\
		\multicolumn{1}{c}{}                     &
		\multicolumn{1}{c}{8171520}              & $2.140\times 10^{-6}$
		                                         & $7.58\%$
		                                         & 100000                             & $9.912\times
		10^{-5}$                                 & $51.23\%$
		\\ \midrule
		DeepRTE                                  &
		\multicolumn{1}{c}{37954}                & $8.210 \times
		10^{-9}$                                 & $3.16\%$
		                                         & 5000                               &
		$5.630 \times 10^{-10}$                  & $2.83\%$
		%  \\ \multicolumn{1}{l}{}                                &
		% \multicolumn{1}{l}{}
		\\ \bottomrule
	\end{tabular}
	\caption{Performance comparison between MIO and DeepRTE frameworks on radiation transport equation solving. MIO results are shown for three different parameter configurations, while DeepRTE achieves superior accuracy with significantly fewer parameters and training epochs.}\label{compare_table}
\end{table}

The comparison shows that while MIO's performance on the training dataset improves as parameters increase, it still has significant errors on the validation dataset sampled from the same distribution as the training set. This observation highlights MIO's limited generalization capability, despite its ability to fit the training data well. Even after extensive parameter adjustments, MIO struggles to accurately predict the solution function for unseen inputs from the same distribution as the training set and fails to capture the underlying physical principles that control the RTE.

In contrast, DeepRTE's architecture enables it to generalize well to unseen inputs, both within the same distribution and in zero-shot settings. By directly involving the physical information in the attenuation and scattering modules, along with the Green's function integral, DeepRTE uses the fundamental structure of the RTE to make accurate predictions. This ability to generalize beyond the training data is very important for practical applications.

\section{Conclusions}\label{sec:conclusions}
In this paper, we introduce DeepRTE, a neural operator framework for solving the radiative transfer equation (RTE). By combining an attention-based architecture with operator learning strategy, DeepRTE directly maps input parameters such as boundary conditions, scattering coefficients, and scattering kernels to RTE solutions. Its key innovations include (1) parameter efficiency: achieving higher accuracy than large data-driven models with fewer parameters by encoding physical laws into the architecture; (2) zero-shot generalization: robustly predicting solutions for unseen boundary conditions without retraining; and (3) interpretability: maintaining linear, physically meaningful operations via physics-guided design.

The results of our experiments demonstrate the remarkable
transfer learning capabilities of DeepRTE.\@ Despite being
trained on a specific set of boundary conditions, the model
exhibits strong performance when applied to new boundary
conditions. The predicted solutions closely match the ground
truth with low error metrics.
This ability to generalize to unseen boundary conditions
highlights the effectiveness of DeepRTE in capturing the
underlying physical principles of the radiative transport equation.
Furthermore, the zero-shot performance of DeepRTE is
particularly impressive. Without any additional training or
fine-tuning, the model accurately predicts the solution of the
radiative transport equation for completely new boundary
conditions. 
% This zero-shot capability is a testament to the
% physics-informed architecture of DeepRTE, which enables it to
% leverage the learned physical principles to make accurate
% predictions in novel scenarios.

% In conclusion, the transfer learning and zero-shot capabilities
% of DeepRTE demonstrate its robustness and versatility in
% solving the radiative transport equation. By leveraging the
% underlying physical principles learned during training, DeepRTE
% can accurately predict solutions for new boundary conditions,
% even in the absence of additional training data. This ability
% to generalize and adapt to novel scenarios highlights the
% potential of physics-informed deep learning frameworks in
% advancing the field of scientific computing.
In conclusion, DeepRTE shows the power of combining deep learning with physical laws to solve complex scientific problems. Its ability to adapt to new conditions without extra training proves how physics-guided models can achieve accurate, reliable results even in unfamiliar scenarios. This approach not only improves efficiency for solving the radiative transfer equation but also offers a practical template for tackling other challenging equations in science and engineering.
\section{Acknowledgments}
This work is mainly supported by National Natural Science Foundation of China (NSFC) funding No. 92270120.
Zheng Ma is also supported by NSFC funding No. 12201401 and Beijing Institue of Applied Physics and Computational Mathematics funding HX02023-60.
Min Tang is supported by NSFC funding No. 12031013 and Shanghai pilot innovation project 21JC1403500.
We also thank Dr. Jingyi Fu for useful discussions.

\appendix
% Notation table (grouped theoretical vs numerical/model). Requires booktabs.
\section{Table of Notation}
\begin{longtable}[c]{c c p{0.78\textwidth}}
  \toprule
  Notation                 & Section & Definition                                                                                                               \\
  \midrule
  \endhead

  \midrule
  \endfoot

  \bottomrule
  \endlastfoot

  % \multicolumn{3}{l}{\textbf{Theoretical symbols}}                                                                                                              \\
  \(d\)                    &    $\S 1$     & Spatial dimension                                                                                                        \\
  \(D\)                    &    $\S1$     & Spatial domain $D\subset \mathbb{R}^d$                                                                                   \\
  \(\partial D\)           &    $\S1$    & Boundary of $D$                                                                                                          \\
  \(\sS^{d-1}\)            &    $\S1$     & Unit sphere in $\mathbb{R}^d$                                                                                            \\
  \(S_{d-1}\)              &    $\S1$     & Surface area of $\sS^{d-1}$, $S_{d-1}=2\pi^{d/2}/\Gamma(d/2)$                                                            \\
  \(\br\)                  &    $\S1$     & Position vector                                                                                                          \\
  \(\bOmega\)              &    $\S1$     & Direction unit vector                                                                                                    \\
  \(\bm{v}\)               &    $\S2$   & Velocity vector $\bm{v}=\lvert \bm{v} \rvert \, \bOmega$                                                                 \\
  \(I\)                    &     $\S1$    & Radiative intensity                                                                                                      \\
  \(I_{-}\)                &    $\S2$     & Inflow boundary intensity on $\Gamma_{-}$                                                                                \\
  \(\mut\)                 &    $\S1$     & Total cross section                                                                                                      \\
  \(\mus\)                 &     $\S1$    & Scattering cross section                                                                                                 \\
  \(\mu_a\)                &     $\S2$    & Absorption cross section $\mu_a=\mut-\mus$                                                                               \\
  \(p\)                    &    $\S1$     & Phase function                                                                                                           \\
  \(\bn(\br)\)             &     $\S2$    & Unit outward normal at $\br\in\partial D$                                                                                \\
  \(\Gamma_{\pm}\)         &    $\S2$     & In/out-flow boundary sets: $\Gamma_{\pm}=\{(\br,\bOmega)\in \partial D\times \sS^{d-1}\mid \mp \bn(\br)\cdot\bOmega<0\}$ \\
  \(s\)                    &   $\S2.1$   & Path length parameter along characteristic line                                                                          \\
  \(s_{-}(\br,\bOmega)\)   &     $\S2.1$    & Distance from $\br$ to boundary along $-\bOmega$                                                                         \\
  \(g\)                    &      $\S5.1$   & Asymmetry parameter of Henyey--Greenstein phase function                                                                 \\
  \(\Phi\)                 &      $\S4.4$   & Scalar density (angular average of $I$)                                                                                  \\
  \(\tau_{\br,\bOmega}(s_1, s_2)\)   &     $\S2.1$    & Optical depth: $\tau_{\br,\bOmega}(s_1,s_2):=\int_{s_1}^{s_2} \mut(\br-s \bOmega) \diff{s}$                                                                                                \\
  \(G\)                    &     $\S2.2$    & Green's function of the RTE                                                                                              \\
  \(\A\)                   &      $\S2$   & Solution operator: $\A[I_-]=I$                                                                                           \\
  \(\J\)                   &    $\S2.1$     & Boundary (attenuation) operator                                                                                          \\
  \(\cL\)                  &     $\S2.1$    & Lifting operator (characteristic integration with $\mu_s$)                                                               \\
  \(\cS\)                  &   $\S2.1$      & Scattering operator                                                                                                      \\
  \(\rho_p\)       &     $\S2.1$    & Spectral radius in weighted $L^p$ space                                                                                  \\
    \(\sigma\)               &     $\S2.2$    & Characteristic length                                                                        \\
    \(\zeta_\sigma\)    &     $\S4.2$    & Mollification kernel                                                                             \\
  \(\delta_{\{\br'\}}\)    &     $\S2.2$    & Dirac distribution on $\partial D$ at $\br'$                                                                             \\
  \(\delta\)               &    $\S2.2$     & Dirac delta on direction space                                                                                           \\
  \addlinespace
  \multicolumn{3}{l}{\textbf{Network}}                                                                                                        \\
    \(\br^{\text{mesh}}_i\)     &    $\S3$     & Spatial mesh point                                                                              \\
        \((\mut^{\text{mesh}})_i\)     &    $\S3$     & Spatial dependent total cross section: $(\mut^{\text{mesh}})_i = \mut(\br^{\text{mesh}}_i)$                                                                              \\
            \((\mus^{\text{mesh}})_i\)     &    $\S3$     & Spatial dependent scattering cross section: $(\mus^{\text{mesh}})_i = \mus(\br^{\text{mesh}}_i)$        \\
              \(D_\mu\)       &  $\S5.1$       & Subdomain of cross section                                                                             \\
  \(d_{\text{model}}\)     &    $\S3.1$     & Latent / truncated representation dimension                                                                              \\
  \(d_{\text{mlp}}\)       &      $\S3.1$   & Width (hidden size) of MLP layers                                                                                        \\
    \(N_{\text{mesh}}\)       &  $\S3.1$       & Number of mesh points                                                                             \\
      \(N_{\text{quad}}\)       &  $\S3.1$       & Number of angular quadrature points                                                                               \\
  \(N_{\text{mlp}}\)       &  $\S3.1$       & Number of MLP layers in attenuation module                                                                               \\
  \(N_{\ell}\)             &    $\S3.2$     & Number of scattering (residual) blocks                                                                                   \\
  \(H\)                    &     $\S3.1$    & Number of attention heads                                                                                                \\
  \(d_k, d_v\)             &    $\S3.1$     & Key / value embedding dimensions in attention                                                                            \\
  \(d_{\tau}\)             &    $\S3.1$     & Dimension of optical-depth feature vector                                                                                \\
  \(\bG^{\text{NN}}\)      &    $\S3$     & Discrete vector representation of Green's function                                               \\
  \(\bm{W}^{\ell}\)        &    $\S3.2$     & Weight matrix in the $\ell$-th scattering block                                                                          \\
  \(\tau^{\text{NN}}_{-}\) &      $\S3.1$   & Neural optical depth estimate  
     \\
  \(\tau_{-,t}\)           &     $\S3.1$    & Total optical depth to boundary
  \\
    \(\tau_{-,s}\)           &   $\S3.1$    & Scattering optical depth to boundary
  \\
  \(\bm{b}^{\ell}\)        &    $\S3.2$     & Bias vector in the $\ell$-th scattering block                                                                                                                                           \\
  \(w_i\)                  &    $\S3$     & Quadrature weights                                                             \\
  \(\ell\)                 &   $\S4.3$      & Mean squared error                                                                                            \\
  \(\mathcal{L}\)          &     $\S4.3$    & Empirical training loss                                                                                                  \\
  \(B\)                    &     $\S5.2$    & Batch size                                                                                                               \\
  \(\eta\)                 &    $\S5.2$     & Learning rate                                                                                                            \\
\end{longtable}

%% Use \section commands to start a section

%% The Appendices part is started with the command \appendix;
%% appendix sections are then done as normal sections
% \appendix
% \section{Example Appendix Section}
% \label{app1}

% Appendix text.

% %% For citations use:
% %%       \citet{<label>} ==> Lamport [21]
% %%       \citep{<label>} ==> [21]
% %%
% Example citation, See \citet{lamport94}.

%% If you have bib database file and want bibtex to generate the
%% bibitems, please use
%%
%%  \bibliographystyle{elsarticle-num-names}
%%  \bibliography{<your bibdatabase>}

%% else use the following coding to input the bibitems directly in the
%% TeX file.

%% Refer following link for more details about bibliography and citations.
%% https://en.wikibooks.org/wiki/LaTeX/Bibliography_Management

% \begin{thebibliography}{00}

% %% For authoryear reference style
% %% \bibitem[Author(year)]{label}
% %% Text of bibliographic item

% \bibitem[Lamport(1994)]{lamport94}
%   Leslie Lamport,
%   \textit{\LaTeX: a document preparation system},
%   Addison Wesley, Massachusetts,
%   2nd edition,
%   1994.

\bibliographystyle{elsarticle-num-names}
\bibliography{refs}

\begin{thebibliography}{56}
\expandafter\ifx\csname natexlab\endcsname\relax\def\natexlab#1{#1}\fi
\providecommand{\url}[1]{\texttt{#1}}
\providecommand{\href}[2]{#2}
\providecommand{\path}[1]{#1}
\providecommand{\DOIprefix}{doi:}
\providecommand{\ArXivprefix}{arXiv:}
\providecommand{\URLprefix}{URL: }
\providecommand{\Pubmedprefix}{pmid:}
\providecommand{\doi}[1]{\href{http://dx.doi.org/#1}{\path{#1}}}
\providecommand{\Pubmed}[1]{\href{pmid:#1}{\path{#1}}}
\providecommand{\bibinfo}[2]{#2}
\ifx\xfnm\relax \def\xfnm[#1]{\unskip,\space#1}\fi
%Type = Book
\bibitem[{Case and Zweifel(1967)}]{case1967linear}
\bibinfo{author}{K.~Case}, \bibinfo{author}{P.~Zweifel}, \bibinfo{title}{Linear Transport Theory}, Addison-Wesley series in nuclear engineering, \bibinfo{publisher}{Addison-Wesley Publishing Company}, \bibinfo{year}{1967}. \URLprefix \url{https://books.google.co.jp/books?id=uQtRAAAAMAAJ}.
%Type = Book
\bibitem[{Lewis and Miller(1993)}]{lewis1993computational}
\bibinfo{author}{E.~Lewis}, \bibinfo{author}{W.~Miller}, \bibinfo{title}{Computational Methods of Neutron Transport}, \bibinfo{publisher}{American Nuclear Society}, \bibinfo{year}{1993}. \URLprefix \url{https://books.google.co.jp/books?id=yTKwjgEACAAJ}.
%Type = Book
\bibitem[{Marshak and Davis(2006)}]{marshak20063d}
\bibinfo{author}{A.~Marshak}, \bibinfo{author}{A.~Davis}, \bibinfo{title}{3D Radiative Transfer in Cloudy Atmospheres}, Physics of Earth and Space Environments, \bibinfo{publisher}{Springer Berlin Heidelberg}, \bibinfo{year}{2006}. \URLprefix \url{https://books.google.co.jp/books?id=NR8yY6M6I2QC}.
%Type = Article
\bibitem[{Koch and Becker(2004)}]{koch2004evaluation}
\bibinfo{author}{R.~Koch}, \bibinfo{author}{R.~Becker},
\newblock \bibinfo{title}{Evaluation of quadrature schemes for the discrete ordinates method},
\newblock \bibinfo{journal}{Journal of Quantitative Spectroscopy and Radiative Transfer} \bibinfo{volume}{84} (\bibinfo{year}{2004}) \bibinfo{pages}{423--435}.
%Type = Article
\bibitem[{Klose et~al.(2002)Klose, Netz, Beuthan, and Hielscher}]{klose2002optical}
\bibinfo{author}{A.~D. Klose}, \bibinfo{author}{U.~Netz}, \bibinfo{author}{J.~Beuthan}, \bibinfo{author}{A.~H. Hielscher},
\newblock \bibinfo{title}{Optical tomography using the time-independent equation of radiative transfer—part 1: forward model},
\newblock \bibinfo{journal}{Journal of Quantitative Spectroscopy and Radiative Transfer} \bibinfo{volume}{72} (\bibinfo{year}{2002}) \bibinfo{pages}{691--713}.
%Type = Article
\bibitem[{Tarvainen et~al.(2005)Tarvainen, Vauhkonen, Kolehmainen, and Kaipio}]{tarvainen2005hybrid}
\bibinfo{author}{T.~Tarvainen}, \bibinfo{author}{M.~Vauhkonen}, \bibinfo{author}{V.~Kolehmainen}, \bibinfo{author}{J.~P. Kaipio},
\newblock \bibinfo{title}{Hybrid radiative-transfer--diffusion model for optical tomography},
\newblock \bibinfo{journal}{Applied optics} \bibinfo{volume}{44} (\bibinfo{year}{2005}) \bibinfo{pages}{876--886}.
%Type = Article
\bibitem[{Joshi et~al.(2008)Joshi, Rasmussen, Sevick-Muraca, Wareing, and McGhee}]{joshi2008radiative}
\bibinfo{author}{A.~Joshi}, \bibinfo{author}{J.~C. Rasmussen}, \bibinfo{author}{E.~M. Sevick-Muraca}, \bibinfo{author}{T.~A. Wareing}, \bibinfo{author}{J.~McGhee},
\newblock \bibinfo{title}{Radiative transport-based frequency-domain fluorescence tomography},
\newblock \bibinfo{journal}{Physics in Medicine \& Biology} \bibinfo{volume}{53} (\bibinfo{year}{2008}) \bibinfo{pages}{2069}.
%Type = Article
\bibitem[{Lathrop(1969)}]{lathrop1969spatial}
\bibinfo{author}{K.~Lathrop},
\newblock \bibinfo{title}{Spatial differencing of the transport equation: positivity vs. accuracy},
\newblock \bibinfo{journal}{Journal of computational physics} \bibinfo{volume}{4} (\bibinfo{year}{1969}) \bibinfo{pages}{475--498}.
%Type = Article
\bibitem[{Han et~al.(2014)Han, Tang, and Ying}]{han2014two}
\bibinfo{author}{H.~Han}, \bibinfo{author}{M.~Tang}, \bibinfo{author}{W.~Ying},
\newblock \bibinfo{title}{Two uniform tailored finite point schemes for the two dimensional discrete ordinates transport equations with boundary and interface layers},
\newblock \bibinfo{journal}{Communications in Computational Physics} \bibinfo{volume}{15} (\bibinfo{year}{2014}) \bibinfo{pages}{797--826}.
%Type = Article
\bibitem[{Martin et~al.(1981)Martin, Yehnert, Lorence, and Duderstadt}]{martin1981phase}
\bibinfo{author}{W.~R. Martin}, \bibinfo{author}{C.~E. Yehnert}, \bibinfo{author}{L.~Lorence}, \bibinfo{author}{J.~J. Duderstadt},
\newblock \bibinfo{title}{Phase-space finite element methods applied to the first-order form of the transport equation},
\newblock \bibinfo{journal}{Annals of Nuclear Energy} \bibinfo{volume}{8} (\bibinfo{year}{1981}) \bibinfo{pages}{633--646}.
%Type = Article
\bibitem[{Tarvainen et~al.(2005)Tarvainen, Vauhkonen, Kolehmainen, Arridge, and Kaipio}]{tarvainen2005coupled}
\bibinfo{author}{T.~Tarvainen}, \bibinfo{author}{M.~Vauhkonen}, \bibinfo{author}{V.~Kolehmainen}, \bibinfo{author}{S.~R. Arridge}, \bibinfo{author}{J.~P. Kaipio},
\newblock \bibinfo{title}{Coupled radiative transfer equation and diffusion approximation model for photon migration in turbid medium with low-scattering and non-scattering regions},
\newblock \bibinfo{journal}{Physics in Medicine \& Biology} \bibinfo{volume}{50} (\bibinfo{year}{2005}) \bibinfo{pages}{4913}.
%Type = Article
\bibitem[{Ren et~al.(2004)Ren, Abdoulaev, Bal, and Hielscher}]{ren2004algorithm}
\bibinfo{author}{K.~Ren}, \bibinfo{author}{G.~S. Abdoulaev}, \bibinfo{author}{G.~Bal}, \bibinfo{author}{A.~H. Hielscher},
\newblock \bibinfo{title}{Algorithm for solving the equation of radiative transfer in the frequency domain},
\newblock \bibinfo{journal}{Optics letters} \bibinfo{volume}{29} (\bibinfo{year}{2004}) \bibinfo{pages}{578--580}.
%Type = Article
\bibitem[{Cockburn(2003)}]{cockburn2003discontinuous}
\bibinfo{author}{B.~Cockburn},
\newblock \bibinfo{title}{Discontinuous galerkin methods},
\newblock \bibinfo{journal}{ZAMM-Journal of Applied Mathematics and Mechanics/Zeitschrift f{\"u}r Angewandte Mathematik und Mechanik: Applied Mathematics and Mechanics} \bibinfo{volume}{83} (\bibinfo{year}{2003}) \bibinfo{pages}{731--754}.
%Type = Article
\bibitem[{Wareing et~al.(2001)Wareing, McGhee, Morel, and Pautz}]{wareing2001discontinuous}
\bibinfo{author}{T.~A. Wareing}, \bibinfo{author}{J.~M. McGhee}, \bibinfo{author}{J.~E. Morel}, \bibinfo{author}{S.~D. Pautz},
\newblock \bibinfo{title}{Discontinuous finite element sn methods on three-dimensional unstructured grids},
\newblock \bibinfo{journal}{Nuclear science and engineering} \bibinfo{volume}{138} (\bibinfo{year}{2001}) \bibinfo{pages}{256--268}.
%Type = Article
\bibitem[{Morel and Warsa(2005)}]{morel2005sn}
\bibinfo{author}{J.~E. Morel}, \bibinfo{author}{J.~S. Warsa},
\newblock \bibinfo{title}{An sn spatial discretization scheme for tetrahedral meshes},
\newblock \bibinfo{journal}{Nuclear science and engineering} \bibinfo{volume}{151} (\bibinfo{year}{2005}) \bibinfo{pages}{157--166}.
%Type = Book
\bibitem[{Lux(2018)}]{lux2018monte}
\bibinfo{author}{I.~Lux}, \bibinfo{title}{Monte Carlo particle transport methods}, \bibinfo{publisher}{CRC press}, \bibinfo{year}{2018}.
%Type = Book
\bibitem[{Spanier and Gelbard(2008)}]{spanier2008monte}
\bibinfo{author}{J.~Spanier}, \bibinfo{author}{E.~M. Gelbard}, \bibinfo{title}{Monte Carlo principles and neutron transport problems}, \bibinfo{publisher}{Courier Corporation}, \bibinfo{year}{2008}.
%Type = Article
\bibitem[{Raissi et~al.(2017)Raissi, Perdikaris, and Karniadakis}]{raissi2017physics}
\bibinfo{author}{M.~Raissi}, \bibinfo{author}{P.~Perdikaris}, \bibinfo{author}{G.~E. Karniadakis},
\newblock \bibinfo{title}{Physics informed deep learning (part i): Data-driven solutions of nonlinear partial differential equations},
\newblock \bibinfo{journal}{arXiv preprint arXiv:1711.10561}  (\bibinfo{year}{2017}).
%Type = Article
\bibitem[{Wang et~al.(2021)Wang, Teng, and Perdikaris}]{wang2021understanding}
\bibinfo{author}{S.~Wang}, \bibinfo{author}{Y.~Teng}, \bibinfo{author}{P.~Perdikaris},
\newblock \bibinfo{title}{Understanding and mitigating gradient flow pathologies in physics-informed neural networks},
\newblock \bibinfo{journal}{SIAM Journal on Scientific Computing} \bibinfo{volume}{43} (\bibinfo{year}{2021}) \bibinfo{pages}{A3055--A3081}.
%Type = Article
\bibitem[{Lu et~al.(2021)Lu, Jin, Pang, Zhang, and Karniadakis}]{lu2021learning}
\bibinfo{author}{L.~Lu}, \bibinfo{author}{P.~Jin}, \bibinfo{author}{G.~Pang}, \bibinfo{author}{Z.~Zhang}, \bibinfo{author}{G.~E. Karniadakis},
\newblock \bibinfo{title}{Learning nonlinear operators via deeponet based on the universal approximation theorem of operators},
\newblock \bibinfo{journal}{Nature machine intelligence} \bibinfo{volume}{3} (\bibinfo{year}{2021}) \bibinfo{pages}{218--229}.
%Type = Article
\bibitem[{Jin et~al.(2022)Jin, Meng, and Lu}]{jin2022mionet}
\bibinfo{author}{P.~Jin}, \bibinfo{author}{S.~Meng}, \bibinfo{author}{L.~Lu},
\newblock \bibinfo{title}{Mionet: Learning multiple-input operators via tensor product},
\newblock \bibinfo{journal}{SIAM Journal on Scientific Computing} \bibinfo{volume}{44} (\bibinfo{year}{2022}) \bibinfo{pages}{A3490--A3514}.
%Type = Article
\bibitem[{Wang et~al.(2021)Wang, Wang, and Perdikaris}]{wang2021learning}
\bibinfo{author}{S.~Wang}, \bibinfo{author}{H.~Wang}, \bibinfo{author}{P.~Perdikaris},
\newblock \bibinfo{title}{Learning the solution operator of parametric partial differential equations with physics-informed deeponets},
\newblock \bibinfo{journal}{Science advances} \bibinfo{volume}{7} (\bibinfo{year}{2021}) \bibinfo{pages}{eabi8605}.
%Type = Article
\bibitem[{Lu et~al.(2022)Lu, Pestourie, Johnson, and Romano}]{lu2022multifidelity}
\bibinfo{author}{L.~Lu}, \bibinfo{author}{R.~Pestourie}, \bibinfo{author}{S.~G. Johnson}, \bibinfo{author}{G.~Romano},
\newblock \bibinfo{title}{Multifidelity deep neural operators for efficient learning of partial differential equations with application to fast inverse design of nanoscale heat transport},
\newblock \bibinfo{journal}{Physical Review Research} \bibinfo{volume}{4} (\bibinfo{year}{2022}) \bibinfo{pages}{023210}.
%Type = Article
\bibitem[{Zhu et~al.(2023)Zhu, Zhang, Jiao, Karniadakis, and Lu}]{zhu2023reliable}
\bibinfo{author}{M.~Zhu}, \bibinfo{author}{H.~Zhang}, \bibinfo{author}{A.~Jiao}, \bibinfo{author}{G.~E. Karniadakis}, \bibinfo{author}{L.~Lu},
\newblock \bibinfo{title}{Reliable extrapolation of deep neural operators informed by physics or sparse observations},
\newblock \bibinfo{journal}{Computer Methods in Applied Mechanics and Engineering} \bibinfo{volume}{412} (\bibinfo{year}{2023}) \bibinfo{pages}{116064}.
%Type = Article
\bibitem[{Li et~al.(2020)Li, Kovachki, Azizzadenesheli, Liu, Bhattacharya, Stuart, and Anandkumar}]{li2020fourier}
\bibinfo{author}{Z.~Li}, \bibinfo{author}{N.~Kovachki}, \bibinfo{author}{K.~Azizzadenesheli}, \bibinfo{author}{B.~Liu}, \bibinfo{author}{K.~Bhattacharya}, \bibinfo{author}{A.~Stuart}, \bibinfo{author}{A.~Anandkumar},
\newblock \bibinfo{title}{Fourier neural operator for parametric partial differential equations},
\newblock \bibinfo{journal}{arXiv preprint arXiv:2010.08895}  (\bibinfo{year}{2020}).
%Type = Article
\bibitem[{Wen et~al.(2022)Wen, Li, Azizzadenesheli, Anandkumar, and Benson}]{wen2022u}
\bibinfo{author}{G.~Wen}, \bibinfo{author}{Z.~Li}, \bibinfo{author}{K.~Azizzadenesheli}, \bibinfo{author}{A.~Anandkumar}, \bibinfo{author}{S.~M. Benson},
\newblock \bibinfo{title}{U-fno—an enhanced fourier neural operator-based deep-learning model for multiphase flow},
\newblock \bibinfo{journal}{Advances in Water Resources} \bibinfo{volume}{163} (\bibinfo{year}{2022}) \bibinfo{pages}{104180}.
%Type = Article
\bibitem[{Guibas et~al.(2021)Guibas, Mardani, Li, Tao, Anandkumar, and Catanzaro}]{guibas2021adaptive}
\bibinfo{author}{J.~Guibas}, \bibinfo{author}{M.~Mardani}, \bibinfo{author}{Z.~Li}, \bibinfo{author}{A.~Tao}, \bibinfo{author}{A.~Anandkumar}, \bibinfo{author}{B.~Catanzaro},
\newblock \bibinfo{title}{Adaptive fourier neural operators: Efficient token mixers for transformers},
\newblock \bibinfo{journal}{arXiv preprint arXiv:2111.13587}  (\bibinfo{year}{2021}).
%Type = Article
\bibitem[{Fan and Ying(2019)}]{fan2019solving}
\bibinfo{author}{Y.~Fan}, \bibinfo{author}{L.~Ying},
\newblock \bibinfo{title}{{Solving Optical Tomography with Deep Learning}},
\newblock \bibinfo{journal}{arXiv}  (\bibinfo{year}{2019}). \DOIprefix\doi{10.48550/arxiv.1910.04756}. \href{http://arxiv.org/abs/1910.04756}{{\tt arXiv:1910.04756}}.
%Type = Misc
\bibitem[{Bradbury et~al.(2018)Bradbury, Frostig, Hawkins, Johnson, Leary, Maclaurin, Necula, Paszke, Vander{P}las, Wanderman-{M}ilne, and Zhang}]{jax2018github}
\bibinfo{author}{J.~Bradbury}, \bibinfo{author}{R.~Frostig}, \bibinfo{author}{P.~Hawkins}, \bibinfo{author}{M.~J. Johnson}, \bibinfo{author}{C.~Leary}, \bibinfo{author}{D.~Maclaurin}, \bibinfo{author}{G.~Necula}, \bibinfo{author}{A.~Paszke}, \bibinfo{author}{J.~Vander{P}las}, \bibinfo{author}{S.~Wanderman-{M}ilne}, \bibinfo{author}{Q.~Zhang}, \bibinfo{title}{{JAX}: composable transformations of {P}ython+{N}um{P}y programs}, \bibinfo{year}{2018}. \URLprefix \url{http://github.com/jax-ml/jax}.
%Type = Misc
\bibitem[{Heek et~al.(2024)Heek, Levskaya, Oliver, Ritter, Rondepierre, Steiner, and van {Z}ee}]{flax2020github}
\bibinfo{author}{J.~Heek}, \bibinfo{author}{A.~Levskaya}, \bibinfo{author}{A.~Oliver}, \bibinfo{author}{M.~Ritter}, \bibinfo{author}{B.~Rondepierre}, \bibinfo{author}{A.~Steiner}, \bibinfo{author}{M.~van {Z}ee}, \bibinfo{title}{{F}lax: A neural network library and ecosystem for {JAX}}, \bibinfo{year}{2024}. \URLprefix \url{http://github.com/google/flax}.
%Type = Misc
\bibitem[{DeepMind et~al.(2020)DeepMind, Babuschkin, Baumli, Bell, Bhupatiraju, Bruce, Buchlovsky, Budden, Cai, Clark, Danihelka, Dedieu, Fantacci, Godwin, Jones, Hemsley, Hennigan, Hessel, Hou, Kapturowski, Keck, Kemaev, King, Kunesch, Martens, Merzic, Mikulik, Norman, Papamakarios, Quan, Ring, Ruiz, Sanchez, Sartran, Schneider, Sezener, Spencer, Srinivasan, Stanojevi\'{c}, Stokowiec, Wang, Zhou, and Viola}]{deepmind2020jax}
\bibinfo{author}{DeepMind}, \bibinfo{author}{I.~Babuschkin}, \bibinfo{author}{K.~Baumli}, \bibinfo{author}{A.~Bell}, \bibinfo{author}{S.~Bhupatiraju}, \bibinfo{author}{J.~Bruce}, \bibinfo{author}{P.~Buchlovsky}, \bibinfo{author}{D.~Budden}, \bibinfo{author}{T.~Cai}, \bibinfo{author}{A.~Clark}, \bibinfo{author}{I.~Danihelka}, \bibinfo{author}{A.~Dedieu}, \bibinfo{author}{C.~Fantacci}, \bibinfo{author}{J.~Godwin}, \bibinfo{author}{C.~Jones}, \bibinfo{author}{R.~Hemsley}, \bibinfo{author}{T.~Hennigan}, \bibinfo{author}{M.~Hessel}, \bibinfo{author}{S.~Hou}, \bibinfo{author}{S.~Kapturowski}, \bibinfo{author}{T.~Keck}, \bibinfo{author}{I.~Kemaev}, \bibinfo{author}{M.~King}, \bibinfo{author}{M.~Kunesch}, \bibinfo{author}{L.~Martens}, \bibinfo{author}{H.~Merzic}, \bibinfo{author}{V.~Mikulik}, \bibinfo{author}{T.~Norman}, \bibinfo{author}{G.~Papamakarios}, \bibinfo{author}{J.~Quan}, \bibinfo{author}{R.~Ring}, \bibinfo{author}{F.~Ruiz}, \bibinfo{author}{A.~Sanchez}, \bibinfo{author}{L.~Sartran},
  \bibinfo{author}{R.~Schneider}, \bibinfo{author}{E.~Sezener}, \bibinfo{author}{S.~Spencer}, \bibinfo{author}{S.~Srinivasan}, \bibinfo{author}{M.~Stanojevi\'{c}}, \bibinfo{author}{W.~Stokowiec}, \bibinfo{author}{L.~Wang}, \bibinfo{author}{G.~Zhou}, \bibinfo{author}{F.~Viola}, \bibinfo{title}{The {D}eep{M}ind {JAX} {E}cosystem}, \bibinfo{year}{2020}. \URLprefix \url{http://github.com/google-deepmind}.
%Type = Article
\bibitem[{Case and Zweifel(1963)}]{case1963existence}
\bibinfo{author}{K.~M. Case}, \bibinfo{author}{P.~F. Zweifel},
\newblock \bibinfo{title}{{Existence and Uniqueness Theorems for the Neutron Transport Equation}},
\newblock \bibinfo{journal}{Journal of Mathematical Physics} \bibinfo{volume}{4} (\bibinfo{year}{1963}) \bibinfo{pages}{1376--1385}. \DOIprefix\doi{10.1063/1.1703916}.
%Type = Book
\bibitem[{Agoshkov(2012)}]{agoshkov2012boundary}
\bibinfo{author}{V.~Agoshkov}, \bibinfo{title}{Boundary value problems for transport equations}, \bibinfo{publisher}{Springer Science \& Business Media}, \bibinfo{year}{2012}.
%Type = Article
\bibitem[{Choulli and Stefanov(1999)}]{choulli1999inverse}
\bibinfo{author}{M.~Choulli}, \bibinfo{author}{P.~Stefanov},
\newblock \bibinfo{title}{{An inverse boundary value problem for the stationary transport equation}},
\newblock \bibinfo{journal}{Osaka Journal of Mathematics} \bibinfo{volume}{36} (\bibinfo{year}{1999}) \bibinfo{pages}{87--104}. \DOIprefix\doi{10.18910/3941}.
%Type = Article
\bibitem[{Vladimirov(1963)}]{vladimirov1963mathematical}
\bibinfo{author}{V.~S. Vladimirov},
\newblock \bibinfo{title}{Mathematical problems in the one-velocity theory of particle transport}  (\bibinfo{year}{1963}).
%Type = Article
\bibitem[{Manteuffel et~al.(1999)Manteuffel, Ressel, and Starke}]{manteuffel1999boundary}
\bibinfo{author}{T.~A. Manteuffel}, \bibinfo{author}{K.~J. Ressel}, \bibinfo{author}{G.~Starke},
\newblock \bibinfo{title}{A boundary functional for the least-squares finite-element solution of neutron transport problems},
\newblock \bibinfo{journal}{SIAM Journal on Numerical Analysis} \bibinfo{volume}{37} (\bibinfo{year}{1999}) \bibinfo{pages}{556--586}.
%Type = Article
\bibitem[{Egger and Schlottbom(2012)}]{egger2012mixed}
\bibinfo{author}{H.~Egger}, \bibinfo{author}{M.~Schlottbom},
\newblock \bibinfo{title}{A mixed variational framework for the radiative transfer equation},
\newblock \bibinfo{journal}{Mathematical Models and Methods in Applied Sciences} \bibinfo{volume}{22} (\bibinfo{year}{2012}) \bibinfo{pages}{1150014}.
%Type = Article
\bibitem[{Pettersson(2001)}]{pettersson2001stationary}
\bibinfo{author}{R.~Pettersson},
\newblock \bibinfo{title}{On stationary solutions to the linear boltzmann equation},
\newblock \bibinfo{journal}{Transport Theory and Statistical Physics} \bibinfo{volume}{30} (\bibinfo{year}{2001}) \bibinfo{pages}{549--560}.
%Type = Article
\bibitem[{Falk(2003)}]{falk2003existence}
\bibinfo{author}{L.~Falk},
\newblock \bibinfo{title}{Existence of solutions to the stationary linear boltzmann equation}  (\bibinfo{year}{2003}).
%Type = Article
\bibitem[{Egger and Schlottbom(2014)}]{egger2014stationary}
\bibinfo{author}{H.~Egger}, \bibinfo{author}{M.~Schlottbom},
\newblock \bibinfo{title}{Stationary radiative transfer with vanishing absorption},
\newblock \bibinfo{journal}{Mathematical Models and Methods in Applied Sciences} \bibinfo{volume}{24} (\bibinfo{year}{2014}) \bibinfo{pages}{973--990}.
%Type = Article
\bibitem[{Stefanov and Uhlmann(2008)}]{stefanov2008inverse}
\bibinfo{author}{P.~Stefanov}, \bibinfo{author}{G.~Uhlmann},
\newblock \bibinfo{title}{An inverse source problem in optical molecular imaging},
\newblock \bibinfo{journal}{Analysis \& PDE} \bibinfo{volume}{1} (\bibinfo{year}{2008}) \bibinfo{pages}{115--126}.
%Type = Article
\bibitem[{Egger and Schlottbom(2014)}]{egger2014lp}
\bibinfo{author}{H.~Egger}, \bibinfo{author}{M.~Schlottbom},
\newblock \bibinfo{title}{{An $L^p$ theory for stationary radiative transfer}},
\newblock \bibinfo{journal}{Applicable Analysis} \bibinfo{volume}{93} (\bibinfo{year}{2014}) \bibinfo{pages}{1283--1296}. \DOIprefix\doi{10.1080/00036811.2013.826798}. \href{http://arxiv.org/abs/1304.6504}{{\tt arXiv:1304.6504}}.
%Type = Article
\bibitem[{Adams and Larsen(2002)}]{adams2002fast}
\bibinfo{author}{M.~L. Adams}, \bibinfo{author}{E.~W. Larsen},
\newblock \bibinfo{title}{Fast iterative methods for discrete-ordinates particle transport calculations},
\newblock \bibinfo{journal}{Progress in nuclear energy} \bibinfo{volume}{40} (\bibinfo{year}{2002}) \bibinfo{pages}{3--159}.
%Type = Misc
\bibitem[{Ba et~al.(2016)Ba, Kiros, and Hinton}]{ba2016layernormalization}
\bibinfo{author}{J.~L. Ba}, \bibinfo{author}{J.~R. Kiros}, \bibinfo{author}{G.~E. Hinton}, \bibinfo{title}{Layer normalization}, \bibinfo{year}{2016}. \URLprefix \url{https://arxiv.org/abs/1607.06450}. \href{http://arxiv.org/abs/1607.06450}{{\tt arXiv:1607.06450}}.
%Type = Article
\bibitem[{Xu et~al.(2019)Xu, Sun, Zhang, Zhao, and Lin}]{xu2019understanding}
\bibinfo{author}{J.~Xu}, \bibinfo{author}{X.~Sun}, \bibinfo{author}{Z.~Zhang}, \bibinfo{author}{G.~Zhao}, \bibinfo{author}{J.~Lin},
\newblock \bibinfo{title}{Understanding and improving layer normalization},
\newblock \bibinfo{journal}{Advances in neural information processing systems} \bibinfo{volume}{32} (\bibinfo{year}{2019}).
%Type = Article
\bibitem[{Raviart(1983)}]{raviart1983analysis}
\bibinfo{author}{P.~Raviart},
\newblock \bibinfo{title}{An analysis of particle methods in numerical methods in fluid dynamics},
\newblock \bibinfo{journal}{Lecture Notes in Math} \bibinfo{volume}{1127} (\bibinfo{year}{1983}).
%Type = Incollection
\bibitem[{Chertock(2017)}]{chertock2017practical}
\bibinfo{author}{A.~Chertock},
\newblock \bibinfo{title}{A practical guide to deterministic particle methods},
\newblock in: \bibinfo{booktitle}{Handbook of numerical analysis}, volume~\bibinfo{volume}{18}, \bibinfo{publisher}{Elsevier}, \bibinfo{year}{2017}, pp. \bibinfo{pages}{177--202}.
%Type = Article
\bibitem[{Hornik et~al.(1989)Hornik, Stinchcombe, and White}]{hornik1989multilayer}
\bibinfo{author}{K.~Hornik}, \bibinfo{author}{M.~Stinchcombe}, \bibinfo{author}{H.~White},
\newblock \bibinfo{title}{Multilayer feedforward networks are universal approximators},
\newblock \bibinfo{journal}{Neural networks} \bibinfo{volume}{2} (\bibinfo{year}{1989}) \bibinfo{pages}{359--366}.
%Type = Article
\bibitem[{Weinan et~al.(????)Weinan, Ma, and Wu}]{weinanpriori}
\bibinfo{author}{E.~Weinan}, \bibinfo{author}{C.~Ma}, \bibinfo{author}{L.~Wu},
\newblock \bibinfo{title}{A priori estimates for two-layer neural networks}  (????).
%Type = Article
\bibitem[{Weinan et~al.(2019)Weinan, Ma, and Wu}]{weinan2019barron}
\bibinfo{author}{E.~Weinan}, \bibinfo{author}{C.~Ma}, \bibinfo{author}{L.~Wu},
\newblock \bibinfo{title}{Barron spaces and the compositional function spaces for neural network models},
\newblock \bibinfo{journal}{arXiv preprint arXiv:1906.08039}  (\bibinfo{year}{2019}).
%Type = Misc
\bibitem[{Kingma and Ba(2017)}]{kingma2017}
\bibinfo{author}{D.~P. Kingma}, \bibinfo{author}{J.~Ba}, \bibinfo{title}{Adam: A method for stochastic optimization}, \bibinfo{year}{2017}. \URLprefix \url{https://arxiv.org/abs/1412.6980}. \href{http://arxiv.org/abs/1412.6980}{{\tt arXiv:1412.6980}}.
%Type = Misc
\bibitem[{Rakhlin et~al.(2012)Rakhlin, Shamir, and Sridharan}]{rakhlin2012}
\bibinfo{author}{A.~Rakhlin}, \bibinfo{author}{O.~Shamir}, \bibinfo{author}{K.~Sridharan}, \bibinfo{title}{Making gradient descent optimal for strongly convex stochastic optimization}, \bibinfo{year}{2012}. \URLprefix \url{https://arxiv.org/abs/1109.5647}. \href{http://arxiv.org/abs/1109.5647}{{\tt arXiv:1109.5647}}.
%Type = Techreport
\bibitem[{Koch et~al.(1991)Koch, Baker, and Alcouffe}]{koch1991solution}
\bibinfo{author}{K.~Koch}, \bibinfo{author}{R.~Baker}, \bibinfo{author}{R.~Alcouffe}, \bibinfo{title}{Solution of the first-order form of the 3-D discrete ordinates equations on a massively parallel machine}, \bibinfo{type}{Technical Report}, Los Alamos National Lab., NM (United States), \bibinfo{year}{1991}.
%Type = Article
\bibitem[{Zeyao et~al.(2004)Zeyao, Lianxiang, and Shulin}]{zeyao2004parallel}
\bibinfo{author}{M.~Zeyao}, \bibinfo{author}{F.~Lianxiang}, \bibinfo{author}{Y.~Shulin},
\newblock \bibinfo{title}{Parallel pipelined sn algorithm for neutron transport on unstructured grid, chinese j},
\newblock \bibinfo{journal}{Comput} \bibinfo{volume}{27} (\bibinfo{year}{2004}) \bibinfo{pages}{587--595}.
%Type = Article
\bibitem[{Liu et~al.(2019)Liu, Li, Sun, and Yu}]{liu2019advances}
\bibinfo{author}{Y.~Liu}, \bibinfo{author}{J.~Li}, \bibinfo{author}{S.~Sun}, \bibinfo{author}{B.~Yu},
\newblock \bibinfo{title}{Advances in gaussian random field generation: a review},
\newblock \bibinfo{journal}{Computational Geosciences} \bibinfo{volume}{23} (\bibinfo{year}{2019}) \bibinfo{pages}{1011--1047}.
%Type = Article
\bibitem[{Ruan and McLaughlin(1998)}]{ruan1998efficient}
\bibinfo{author}{F.~Ruan}, \bibinfo{author}{D.~McLaughlin},
\newblock \bibinfo{title}{An efficient multivariate random field generator using the fast fourier transform},
\newblock \bibinfo{journal}{Advances in water resources} \bibinfo{volume}{21} (\bibinfo{year}{1998}) \bibinfo{pages}{385--399}.

\end{thebibliography}
% \end{thebibliography}

\end{document}